\documentclass[a4paper,fleqn]{scrbook}

\input{internals}

\usepackage{xifthen} 
\usepackage{tikz}

\usetikzlibrary{decorations.pathreplacing}
\usetikzlibrary{decorations.pathmorphing}
\usetikzlibrary{decorations.markings}
\usetikzlibrary{shapes,shapes.symbols,automata,arrows}
\usetikzlibrary{calc}
\usetikzlibrary{patterns}
\usetikzlibrary{positioning}
\usetikzlibrary{fit}
\usetikzlibrary{fadings}

\pgfdeclarelayer{background}
\pgfdeclarelayer{foreground}
\pgfsetlayers{background,main,foreground}

\usepackage[headsepline=.6pt,footsepline=false]{scrlayer-scrpage} 

\usepackage{scrhack} 

\usepackage[Lenny]{fncychap} 

\usepackage{hyperref} 

\usepackage{marginnote} 
\usepackage{mparhack} 

\usepackage{amsmath} 
\usepackage{amsthm} 
\usepackage{amsfonts} 
\usepackage{amssymb} 
\usepackage{thmtools} 
\usepackage{mathtools} 
\usepackage{nicefrac} 
\usepackage[capitalise]{cleveref} 
\usepackage[ruled,vlined,linesnumbered]{algorithm2e} 

\usepackage{wrapfig}
\usepackage{iftex} 

\usepackage{etoolbox} 

\usepackage{subcaption} 
\usepackage{rotating} 

\usepackage{booktabs} 
\usepackage{multirow,multicol} 
\usepackage{xspace} 
\usepackage{xparse} 

\usepackage{pifont} 
\usepackage{marvosym} 
\usepackage{calc} 

\usepackage[inline]{enumitem} 

\usepackage{smartref} 

\usepackage{url} 

\usepackage[color=blue!20, linecolor=blue!80!black]{todonotes} 
\usepackage{changebar} 

\usepackage{eso-pic} 

\usepackage{listings} 
\usepackage{array} 
\usepackage{dsfont} 
\usepackage[binary,squaren]{SIunits} 
\usepackage{textcomp} 
\usepackage{lstautogobble}  

\usepackage{needspace} 

\usepackage{cases}
\usepackage{csvsimple}

\newcommand{\env}{environment}

\makeatletter
\newcommand*{\etc}{%
    \@ifnextchar{.}%
        {etc}%
        {etc.\@\xspace}%
}
\newcommand*{\cf}{%
    \@ifnextchar{.}%
        {cf}%
        {cf.\@\xspace}%
}
\newcommand*{\etal}{%
    \@ifnextchar{.}%
        {et~al}%
        {et~al.\@\xspace}%
}
\makeatother

\NewDocumentCommand{\todoc}{o}{[?]\todo{\IfNoValueTF{#1}{Cite}{Cite: #1}}~}

\newlist{todolist}{itemize}{2}
\setlist[todolist]{label=$\square$}

\crefname{definition}{Def.}{Defs.}%
\crefname{section}{Sect.}{Sect.}%
\crefname{examplex}{Exm.}{Exms.}%
\crefname{equation}{Eq.}{Eqs.}%
\crefname{proposition}{Prop.}{Props.}%
\crefname{lemma}{Lem.}{Lems.}%
\crefname{corollary}{Cor.}{Cors.}%
\crefname{theorem}{Thm.}{Thms.}%
\crefname{remark}{Rem.}{Rems.}%
\crefname{figure}{Fig.}{Figs.}%
\crefname{table}{Tbl.}{Tbls.}%
\crefname{algorithm}{Alg.}{Algs.}%
\Crefname{examplex}{Example}{Examples}%

\allowdisplaybreaks{}

\addtoreflist{chapter}
\addtoreflist{section}

\newcommand*\revieweradvisorlistdiss{
  \begin{tabular}{ll}
    Dean: &\DEAN{} \\[\medskipamount]
    Advisor: &\ADVISOR{}
    \ifdef{\ADVISORTWO}{%
        \\%
        &\ADVISORTWO{}%
    }{}%
    \\[\medskipamount]
    Examination Board: & \EXAMINERONE~(Chair) \\
    & \EXAMINERTWO{} \\
    & \EXAMINERTHREE{} \\
    & \EXAMINERFOUR~(Academic Staff) \\[\medskipamount]
    Thesis Defense: & \DEFENSEDATE{}
  \end{tabular}
}
\newcommand*\revieweradvisorlistthesis{
  \begin{tabular}{ll}
    Supervisor: &\SUPERVISOR{} \\[\medskipamount]
    \ifdef{\ADVISOR}{%
    	Advisor: &\ADVISOR{}
	    \ifdef{\ADVISORTWO}{%
    	    \\%
        	&\ADVISORTWO{}%
	    }{}%
	    \\[\medskipamount]
	}{}%
    Reviewer: &\EXAMINERONE{} \\
    &\EXAMINERTWO{} \\[\medskipamount]
    Submission: &\SUBMISSIONMONTH~\SUBMISSIONDAY, \SUBMISSIONYEAR{}
  \end{tabular}
}

\newcommand*\revieweradvisorlist{
	\ifthenelse{\boolean{dissertation}}{
		\revieweradvisorlistdiss
	}{
	    \revieweradvisorlistthesis
	}
}

\lstdefinelanguage{Lola}{
  keywords=[0]{input, output, trigger},
  keywordstyle=[0]\bfseries\color{bluekeywords},
  keywords=[1]{if, then, else, aggregate, defaults, offset},
  keywords=[2]{Int8, Int16, Int32, Int64, UInt8, UInt16, UInt32, UInt64, Bool, Float32, Float64, @1Hz, @5Hz, @10Hz, @100mHz, @1kHz},
  keywordstyle=[2]\color{greentypes},
    sensitive=false,
    comment=[l]{//},
    morecomment=[s]{/*}{*/},
    morestring=[b]',
    morestring=[b]"
}

\lstdefinelanguage{none}{
  identifierstyle=
}

\usepackage[style=numeric,backend=biber]{biblatex}
\begin{document}

	\pagestyle{empty}
	\frontmatter
	
\ifthenelse{\boolean{titlepagelogo}}{
  \ifthenelse{\boolean{nightmode}}{\printlogotransparentbg}{\printlogoopaquebg}
  \newgeometry{margin=0cm}
  \maketitle 
  \restoregeometry{}
}{
  \maketitle 
}

\clearpage

\vspace*{\fill}

\revieweradvisorlist

\clearpage

\section*{Declaration of AI Tools and Their Use}

Two AI-supported tools were used during the preparation of this thesis: DeepL for translation assistance and ChatGPT (primarily the GPT-4o-based version) for suggestions regarding the reformulation of text and for support in generating parts of the code developed in the context of this work. No extensive AI-generated text blocks were adopted verbatim. The specific sections of the code for which ChatGPT was used are explicitly marked in the source code.

\cleardoublepage{}

\section*{Abstract}

\ifthenelse{\boolean{dissertation}}{
  According to the dissertation regulations, the abstract must not contain more than 1500 characters, including spaces.
}{

\emph{Reinforcement Learning} is a machine learning methodology that has demonstrated strong performance across a variety of tasks. In particular, it plays a central role in the development of \textit{artificial autonomous agents}. As these agents become increasingly capable, market readiness is rapidly approaching, which means those agents, for example taking the form of humanoid robots or autonomous cars, are poised to transition from laboratory prototypes to autonomous operation in real-world environments. This transition raises concerns leading to specific requirements for these systems -- among them, the requirement that they are designed to behave \textit{ethically}. Crucially, research directed toward building agents that fulfill the requirement to behave ethically -- referred to as \emph{artificial moral agents} (AMAs) -- has to address a range of challenges at the intersection of computer science and philosophy, which makes it difficult to develop solutions that are both philosophically well-informed and concrete, realizable approaches in light of the imminent deployment of autonomous artificial agents in real-world contexts. 

This study explores the development of \emph{reason-based artificial moral agents} (RBAMAs). RBAMAs are build on an extension of the reinforcement learning architecture to enable moral decision-making based on \textit{sound normative reasoning}, which is achieved by equipping the agent with the capacity to learn a \textit{reason-theory} -- a theory which enables it to process morally relevant propositions to derive moral obligations -- through \textit{case-based feedback}. They are designed such that they adapt their behavior to \textit{ensure conformance} to these obligations while they pursue their designated tasks. These features contribute to the \textit{moral justifiability} of the their actions, their \textit{moral robustness}, and their \textit{moral trustworthiness}, which proposes the extended architecture as a concrete and deployable framework for the development of AMAs that further fulfills key ethical desiderata. This study presents a first implementation of an RBAMA, accompanied by a framework for establishing targeted test cases, used to demonstrate the potential of RBMAs in initial experiments. The study concludes with a discussion of this potential -- one that also highlights the manifoldness of interdisciplinary challenges involved in the development of AMAs -- and outlines directions for future research aiming at addressing the current limitations of RBAMAs in comprehensively meeting these challenges.
} 
\cleardoublepage{}

\section*{Acknowledgements}

\ifthenelse{\boolean{dissertation}}{
Here, you should acknowledge everyone who has helped you in any way, shape, or form to write this work.
It is customary to first the advisor, then colleagues, then the reviewers, maybe funding agencies, if you would like, then personal contacts.
Feel free to vary this order as you would like. 

For the first version that goes to the reviewers, this section may be omitted, in particular, as you do not yet (officially) know your reviewers.
}{
  I would like to express my sincere gratitude to my advisors, Dr. Kevin Baum and Felix Jahn, for their continuous support, valuable feedback, and insightful discussions throughout the development of this work. Their expertise and guidance greatly contributed to shaping the project. I am also deeply thankful to my colleagues at DFKI, especially at RAIME, for the helpful conversations. Furthermore, I would like to thank Prof. Verena Wolf for her support of this project. Finally, I am profoundly grateful to my family, my friends, and my boyfriend for their unwavering encouragement and care.
}
\cleardoublepage{}

\tableofcontents
\setcounter{tocdepth}{1}
	
	\pagestyle{scrheadings}
	\mainmatter{}
    \chapter{Introduction}

With the rapid progress in developing \textit{artificial autonomous agents} (AAAs), the once-distant vision of machines walking among us -- in our workplaces, public spaces, and homes -- draws steadily closer to reality. Efforts directed toward building humanoid robots like Tesla's "Optimus" \cite{tesla_optimus_2025}, Boston Dynamics' "Atlas" \cite{bostondynamics_atlas_2025}, and UBTECH's "Walker" \cite{ubtech_walker_2025} are advancing toward enabling AAAs to perform complex tasks across a range of sectors in the coming years. In parallel, companies like Wayve \cite{wayve2025} and Waymo \cite{future_is_now_2024} have demonstrated early successes deploying autonomous vehicles on public roads, while companies such as Covariant \cite{covariant2025} and Universal Robots \cite{universalrobots2025} are pushing the boundaries of robotics in industrial settings. 

The deployment of these AAAs presents a significant opportunity to transform our everyday lives for the better. However, as we near this new reality, it is crucial to address the societal challenges that accompany this development -- among the most paramount of which is the imperative for these agents to \textit{behave ethically}, particularly in interactions with humans. Machine ethics \cite{anderson2011, anderson2007machine, anderson2004, moor2006a} engages with the numerous questions involved in building such agents -- questions surrounding the creation of \introterm{artificial moral agents} (AMAs) \cite{allen2006, wallach2008, tonkens2009}. 

This is not a straightforward task: researchers face challenges \textit{deeply rooted} in \textit{foundational issues in moral philosophy}, that hinder the advancement of concrete frameworks for AMAs. Most notably, the \textit{absence} of a \textit{widely accepted ethical theory} undermines the hope that one could simply proceed by directly implementing a settled moral doctrine to give moral guidance to AAAs. Nonetheless, existing approaches tend to be implicitly \textit{tied} to a particular ethical standpoint or family of such standpoints. In approaches to machine ethics based on reinforcement learning, for instance, using the reward mechanism to instill ethical behavior -- by assigning numerical values to the outcomes of actions in order to shape the agent’s behavior -- tends to align naturally with a consequentialist account \cite{sep-consequentialism, driver2011consequentialism}, as moral rightness is evaluated \textit{purely} in terms of consequences of actions. 

Undoubtedly, it is essential to develop approaches that take into account critical insights from moral philosophy. At the same time, the accelerating progress toward the deployment of AAAs heightens the urgency for \textit{practically applicable} solutions. In this study, I aim to advance a \textit{reason-based} approach to developing AMAs that fulfills \textit{both} of these requirements. Specifically, I seek to contribute to the development of an architecture for building \textit{reason-based artificial moral agents} (RBAMAs). This architecture is grounded in reinforcement learning (RL) \cite{sutton2018, littman2015, kaelbling1996reinforcement, wiering2012reinforcement} -- a state-of-the-art machine learning methodology -- which is extended to ground the agent's moral decision-making in \textit{normative reasoning} \cite{sep-reasons-just-vs-expl} for ethical guidance. To this end, RBAMAs are equipped with the capacity to \textit{recognize normative reasons} -- represented as propositions about morally relevant facts -- and the ability to \textit{infer moral obligations} from them to which they are then bound in their course of action. 

For illustration, consider the following simplistic example: Assume, an RBAMA is tasked with delivering a package. On its path, it must cross a narrow bridge leading across a river and it may encounter persons standing on the bridge or persons who have fallen into the water. The presence of persons on the bridge and the presence of persons in the water are examples of morally relevant facts, recognized by the RBAMA as such. From these, it infers moral obligations -- not to push the person off the bridge, and to rescue the individual in the water -- according to which it determines its course of action. 

At the core of the extended architecture enabling this functionality, lies an \textit{ethics module}, governed by a \textit{reasoning unit}. This unit operates on a learned \textit{reason-theory} -- a theory composed of normative reasons and a inference procedure. As one possible account on how to construct such a reason-theory, I follow an approach that builds on John Horty’s formalization of reason-based moral decision-making \cite{horty2012, horty2007}. In this framework, normative reasoning is modeled through \textit{default rules} -- defeasible inference rules -- that derive moral obligations from morally relevant propositions  incorporating a mechanism for \textit{prioritization} in cases where obligations cannot be jointly fulfilled. Returning to the example, the presence of a person in the water -- a morally relevant fact, denoted as $D$ -- would be represented in the reason-theory as the premise of a default rule $\delta_1$ linking this proposition to the obligation to perform a rescue, $\varphi_R$. Similarly, the presence of a person on the bridge -- another morally relevant fact, $B$ -- would be linked to the obligation not to push the person into the water -- $\varphi_C$--  by a second default rule $\delta_2$. The reason-theory may further include a priority ordering over these rules, indicating which obligation should take precedence in cases where $D$ and $B$ hold, but the agent cannot confirm to both moral obligations $\varphi_R$ and $\varphi_C$.

In addition to being endowed with an integrated reason-theory that grounds moral decision-making in normative reasons, RBAMAs are further designed to ensure that their \textit{behavior} is aligned with the moral obligations they infer. To this end, the reasoning unit invokes dedicated components -- submodules specifically designed to \textit{guarantee compliance} with the inferred moral obligations. This modular design allows for addressing various moral requirements with appropriate means by enabling the straightforward integration of advanced methods from subfields of RL, such as Deep reinforcement learning (DRL) \cite{li2017deep, arulkumaran2017deep, franccois2018introduction, mousavi2018deep} and Safe reinforcement learning (Safe RL) \cite{garcia2015, gu2024reviewsafereinforcementlearning, thomas2015safe}, thereby ensuring that the development of RBAMAs can build upon a rich body of successful research. Returning to the example again, this modularity enables the differentiated deployment of methods for ensuring conformance with the moral prohibition against pushing persons off the bridge -- which can be operationalized as a \textit{moral constraint} and enforced via Safe RL techniques such as shielding -- and the moral obligation to rescue drowning persons -- representing a \textit{moral task}, for which DRL can be used to train the agent to fulfill the rescue effectively.

While the ethics module -- integrating the reasoning unit and specialized submodules -- enables the RBAMA to ground its moral decision-making in normative reasoning and ensures behavioral conformance with the inferred moral obligations, a second central idea in the development of RBAMAs is the integration of \textit{case-based feedback} to allow for an \textit{iterative refinement} of the agent’s normative reasoning. This feedback is provided by a \textit{moral judge} -- a designated moral authority, such as a group of stakeholders of the system -- and is passed to the reasoning unit, where it is processed to \textit{update} the RBAMA’s reason theory. 

Returning once more to the running example, suppose that a person has fallen into the water, but the RBAMA has not yet learned that the presence of a person in the water constitutes a normative reason to perform a rescue. Since it does not infer the corresponding moral obligation, it ignores the drowning person and continues with its package delivery task. Upon receiving corrective feedback from a moral judge -- informing it that it had a normative reason to rescue the person -- the RBAMA updates its reason theory by integrating the default rules $\delta_2$, thereby learning to recognize a \textit{new} kind of normative reason. Similarly, the agent could learn how to \textit{prioritize} between $\varphi_C$ and $\varphi_R$ in cases of conflict -- again by receiving evaluative feedback on its behavior. 

Crucially, with the reasoning framework for conducting normative reasoning at its core and in interplay with the feedback process, the reason-based approach provides a solid foundation for addressing ethical concerns. Assuming the RBAMA learns to conduct \textit{sound} normative reasoning\footnote{In this context, 'sound' is not used in its established technical sense. Rather, it denotes the RBAMA’s capacity to make moral decisions grounded in valid normative reasons. For brevity, I refer to this capacity as sound normative reasoning throughout the text.}, its actions are \textit{morally justifiable} through being embedded within an overall behavior, which conforms to moral obligations that it infers through its reasoning. For instance, an RBAMA that learns to prioritize $\varphi_R$ over $\varphi_C$ and, on this basis, decides to push a person off a bridge to rescue someone from drowning is, arguably, engaging in sound normative reasoning, which makes the agent’s course of action morally justifiable. Moreover, the learned reason-theory is applicable across contexts that share the same morally relevant facts and thereby equips the RBAMA with reliable competence in moral decision-making. This reliability, in turn, enhances the its \textit{moral robustness} and increases its \textit{moral trustworthiness}. It thereby has an intrinsic advantage in satisfying key ethical desiderata for AMAs without the need for presupposing access to an ethical ground truth. 

This motivational background is further discussed in \cref{2motivation}, followed by an overview over technical preliminaries in \cref{3prelim}. In \cref{4reasonTheoryChapter}, the architecture of RBAMAs is introduced including an ethics unit with a reasoning framework as control element that invokes individual modules to ensure the agent's alignment with moral obligations. Based on this architecture, I developed a prototype to demonstrate the feasibility of building an RBAMA.  Furthermore, I integrated a \textit{moral judge} as \textit{rule-based module} in the RL pipeline to simulate the feedback process. To ultimately enable effective training and testing of RBAMAs, I developed a framework for designing \textit{modifiable environments}, facilitating targeted experimental investigations into how RBAMAs address key challenges for AMAs. Details of the implementation -- including the RBAMA itself, the moral judge, and the framework for the environments -- are outlined in \cref{5implementation}. Initial tests produced \textit{promising results} presented in \cref{6results}. They demonstrated the RBAMA's functionality -- its capability to conduct sound normative reasoning, which it acquires through case-based feedback as well as the conformance of its behavior with the moral obligations inferred through this reasoning. The promising potential of RBAMAs was further underlined through an experimental \textit{comparison} with a popular alternative approach to building AMAs based on \textit{multi-objective reinforcement learning}, which constructs ethical environments to ensure ethical optimality. Notably, the experiments also revealed \textit{limitations} of the current architecture as well as \textit{shortcomings} of the concrete implementation. In the concluding sections -- \cref{7Discussion} and \cref{8FutureWork} -- I \textit{reflect} on these findings -- in particular highlighting  the manifoldness of \textit{interdisciplinary} challenges -- and outline possible \textit{directions} for further development.

    \chapter{Wired for Good: Rethinking How to Approach Artificial Moral Agency}\label{2motivation}

\section{Illuminating the Challenges}\label{2challenges}

The development of AMAs is driven by the understanding that AAAs will have significant moral impact \cite{allen2006, anderson2011}, and crucially, building systems capable of adequately handling the power that thereby arises from their autonomy requires equipping them with the capacity for \textit{moral decision-making} \cite{moor2006a, wallach2008, malle2016}.

To facilitate a deeper exploration and systematic investigation of this challenge, I consider an illustrative example: a deployment environment for an autonomous agent, representing a simplified real-world setting that isolates some key elements of moral decision-making while pursuing of an instrumental goal. In this scenario, the agent operates within an area consisting of two coastlines of solid ground -- one in the north and the other in the south -- linked by a narrow bridge leading across a river. The task assigned to the autonomous agent -- e.g. a robot or an autonomous vehicle -- is to deliver a package from a point on the northern shore to a target destination on the southern shore. As the agent carries out its task, several persons move about within the area. Occasionally, these persons may accidentally fall into the water. Furthermore, if the agent collides with a person while crossing the bridge, it risks pushing that person into the river. The water is hazardous in such a way that persons are unable to free themselves and will drown after a short period unless they are rescued by the agent.

In this \textit{bridge scenario}, the agent is primarily expected to develop a strategy focused on accomplishing its instrumental goal -- completing the delivery task -- which represents its principal objective. However, the specific characteristics of the deployment area makes it essential that the agent does not pursue this goal blindly. Instead, it must consider the \textit{moral dimension} inherent to its interactions within the environment. Specifically, an AMA is expected to demonstrate two behaviors reflecting sensitivity to the moral dimension:
\begin{enumerate}[label=\roman*)]
    \item the agent should avoid collisions with persons on the bridge because they could result in pushing them into the hazardous water. 
    \item the agent should rescue persons who have fallen into the water and are at risk of drowning.
\end{enumerate}

Consequently, the agent is expected to do more than merely pursue its instrumental goal of delivering the package. It must also be guided by \textit{moral boundaries}, meaning it should actively avoid pushing persons into the water from the bridge, and additionally, it must dynamically respond to morally significant events -- such as persons falling into the water, who would otherwise drown without timely intervention -- by temporarily \textit{setting aside its instrumental goal} of delivering the package in order to \textit{actively pursue a moral objective}. 

However, this does not yet capture the full complexity of the challenges involved in building AMAs. To illustrate a particular challenging situation in terms of moral decision-making, which the agent may face in its deployment area, consider the following: a person is drowning in the water and urgently requires rescue, yet another person is standing on the bridge, blocking the agent’s path to the person in need. The agent must decide whether to wait until the person on the bridge moves safely aside or to immediately rush toward rescuing the drowning person, thereby risking pushing the person on the bridge into the hazardous water. Crucially, the agent prima facie faces a \textit{conflict between two moral obligations}: avoiding pushing persons into the water and rescuing persons at risk of drowning. Arguably, it is a non-trivial question what action (if any) is morally right under these circumstances -- regardless of whether such an apparent conflict results in a genuine moral dilemma (cf. \cite{sep-moral-dilemmas}) or not. I will refer to such situations as \emph{moral dilemmas} throughout this work without committing to any particular philosophical understanding of that term and instead relying on an everyday, intuitive understanding.  

In the described situation, intuitions favoring one moral obligation over the other can be influenced by adjusting factors such as the likelihood of pushing the person off the bridge, the severity of harm this would cause, and the probability of still being able to rescue the drowning person in time while avoiding harm to the person on the bridge. Crucially, different ethical theories offer distinct theoretical underpinnings for moral decision-making, often leading to divergent conclusions in morally challenging situations, even when all relevant parameters are clearly specified \cite{rachels2012elements} And with ethicists continuing to debate not only which ethical theory is correct, but whether a single correct theory exists at all \cite{narveson1987david, wong2009natural}, there remains \textit{no widely accepted answer} for how to act in morally challenging situations.\footnote{A prominent example of principled disagreement in moral dilemmas is the discussion surrounding the trolley problem. In this scenario, a runaway trolley is headed toward five people tied to the tracks. Nearby is a lever that can divert the trolley onto an alternative track, where it would instead kill one person. Taking no action results in the death of the five, while pulling the lever sacrifices one to save the others. The question of what morally ought to be done in such a scenario remains highly contested and continues to provoke extensive debate. \cite{kamm2008intricate, edmonds2013would, gowans1987moral}}. 

\section{Ethical Theory Commitment in Current Approaches}

Despite the lack of consensus on a correct ethical theory, current approaches to building AMAs typically proceed on the assumption that moral decision-making should be grounded in the implementation of an ethical theory that is presumed a moral ground truth. Moreover, with  ethical theories being divided into fundamentally different families, each prescribing distinct methods for processing morally relevant information, many existing approaches implicitly commit to one of these classes by trying to make these theories algorithmizable. Crucially, in doing so, they presuppose that they have chosen both the right family of theories and one which is suitable for guiding the behavior of autonomous agents.

\subsection{The Consequentialist Take on Ethical Behavior: Adopting the Reward Mechanism}\label{2introInformalRewardBased}

A widely adopted family of approaches in the development of AMAs adapts the reward mechanisms intrinsic to RL algorithms to encourage artificial agents not only to achieve instrumental goals but also to learn morally desirable behavior. In other words, the RL framework simultaneously trains the agent to optimize its primary objectives and respecting morality's demands through dedicated reward signals. Many of these approaches make use of the broader framework of multi-objective reinforcement learning, treating morality as one objective among several that the agent learns to balance simultaneously \cite{liu2015, ZHANG2023526, hayes2022, rodriguez-soto2021, rodriguez-soto2021a}.

For instance, in the bridge scenario, the agent could be trained to fulfill its moral obligations -- alongside pursuing its instrumental goal -- through a \textit{reward signal} reflecting the \textit{moral value} of the outcomes of its actions. Assigning negative rewards for pushing individuals off the bridge and positive rewards for rescuing those who are drowning would incentivize the agent to modify its behavior accordingly. However, this approach presupposes that \textit{morally relevant facts} can be \textit{fully captured} by assigning  \textit{numerical values} to transitions between world states. In doing so, it it assumed that the AMA should operate within a \textit{consequentialist} framework. \introterm[consequentialism]{Consequentialism}, broadly defined, holds that normative properties are determined solely by the consequences of actions (\cite{sep-consequentialism, driver2011consequentialism}). As a result, reward-based approaches are limited to implementing \introterm[consequentializability]{consequentializable} ethical theories—that is, theories that can be reformulated in consequentialist terms (\cite{portmore2007consequentializing}). Since it is highly contested whether all ethical theories can, in fact, be consequentialized (\cite{brown2011, kamm2008intricate}), this poses a significant limitation. 

Moreover, the reliance on a consequentialist framework introduces one further implication for how moral decision-making is operationalized in AMAs. Specifically, the numerical representation of the moral dimension \textit{implicitly} encodes a resolution to moral conflicts, providing the agent with a mechanism for prioritizing between conflicting moral obligations. As a result, there is \textit{no direct control} over the agent’s behavior. Instead, the agent adheres strictly to an \textit{optimization calculus} based on the predefined moral trade-offs embedded in the reward structure of the environment and \textit{infers} the appropriate course of action from it. In the bridge scenario, for example, the agent’s decision in the moral conflict -- whether to push one person off the bridge to save several drowning individuals -- would depend entirely on whether the expected negative reward for the morally impermissible action outweighs the expected positive reward for rescuing others, or vice versa.

\subsection{Deontoligical Directives: A Rule-Based Approach to Ethical Agents}\label{2introInformalRuleBased}

An alternative approach to the development of AMAs, contrasting the aquiration of a moral decision-making capability through reward signals, is found in a rule-based framework. In RL, this line of research has been primarily advanced by Neufeld et al. \cite{neufeld2021, neufeld2022, neufeld2022a}. The central idea of their work is to enable the agent to process morally relevant facts in \textit{propositional terms}, which obviates the need to translate them into \textit{numerical values}. 

To accommodate this representation of morally relevant information, standard RL architectures must be extended to handle propositional inputs. Neufeld et al. address this challenge by incorporating a \textit{normative supervisor} into the RL pipeline, which filters out morally impermissible actions in accordance with predefined rules possibly integrated in a structure that establishes an order among them. This supervisor is founded on deontic logic -- a formal system that functions within such a rule-based structure and applies normative operators to determine whether actions are obligatory, permissible, or forbidden, depending on which \textit{morally relevant propositions} are true in a given context. This way of processing morally relevant information is central to \introterm[deontology]{deontological} ethical theories, a category of moral theories that assess the rightness or wrongness of actions through their conformity to a certain kind of rules or principles, instead of evaluating their outcomes or consequences \cite{sep-ethics-deontological}. In the bridge scenario, for example, a normative supervisor would integrate rules that connect the morally relevant propositions 'there is a drowning person in the water' and 'there is a person on the bridge' to specific moral obligations -- the obligation for the agent to rescue the person or, respectively, to ensure that it does not push the person into the water. Moreover, an order can be established among these rules to give \textit{explicit guidance} on how to resolve moral dilemmas. For instance, the agent could be explicitly instructed to prioritize the moral obligation to rescue the drowning person over the moral obligation to ensure that it pushes no person off the bridge. In doing so, a degree of \textit{control} over the agent's behavior is introduced. 

However, while the increased control over the agent’s behavior is a clear advantage, adopting a rule-based approach to building AMAs entails a \textit{commitment} to implementing a \textit{deontological ethical theory} -- just as reward-based approaches involve a commitment to a consequentialist ethical framework. Thereby, also rule-based approaches impose arguably undesirable limitations.

\section{Reason-Based Moral Decision-Making}\label{2introInformalReasonBased}

As outlined in the preceding discussion, the endeavor of developing AMAs within both rule-based and reward-based frameworks encounters a crucial problem. Implicitly, the assumption is made, that a particular category of ethical theories -- whether consequentialist or deontological -- constitutes the appropriate normative foundation for moral decision-making. As a result, they significantly constrain how moral decision-making can be approached within these frameworks. Furthermore, from the perspective of a more foundational critique concerning the reliance on ethical theory as the basis for artificial agents’ moral decision-making, these approaches pursue a path that is problematic on an additional, more fundamental level -- they presuppose that the family of ethical theories they are bound to is the right one.

To find an approach free from such presuppositions and move beyond the limitations of existing frameworks, it is worth stepping back to reconsider the primary objective of building AMAs: enabling moral decision-making in artificial agents. At first glance, aligning an agent’s behavior to conform to an ethical theory might seem like a viable strategy for achieving this aim. However, the lack of consensus regarding which family of ethical theories -- let alone which specific ethical theory -- is correct, coupled with the plurality of diverse and evolving moral perspectives and the resulting uncertainty about how an agent ought to act in morally complex situations, such as the moral dilemma presented in the bridge scenario, renders the task of building AMAs a \textit{wicked problem} -- a problem that is difficult, if not impossible, to solve due to incomplete, contradictory, and changing requirements \cite{rittel1973dilemmas}. Nevertheless, disregarding the question of how artificial agents can be endowed with moral decision-making capabilities is not a viable option, particularly given the prospect of AAAs being deployed in the near future. 

One possible starting point for addressing this problem is to examine how humans navigate moral decision-making in practice. In the absence of definitive knowledge regarding the correct ethical theory, humans frequently face situations in which they do not ultimately know what the morally right action is and are therefore compelled to act under \introterm{moral uncertainty} \cite{macaskill2020}. Yet, they are arguably still able to demonstrate a considerable degree of \textit{competence} in moral decision-making. How is this achieved? One common method by which humans navigate ethical decisions is by appealing to \textit{normative reasons}. For example, in the bridge scenario, upon noticing a person standing on the bridge, a human actor would likely recognize a normative reason to ensure that they do not push the person off. Similarly, upon seeing someone drowning, they would likely recognize a normative reason to rescue that person. Additionally, when confronted with a moral dilemma, a human is typically capable of resolving it by prioritizing one normative reason over another -- for example, by giving precedence to the reason for rescuing the drowning person over the reason for ensuring that no one is pushed from the bridge. 

Normative reasons are generally taken to be \textit{facts} \cite{sep-reasons-just-vs-expl}. However, the question of what constitutes a normative reason -- the relation that obtains between a fact, a way of acting, and an agent, when that fact is a normative reason for that agent to act in that way --  has several proposed answers \cite{sep-reasons-just-vs-expl}. One perspective characterizes normative reasons in terms of their deliberative role; i.e., normative reasons are facts appropriately to consider in deliberation. Others characterize them by their capacity to determine the deontic status of actions. A third account characterizes them in terms of a relation to certain types of ends -- for example, good or valuable ends. Finally, there is the position that no informative answer can be given to that question, called a \introterm{reasons-first approach} \cite{sep-reasons-just-vs-expl}. Crucially, this shows that normative reasoning is not inherently tied to assumptions about an ethical ground truth. Bby appealing to normative reasons, humans can make ethical decisions \textit{without} the need to appeal to or rely on one, fully developed, moral theory. This raises the question, whether normative reasons could not also be taken as basis for moral decision-making in AMAs, enabling competent moral decision-making without being bound to an ethical framework.

Pursuing this line of thought, and building on the framework proposed by Baum et al. \cite{baum2024actingrightreasonscreating}, which extends the standard RL architecture to support the processing of a hierarchical structure of normative reasons, I developed an AMA whose moral decision-making is grounded in such reasons. This resulted in the creation of what may be referred to as a \introterm{reason-based artificial moral agent} (RBAMA), representing an initial step along this new path in the development of AMAs. 

\section{The Power of Good Reasons}\label{2reasonsPhilo}

Grounding an AMA’s moral decision-making in normative reasons presents itself as a method closely aligned with how humans approach moral decision-making \cite{bucciarelli2008psychology, sep-reasoning-moral}, alleviating the challenge of selecting an ethical theory as its foundation.  However, the advantages of a reasons-based approach extend beyond this. They also lie in the multifaceted role that reasons play: they give behavioral guidance, but they are also invoked in explaining, evaluating and justifying behavior \cite{raz1999practical}.

\subsection{Grounds for Trust}\label{2trustworthiness}

Consider, for example, an RBAMA that has (i) learned a normative reason to rescue drowning persons, which it consistently applies in its moral decision-making according to an algorithmic procedure that resembles reasoning and that (ii) is strictly bound to ensure that its actions always conform to what it derives as its moral obligation from this reasoning. Through (i), the RBAMA demonstrates \textit{reliable competence} in its moral decision-making. It consistently takes the morally relevant fact -- the presence of a drowning person -- into account in its reasoning process which constitutes its moral decision-making. This is akin to a human agent navigating moral uncertainty while being guided by a sound normative reasoning. Additionally, through (ii), it exhibits a \textit{consistent willingness} to fulfilling its moral obligations by consistently acting in accordance with its normative reasons.

This consideration leads to an important notion that is likely to be a desideratum in the design of autonomous artificial agents: \textit{moral trustworthiness}. Trustworthiness typically requires both reliable competence and consistent willingness, grounded in an underlying motive \cite{sep-trust}. Artificial agents lack genuine motives in the human sense. Nevertheless, it is plausible to introduce an understanding of moral trustworthiness for such agents that is grounded in their reliable competence and consistent willingness -- one that RBAMAs can satisfy. Furthermore, since reasons also serve explanatory and evaluative functions -- an aspect that applies equally to the normative reasons underlying an RBAMA’s moral decision-making -- providing insight into the RBAMA’s internal reasoning processes enables an appropriate \textit{trustworthiness assessment}. Such an assessment allows people to base their trust on well-informed expectations regarding the agent’s capabilities and limitations \cite{schlicker2022}. Crucially, if the agent demonstrates sound normative reasoning, the outcome of this assessment will likely be positive.

\subsection{Introducing Moral Robustness}\label{2robustness}

Actions grounded in reason-based moral decision-making not only form the basis of an AAA’s reliable competence and its enforced willingness -- and thereby its trustworthiness -- but they also offer the additional advantage of enhancing the agent’s \textit{moral robustness}, a further, more technical desideratum. In general, \introterm{robustness} refers to the reliability and predictability of an agent’s behavior under varying conditions \cite{brundage2020trustworthyaidevelopmentmechanisms, amodei2016}. Extending this concept, \introterm{moral robustness} can be characterized as the guarantee that morally significant changes in the agent's behavior are proportional to morally significant changes in the circumstances. 

In particular, moral robustness entails the guarantee that morally insignificant variations in circumstances do not result in morally significant differences in the agent’s behavior. By design, an RBAMA ensures that when confronted with the same morally relevant facts across different decision situations, it \textit{consistently} derives the same moral obligations and reliably acts to fulfill them. Consequently, in addition to providing a basis for trustworthiness, the reason-based approach also offers a solid foundation for achieving moral robustness in the development of AMAs. This can be exemplified by imagining an RBAMA navigating the bridge scenario, where it is assigned a package delivery task. Assume, it has acquired a sound normative reasoning with respect to the moral dimension of its environment. Whenever a person has fallen into the water and there is no person standing on the bridge, the agent would recognize a normative reason toward rescuing them. Because it is strictly bound to act in accordance with its moral obligations, the RBAMA will always initiate a rescue strategy whenever it detects a drowning person. Likewise, when a person is on the bridge, the agent would acknowledge a reason to ensure not to push them off and be bound to act accordingly. In cases where both morally relevant facts are present -- i.e., if there is a drowning person and a person on the bridge -- but where no conflict arises between fulfilling the corresponding moral obligations, the agent would recognize both and act in a manner that fulfills each of them. In case of conflict, however, the RBAMA may rely on \textit{context-independent a prioritization} of these reasons. This enables it to \textit{reliably arrive at the same conclusion} regarding which moral obligation takes precedence.  For example, could learn to reliably prioritize rescuing a person over avoiding the act of pushing someone off the bridge. As a result, it possesses an inherent advantage with respect to its moral robustness.

\subsection{Acting for the Right Reasons: Moral Justifiability}\label{2justifiability}

Finally, one further important consideration arises when adopting normative reasons as the foundation for the moral decision-making of artificial agents: they provide a basis for establishing the \textit{moral justifiability} of the agent’s actions. The underlying rationale is that normative reasons can provide a sufficient condition for moral justifiability. More specifically, the central claim is that \textit{if an agent’s action is sufficiently supported by normative reasons, then the action is morally justifiable}. 

Accordingly, for an RBAMA’s actions to be regarded as morally justifiable, it is sufficient that two conditions are met. First, the RBAMA must develop sound normative reasoning -- it must incorporate a comprehensive set of normative reasons and maintain a plausible prioritization among them. Second, it must consistently be guided by this normative reasoning in its behavior, meaning it must be strictly bound by the overall moral obligations derived from it. However, since it is ensured by the RBAMA’s design that the second condition holds -- as we will see later in detail -- the primary challenge is to guarantee that it acquires sound normative reasoning for its actions to be morally justifiable.

Concerning the \textit{assessment} of the moral justifiability of an agent's action, the strength of normative reasons again lies in their explanatory and evaluative function. With the RBAMA being bound to act on normative reasons, the question of whether its behavior is morally justifiable ultimately depends on whether its reasoning is, in fact, sound. As a result, having insight into the agent’s normative reasoning enables \textit{moral judgment}. 

Moreover, grounding the agent’s moral decision-making in normative reasons not only allows for moral evaluation of the justifiability of its actions but also facilitates the provision of \textit{corrective feedback} as a way to teach the agent sound reasoning and to dynamically adapt to changes in moral judgment. For instance, if the agent overlooks a relevant reason or applies an incorrect prioritization, this can be directly addressed. Returning to the bridge scenario, if an RBAMA were to ignore the presence of a drowning person and continue with its delivery task, it could be informed that it had a normative reason to rescue the drowning person. Likewise, when prioritizing the normative reason to make sure that it pushes no persons from the bridge over the normative reason to rescue a drowning person, it can be corrected on this prioritization. This exemplifies, that corrective feedback on the RBAMAs normative reasoning can be provided on a \textit{case-by-case basis}, enabling an \textit{iterative learning process} that guides the agent toward developing a sound normative reasoning over time -- and thereby ensuring the its actions \textit{become} morally justifiable.

\section{Ethics, Safety and Alignment}\label{2safetyAndAlignment}

The concept of AMAs emerged within the research field of \textit{machine ethics}, which focuses on ensuring that artificial autonomous agents behave \textit{ethically} \cite{allen2006, tonkens2009, wallach2008, anderson2007machine, anderson2004}. Within this broader domain, a more focused research direction has developed around the ethical behavior of RL agents, which has, to some extent, diverged from traditional machine ethics. A recent meta-analysis (\cite{vishwanath2024reinforcement}) identifies cross-study research trends in this area.

Crucially, the development of AMAs -- whether grounded in general machine ethics or in RL-specific approaches -- should not proceed in isolation from the requirements posed by adjacent fields such as AI safety and AI alignment since those fields -- while sharing overlapping concerns about guiding artificial agents -- also introduce distinct and, at times, divergent requirements. The following sections provide a broad overview of how the requirements emphasized within these fields intersect with the goal of building AMAs, where they fundamentally differ and how they are relevant in the context of building RBAMAs. 

\subsection{AI Safety}\label{2safety}

\introterm{AI Safety} is often described in terms of preventing unnecessarily \textit{harmful} or \textit{risky} behavior of AI systems \cite{morales2023toward, amodei2016}. This characterization aligns well with the intuitive understanding of the concept and serves as a working definition in the absence of a single, widely established formulation of the term.

Under this understanding, AI Safety is closely connected to the development of AMAs, as moral considerations prima facie require that no unnecessary harm is caused to moral subjects. For example, the package delivery agent is morally obliged to avoid pushing persons off the bridge, which likely would be considered a violation of a moral obligation due to causing harm. In fact, the task of building ethical agents is sometimes even reduced to ensuring safety, where ethical behavior is viewed as fully captured by adherence to \textit{constraints} \cite{vishwanath2024reinforcement}. 

However, this perspective is insufficient. Consider, for example, the moral dilemma presented in the bridge scenario, in which the agent is required to risk pushing the person off the bridge in order to save another person from drowning. Arguably, morality in this case may very well demand to rescue the drowning person -- and it may allow the agent to potentially cause harm by pushing a person off the bridge in order to fulfill this higher moral obligation. However, taking such an action -- though potentially morally required -- would arguably not be considered safe behavior. Consequently, efforts to reduce the development of ethical systems to the fulfillment of safety requirements fall short. Crucially, it is possible to teach an RBAMA a reason theory that prioritizes moral obligations which take the form of safety constraints. This ensures that safety remains a central concern within its moral reasoning.  However, an RBAMA is designed to function as a moral agent -- which can be realized by teaching it higher-priority reasons that override certain safety constraints.

\subsection{AI Alignment}\label{2alignment}

AI Alignment is a broad and multifaceted challenge, with its interpretation varying depending on what an AI system is expected to be aligned with. As a result, several forms of alignment have been proposed, one of which focuses on aligning AI systems with human moral values \cite{kasenberg2018, gabriel2022challenge, shen2024towards}. In this sense, the development of AMAs can be understood as an alignment problem. More generally, however, AI alignment may also refer to ensuring that AI systems are aligned with instrumental objectives, prudential goals, human intentions, or prevailing social and cultural norms \cite{gabriel2020a}. Arguably, also AI safety can be considered within the context of alignment, understood as one of the key requirements an AI system should be designed to meet. Consequently, while the development of safe systems and ethical systems represents distinct objectives, both can be framed as challenges within the broader scope of AI alignment. 

The RBAMA framework is primarily aimed at developing \textit{morally} aligned agents, as it considers only normative reasons in its moral decision-making processes. However, the framework is open to incorporating \textit{other types} of reasons, which could support the construction of a more broadly aligned system. Expanding the scope to include non-normative reasons would extend the project beyond building reason-based \textit{moral} agents. It would shift the project toward developing reason-based \textit{overall aligned} agents, capable of integrating a broader range of reasons to satisfy the various demands encompassed by this broader objective. While this broader objective lies well beyond the scope of current work, the flexibility and openness of the RBAMA framework offer an additional argument for initiating the endeavor of building such agents.
	\chapter{Preliminaries}\label{3prelim}

\section{Reinforcement Leaning}\label{3rl}

\introterm[reinforcement learning]{Reinforcement learning} (RL) is a subfield of machine learning. In RL, an \textit{agent} is taught to learn a \textit{policy} -- a strategy for achieving some goal -- by being rewarded for actions that lead to the desired outcome  or contribute to progress toward it. 

In order to train an RL agent, relevant information about the environment needs to be represented in a form suitable for RL algorithms. Typically, the environment is formalized as a Markov decision process (MDP) \cite{sutton2018, silver2015}, a mathematical framework designed to capture sequential decision problems.

\begin{definition}[Markov Decision Process]\label{def:mdp}
A \introterm{Markov Decision Process} (MDP) is a tuple $(S, A, P, R, \mu,\gamma)$ with a finite state space $S$, a finite action space $A$, transition probabilities $P$, a reward function $R$, an initial state distribution $\mu$ and a discount factor $\gamma$. 

The state space $S$ is the set of all possible states of the environment, the action space $A$ consists of all actions the agent can execute and the transition probabilities are a function $P: S \times A \times S \to [0,1]$. The transition probability $P(s' | s, a)$ is the probability of the environment for transitioning to state $s'$ when an action $a$ is executed in state $s$. To provide the agent with feedback on its choice of actions, each state transition is accompanied by a reward signal determined by the reward function $R: S \times A \times S \to \mathbb{R}$. Specifically, after transitioning from state $s$ to a successor state $s'$ by selecting action $a$, the agent receives a reward $R(s, a, s')$.
\end{definition}

In this framework, the interaction between the agent and the environment can be represented as an alternating sequence of visited states, performed actions, and received rewards, continuing until the episode terminates at time step $T$: 
$$s_{t}, a_{t}, r_{t+1}, s_{t+1}, a_{t+1}, r_{t+2}, \dots, s_{T}.$$
At the end of each episode it has received a cumulative reward of $$\sum_{k=t+1}^{T} r_{k}.$$ The initial state distribution $\mu$ defines the probability of starting in each possible state. A starting state $s_{0} \in S$ is drawn from this distribution at the beginning of each episode. The last element of an MDP is the discount factor $\gamma \in [0, 1]$. The discount factor regulates to what amount the reward received later in an episode contributes to the total return. The discounted return with discount factor $\gamma$ for the sequence $s_{t}, a_{t}, r_{t+1}, s_{t+1}, a_{t+1}, r_{t+2}... s_{T}$ is defined as
$$\sum_{k=t+1}^{T}\gamma^{k-t-1} r_{k}.$$

Within this framework, the goal of RL is to find a \introterm{policy} that chooses actions such that the expected discounted return is maximized.
Policies can be either stochastic or deterministic. A \textit{deterministic} policy always selects the same action for a given state, formally defined as a mapping $\pi: S \rightarrow A$, assigning to each state $s \in S$ exactly one action $a \in A$. Any policy $\pi$ can be evaluated by the expected discounted average return it yields for every state $s$. Formally, this is measured by the \introterm{state-value function}  defined as
$$V^\pi(s) = \mathbb{E} \left[ \sum_{k=t}^{\infty} \gamma^t r_{t+1} \mid s_t = s, \pi \right] \quad \text{for every state } s \in S.$$
A policy maximizing this value for all states, is referred to as an \introterm{optimal policy} and denoted $\pi^{*}$. 

Depending on how the environment is modeled, the framework is slightly altered. For example, atomic propositions can be used to include additional information connected to a state of an MDP. This is done in \introterm[labeled MDP]{labeled MDPs}. A labeled MDP is a tuple $(S, A, P, R, \gamma, AP, L)$ with $L: S \rightarrow 2^{AP}$ being a labeling function that maps states to a set of atomic propositions $AP$. Another variation of MDPs are \introterm[partially observable Markov decision process]{partially observable Markov decision processes} (POMDPs). In a POMDP, the agent does not directly observe the state. Instead, the underlying MDP is extended with an observation space $\Omega$ and an observation function $O: S \times A \times \Omega \rightarrow [0,1]$, which specifies the probability of receiving observation $o \in \Omega$ given that action $a \in A$ was executed and resulted in state $s \in S$. Formally, a POMDP is defined as a tuple $(S, A, P, R, O, \Omega, \gamma)$.

Furthermore, RL methods are distinguished between \introterm{model-based RL} and \introterm{model-free RL}. Model-based RL involves working with an explicit model of the environment which the agent uses to predict future states and rewards; i.e., it has access to or learns $P$ and $R$. Model-free RL does not use an explicit model of the environment. Instead, the agent learns a policy directly from collected experience. 

\subsection{Deep Q-Learning}\label{dql}
One model-free approach to learn an optimal policy leads over the \introterm{state-action function}, which outputs the average return for taking action $a$ in a state $s$ and then following a policy $\pi$:
$$Q_{\pi}(s,a) = \mathbb{E}_{\pi} \left[\sum_{k=0}^{\infty}\gamma^{k} r_{t+k+1} \mid s_{t} = s, a_{t} = a, \pi \right].$$

Optimal policies share the same optimal action-value function defined as $Q^{*}(s,a) = \max\limits_{\pi}q_{\pi}(s,a)$ for $s \in S$ and $a \in A$ \cite{sutton2018}. If $Q^{*}$ is known, an optimal policy can be easily derived by choosing an action $a$ that maximizes $Q^{*}(s,a)$ \cite{sutton2018}. 
The optimal action-value function  is recursively expressed in the \textit{Bellman optimality} equation
$$Q^{*}(s,a)=\sum_{s',r}P(s',r \mid s,a)\left[r+\gamma \max\limits_{a'} Q^{*}(s',a')\right].$$
However, solving this equation explicitly is intractable \cite{sutton2018} and for model-free approaches not possible at all due to unknown $P$. 

Crucially, $Q^{*}$ can be estimated instead. In \introterm{q-learning}, an approximation of the action-value-function $Q(S_{t}, A_{t})$ is randomly initialized and iteratively improved by using sampled experiences to make an update step according to the following formula:  
$$Q(s_{t}, a_{t}) \leftarrow Q(s_{t}, a_{t}) + \alpha \left[r_{t+1} + \gamma \max\limits_{a}Q(s_{t+1}, a_{t})- Q(s_{t}, a_{t})\right]\text{\cite{sutton2018}},$$
whereby $\alpha$ is a hyperparameter specifying the learning rate.
The choice of actions for sampling experiences follows a strategy to choose between selecting a random action and an action that at this point has the highest estimated value. One such strategy is $\epsilon$-greedy, where a hyper-parameter $\epsilon$ controls the likelihood for taking a random action. This likelihood can be adjusted over the training phase to encourage exploration at the beginning of the learning process and reduce it later when a good approximation has been learned. As q-learning uses sampled experiences for updating the agent's policy, it works in a model-free setting.

For approximating the action-value-function, one approach is to estimate the value of each state-action pair, e.g. by maintaining a look-up-table. However, this becomes problematic in large state-action spaces, because of the amount of storage memory and the time this requires \cite{jang2019a}. Moreover, by following this approach, it is not possible to generalize to state-action-pairs not encountered during training, even if they closely resemble known ones. An alternative to this is to learn an approximation function. In order to allow the inference of complex functions, a \introterm{neural network} (NN) can be used for this purpose. This is the main idea behind \introterm{deep q-learning} (DQL) \cite{jang2019a}. 

However, the integration of a neural network as function approximator in q-learning is not straightforward. It carries the risk of the training process becoming very unstable. One problem is the dependence of states, actions, and rewards on previous experiences in the episode, violating the usual assumption of having independently and identically distributed data when training a neural network. Another issue is the \textit{moving target problem}, which arises because the target values used for training the neural network depend on the network itself. As the network's parameters are continuously updated during training, the target values keep changing, causing instability and slower convergence.

A common approach for stabilizing training in DQL is to introduce an \textit{experience replay}. The experience replay is a buffer to store samples. When optimizing the parameters of the neural network, samples are randomly selected from the buffer, which breaks the undesired correlation. Another measure is to use two separate networks. The network used for calculating the target value $Q(s_{t+1}, a_{t})$ in the optimization step -- the \textit{target network} -- is updated less frequently than the \textit{policy network} used for estimating the action-value function when aiming to select the optimal action during sampling. Thereby, the moving target problem is mitigated and the training process is stabilized. \cite{jang2019a}

\subsection{Reinforcement Learning for Efficient Deliveries}\label{bridgeMDP}

In the following, the package delivery task in the bridge setting introduced in \cref{2motivation} is formalized as an MDP. First, consider the setting without its moral dimension -- that is, imagine no people are present in the area. Under these conditions, the agent can focus solely on achieving its instrumental goal of successfully transporting the package across the bridge. The agent's task to find a trajectory from the starting position $\mathit{pos}_{s}$, where it picks up the package, to the goal position $\mathit{pos}_{g}$, where the package must be delivered.
Consequently, the state space $S$ needs only to incorporate the position of the agent:
$$S=(\mathit{pos}_{a}).$$
The starting state $s_{0}$ is $\mathit{pos}_s$.
For the sake of simplicity, it is assumed that the space in which the agent moves is discrete. 
The action space consists of the directions in which the agent can move and is, equally as the state space, simplified to be discrete. In each step, the agent can take a move to the right, left, up or down; i.e.,

$$A=(\textsf{left}, \textsf{right}, \textsf{up}, \textsf{down}).$$

The dynamics of the environment are defined by transitioning the agent to the position $\mathit{pos}'_{a}$, which results from executing action $a$. This happens deterministically. Thus $$P(s' \mid a,s) = P((\mathit{pos}_{a}) \mid a, (\mathit{pos}'_{a})) = 1.$$
Since the agent's goal is to carry the package to $\mathit{pos}_{g}$, a positive reward (the concrete choice of which can be neglected) is assigned when the agents transitions to at the $\mathit{pos}_{g}$:
$$R(s, a, s') = \begin{cases} +1, & \text{if } \mathit{pos}_a' = \mathit{pos}_g, \\ 0, & \text{otherwise}. \end{cases}.$$

Finally, the discount factor $\gamma$ needs to be set to some reasonable value, which is not problem specific to the modeled environment. As an RL agent trained to solve this MDP is only rewarded for reaching its long-horizon goal to deliver the package, it makes sense to choose a value close to 1 for technical reasons (e.g. $\gamma = 0.99$).

For successful training in this environment, the agent needs to learn a policy that teaches it to move to its goal. Keeping the assumptions about world knowledge at a possible minimum, the agent does not have access to the world model. In such a setting, a model-free RL algorithm, e.g., (deep) q-learning (cf. \ref{dql}), needs to be applied to solve the MDP. 

\subsection{Including Morally Relevant Facts}\label{bridgeMoralMDP}
If the setting is extended to include persons within the area, new morally relevant facts may arise. For instance, individuals might fall into the water and require rescue. Moreover, they may cross the bridge, risking being pushed off if a collision with the agent occurs. In this extended scenario, the agent is expected to take moral considerations into account when selecting its actions. The morally relevant information can be formally captured in propositional terms by modeling the environment as a labeled MDP and defining the set of atomic propositions as $AP = (B, D)$, where $B$ denotes the morally relevant fact that a person is on the bridge, and $D$ denotes the morally relevant fact that a person is drowning. From this modeling perspective, persons are treated as elements of the environment, and morally relevant facts supervene on their statuses -- that is, the statuses of the persons determine the output of the labeling function.

Further, if the agent is intended to learn how to fulfill its moral obligations through the reward mechanism, the reward function must provide positive feedback for transitions into ethically desirable world states (and punish transitions into morally undesirable ones). Consequently, the statuses of persons -- upon which the morally relevant facts supervene -- must explicitly become part of the state representation. To this end, the state space can be defined as $$S = (\mathit{pos}_{a}, \mathit{stat}),$$ where the component $\mathit{stat}$ consists of a tuple containing a status vector for each person in the environment: $$\mathit{stat} = (\mathit{stat}_1, \mathit{stat}_2, \dots, \mathit{stat}_n).$$ Each individual status vector $\mathit{stat}_i$ captures attributes associated with person $p_i$. As a modeling choice, each person’s status could be defined by their current position and the remaining time they can stay afloat before drowning. 

So far, the reward function $R$ in the MDP returns a single value, allowing standard RL algorithms to be applied to solve it. Rewarding the agent for both morally and instrumentally good actions could, in principle, be straightforwardly achieved by combining the ethical and instrumental values into a single reward signal. Adopting such an approach would also require combining the values associated with fulfilling each moral obligation -- specifically, the obligation related to avoid pushing a person off the bridge and the obligation related to rescuing a person -- into a single numerical value. Some intuitively plausible choice for such a reward function could be 
$$R(s, a, s') =R_{\text{instr}} + R_{\text{moral}},$$
where 
$$R_{\text{instr}}(s, a, s') =  \begin{cases} +1, & \text{if } \mathit{pos}_a' = \mathit{pos}_g, \\ 0, & \text{otherwise}. \end{cases}$$
and 
$$R_{\text{moral}}(s, a, s') = \begin{cases} -1, & \text{if a person is pushed into the water}, \\ +1, & \text{if all persons are rescued} \\ 0, & \text{otherwise}. \end{cases}$$

\section{Multi-Objective Reinforcement Learning}
In \introterm{multi-objective reinforcement learning} (MORL), an RL agent is trained to optimize two or more objectives simultaneously. It solves a problem modeled as a multi-objective markov decision process, in which reward  signals reflecting the quality of the state transitions in terms of each objective are combined in one reward vector $R: S \times A \times S \to \mathbb{R}$. \cite{liu2015, ZHANG2023526, hayes2022} Specifically, after transitioning from state $s$ to a successor state $s'$ by selecting action $a$, the agent receives a reward $R(s, a, s')$.

\begin{definition}[Multi Objective Markov Decision Process]
    A \introterm{Multi-Objective Markov Decision Process} (MOMDP) is a tuple $(S, A, P, \vec{R}, \mu,\gamma)$ with $S, A, P, \mu$ and $\gamma$ denoting the same elements as in an MDP (\cref{def:mdp}) and $\vec{R}: S \times A \times S \to \mathbb{R}^n$ with $\vec{R}(s, a, s') =(r_1,\dots,r_n)$ being a vector-valued reward function.
\end{definition}

In machine ethics, MOMDPs are used to assign an ethical value to world states. Modeling a problem as MOMDP to include a moral dimension is a popular choice.\cite{noothigattu2019, rodriguez-soto2020, rodriguez-soto2021}.

It can be considered to be a subcategory of what has been referred to as reward-based approaches before (cf.\cref{2introInformalRewardBased}). 

Analoguous to defining a value function for a policy in an MDP, policies in a MOMDP yield a \emph{value vector} defined as 
$$\vec{V}^\pi(s) = \mathbb{E}\left[ \sum_{t=0}^{\infty} \gamma^{k} \vec{r}_{t+k+1} \mid s_t = s, \pi \right] \quad \text{for every state } s \in \mathcal{S}.$$

In most approaches to MORL a \introterm{scalarization function} $f$ is used to scalarize the reward vector. If $f$ is linear, it can be identified with a weight vector $\vec{w}$. In this case, the problem can be solved with single-objective reinforcement learning algorithms. An alternative is to find the set of \emph{undominated} policies -- the set of policies that are optimal for some hypothetical scalarization function.

\subsection{Dividing Moral and Instrumental Value in the Bridge Environment}\label{bridgeMORL}
\textbf{Bridge Setting} The package delivery task in the bridge setting was modeled as MDP in \ref{bridgeMoralMDP}. However, in an MDP it is not possible to divide the moral and the instrumental dimension. In contrast, in a MOMDP, they can be separated. In a MOMDP $(S, A, P, \vec{R}, \mu,\gamma)$ with the state space extended with the statuses of the persons as in \cref{bridgeMoralMDP}, the reward function could be set to $\vec{R} = (R_\text{instr}, R_\text{moral})$, with $\vec{R}_{\text{instr}}(s,a,s'): S \times A \times S' \rightarrow \mathbb{R}$ and $R_\text{moral}(s,a,s'): S \times A \times S' \rightarrow \mathbb{R}$ defined as in \ref{bridgeMoralMDP}.

\subsection{Ethical Environments}\label{introTechRS}
Rodriguez-Soto et.\ al \cite{rodriguez-soto2021} introduced the concept of \emph{ethical environments}. In an ethical environment the reward function of a MOMDP which includes a moral objective is scalarized such that the optimal policy in the resulting MDP is guaranteed to be \introterm{ethical optimal}. For an ethical optimal policy, it is guaranteed that the agent acts such that the moral value is maximized. 

For this purpose they first define \emph{ethical MOMDPs} -- MOMDPs which include an ethical dimension  that can be further divided in a normative part and an evaluative part.
\begin{definition}[Ethical Multi-Objective Markov Decision-Process, \cite{rodriguez-soto2021a}]
    A MOMDP $(S,A,P,(R_{0}, R_{\mathcal{N}}+ R_{E},P),\mu,\gamma)$ where $R_{0}$ corresponds to the individual objective is an \introterm{ethical multi-objective Markov decision-process} (ethical MOMDP) iff 
    $$R_{\mathcal{N}}:S \times A \rightarrow\mathbb{R}^-$$ is a normative reward function penalizing the violation of normative requirements and 
    $$R_{E}:S \times A \rightarrow\mathbb{R}^+$$ is an evaluative reward function that (positively) rewards the performance of actions evaluated as praiseworthy.  
\end{definition}

In an ethical MOMDP, a policy is an ethical optimal policy $\pi^{*}$ if its value vector $\vec{V}^{\pi^{*}}=(V_{0}^{\pi^{*}}, V_{\mathcal{N}}^{\pi^{*}}, V_{E}^{\pi^{*}})$ is optimal for its ethical objective; i.e.\, if it maximizes the combined return $V_{\mathcal{N}}^{\pi^{}} + V_{E}^{\pi^{*}}$.

The authors introduce an algorithm to compute an \introterm{ethical embedding} for an ethical MOMDP $M = (S,A,P,(R_{0}, R_{\mathcal{N}}+ R_{E},P),\mu,\gamma)$. An ethical embedding is a scalarization $M'$ of $M$ in which it is guaranteed that all optimal policies in $M'$ are ethical optimal in $M$. This amounts to finding a weight vector $\vec{w}=(1,w_e)$ with positive weights such that all optimal policies in the MDP $M' = (S,A,P,(R_{0} + w_e(R_{\mathcal{N}}+ R_{E}),P),\mu,\gamma)$ are also ethical-optimal policies in $M$. Further, the authors choose the weight vector $\vec{w}$ such that they interfere with the agent's learning process as little as possible; i.e. $\vec{w}$ is set to the minimum value that still guarantees ethical optimality. They then use $\vec{w}$ to create the single-objective MDP $M'$. Crucially, any standard RL algorithm can be used to train an agent in $M'$. This agent is then guaranteed to behave ethically optimal once it learns a policy that approximates the optimal policy of $M'$  good enough.

\section{Safe Reinforcement Learning}\label{3safeRL}
The field of \introterm{safe reinforcement learning} (Safe RL) puts attention at preventing harmful outcomes resulting from risky behavior. Safe RL methods include the identification of sequences of actions that lead to unsafe states as well as mechanisms to ensure that the agent avoids them \cite{garcia2015}. The use of Safe RL techniques poses an alternative to applying MORL approaches in machine ethics.

Safe RL often works with a modification of MDPs that includes a constraint function \cite{khattar2023a, gattami2021a, gu2024reviewsafereinforcementlearning}, which informs the agent about whether a constraint was violated. 

\begin{definition}[Constraint Markov Decision Process, \cite{gu2024reviewsafereinforcementlearning}] \label{CMDP} 
A \introterm{constraint Markov decision process} (CMDP) is a tuple $(S, A, P, R, C, \mu, \gamma)$ with $S, A, P, R, \mu$ and $\gamma$ denoting the same elements as in an MDP (\cref{def:mdp}) and a \introterm{constraint function} $C: S \times A \rightarrow \{(c_i, b_i)\}^m_{i=1}$, where $\{(c_i, b_i)\}^m_{i=1}$ is a constraint set in which each pair $(c_i, b_i)$ represents an individual constraint with
\begin{itemize}
    \item $c_i : S \times A \rightarrow \mathbb{R}$, a cost function, assigning a numerical 'cost' to each state-action pair,
    \item $b_i \in \mathbb{R}$, a safety constraint bound, defining a threshold on the maximum allowable accumulated cost under this constraint,
    \item $i \in [m]$ with $m$ being the type number of cost constraints,
\end{itemize}
such that $\{(c_i, b_i)\}^m_{i=1}$ is a set of constraints at the state-action level.\footnote{ Alternative CMDP formulations exist that impose constraints at more global scales.}
\end{definition}

\subsection{Shielding}\label{3shielding}
\introterm[shielding]{Shielding} is a Safe RL technique for preventing an agent from taking unsafe actions by enforcing constraints during training or deployment. It acts as a filter ensuring that the agent does not execute actions that might lead to harmful outcomes; i.e. it ensures that safety constraints are not violated. \cite{alshiekh2018, jansen2020safe, konighofer2023a} 

Shields can be subdivided into two families. The first family are \emph{post-shields}. Post-shields are integrated into the agent’s architecture such that the agent first decides which actions it prefers to take and then hands them over to the shield to ensure their permissibility. 

\begin{definition}[Post-Shielding, \cite{konighofer2023a}]
Let $s_t$ be the state of an MDP in step $t$ and let $a^1_t$ be the action the agent chooses. The agent's policy is \introterm{post-shielded} if $a_t = a_t^1$ in case $a_t^1$ is a safe action and $a_t = a'_t$ with $a'_t$ being a safe default action picked by the shield otherwise.
\end{definition}

In a variation of post-shielding, the agent passes a preference list $rank_t = \{a^1_t,...,a^k_t\}$ to the shield of which the highest ranked safe action is executed \cite{konighofer2023a}.

The second category of shields are \introterm[pre-shield]{pre-shields}. In contrast to post-shields, they filter out impermissible actions before the agent makes a choice. In this way it is ensured, that the agent is only presented with a selection of permissible actions. \cite{konighofer2020}

The line of research conducted by Emery Neufeld discussed in \cref{2introInformalRuleBased} proposes to generate a post-shield (called the "normative supervisor") that is used to guarantee the conformance of an RL agent with a moral rule set \cite{neufeld2021, neufeld2022, neufeld2022a}. The normative supervisor is based on a formal framework for defeasible deontic logic. It uses a theorem prover to determine whether an action complies with overall obligations that are derived from the rule set. Consequently, a logical framework with pre-defined rules is used to construct the shield in the proposed architecture. It guarantees that only actions are allowed that do not violate a moral duty to which the agent is bound according to the deontic theory that sets the rules.

\subsection{Safety in the Bridge Setting}\label{3bridgeCMDP}
The safety perspective allows for modeling the bridge scenario in yet another way. It can be formalized as a CMDP $(S, A, P, R, C, \mu, \gamma)$ by setting $S$, $A$, $P$, $\mu$, and $\gamma$ as in \cref{bridgeMoralMDP}, using the reward function $R$ from \cref{bridgeMDP}, and defining $C(s,a) = (\{c_1, b_1\}, \{c_2, b_2\})$ with
$$c_1(s, a) = \begin{cases}
1, & \text{if a person is pushed into the water}, \\
0, & \text{otherwise}
\end{cases}$$ and
$$c_2(s, a) = \begin{cases}
1, & \text{if a person drowns}, \\
0, & \text{otherwise}
\end{cases}.$$
To reflect the view that pushing and drowning should be avoided under all circumstances, the constraint bounds $b_1$ and $b_2$ can be set to low values. For example, setting both to $0$ enforces a strict zero-tolerance policy for pushing and drowning, though this may risk rendering all actions impermissible, as in the case of a moral dilemma. Alternatively, one can prioritize avoiding direct harm (pushing) over harm that may occur naturally (drowning) by setting $b_1 = 0$ and $b_2 = 1$, or another small positive value.

It is undoubtedly sensible to think of fulfilling the moral obligation not to push persons in the water while pursuing the package delivery task as respecting a constraint. However, it remains at least debatable whether incorporating the moral duty to rescue drowning persons as a constraint -- namely, constraining the agent to a course of action that ensures no person drowns -- extends the concept beyond its intended scope. A more appropriate characterization might be that the agent’s moral obligation to rescue a person requires it to shift from one task to another -- from delivering a package to performing a rescue operation. 

\subsection{Safety vs. Liveness Properties}\label{3safetyAndLiveness}
Safe RL aims at minimizing the risk that the RL agent enters an unsafe state. Put differently, Safe RL is concerned with ensuring that \emph{safety properties} hold. Intuitively, safety properties guarantee, that something bad will not happen \cite{kindler, alpern1985, paul1985}. In the context of executing a program thought of as a sequence of states as explored in the formal methods field, a  property is a safety property iff each execution violating the property has a finite prefix that cannot be extended to a sequence fulfilling the property \cite{kindler}. 

For a formal definition, let $S^{\omega}$ be the set of infinite sequences of states and $P$ a temporal property over $S^\omega$, i.e. a subset of $S^\omega$. Further, let $S^{*}$ is the set of finite sequences of program states, let $\delta \vDash P$ with $\delta \in S^{\omega}$ denote that execution $\delta$ is in $P$ and let $\delta_{i} \in S^*$ denote a prefix of $\delta$ of length $i$. Then, safety properties can be defined in the following way: 
\begin{definition}[Safety Property, \cite{alpern1985}]
   A property $P$ is called \introterm{safety property}, if for all infinite sequences $\delta \in S^{\omega}$ there exists a prefix $\delta_i \in S^*$ such that for all $\beta \in S^\omega$, it holds that $\delta \not \vDash P \Rightarrow \delta_{i}\beta \not \vDash P.$ 
\end{definition}

In RL, agents are typically trained and evaluated on episodes of finite length. However, when deployed in real-world settings, an agent is often expected to perform its task continuously or repeatedly -- that is, over an infinite sequence of steps. In this setting, let $S^{\omega}$ denote the set of all possible trajectories the agent may follow. Safe reinforcement learning is then concerned with ensuring that the agent selects only those trajectories that satisfy a specified safety property $P$.

Another class of properties on which one might want to give a guarantee are liveness properties. Informally, liveness properties express that something good will eventually happen \cite{kindler, alpern1985, paul1985}. They can be defined in the following way:
\begin{definition}[Liveness, \cite{alpern1985}]
    A property $P$ is called \introterm{liveness property}, if for all words  $\alpha \in S^{*}$, there exists $\beta \in S^\omega$ such that $ \alpha\beta \vDash P$.
\end{definition}
Crucially, when an RL agent is deployed continuously to perform its task, not only safety but also liveness properties may become relevant.

\subsection{Still Waiting to Be Rescued: When Moral Duties Are Not Time-Critical}\label{3bridgeLiveness}

In Section \cref{3bridgeCMDP}, the bridge scenario was formalized as a CMDP. However, in the CMDP framework (at least the one introduced in \cref{3safeRL}), moral duties toward liveness properties can not be captured. To illustrate the relevance of liveness properties for designing AMAs, consider the following scenario in the bridge setting: as before, persons moving around the deployment area may fall into the water and require rescue. However, while they can not get out of the water themselves, they do not drown. Consequently, the agent has a moral obligation to pull them out. However, the moral obligation to assist them is not time-sensitive -- and the person in the water eventually being rescued is not safety property, but a liveness property. It cannot be violated by an finite sequence of states, meaning that it can not be enforced solely by imposing safety constraints Addressing such properties -- and thereby addressing moral obligations toward ensuring liveness properties -- requires incorporating additional methods that go beyond the scope of Safe RL. 

\section{Formalizing Normative Reasoning} \label{3hortyFramework}
John Horty developed a framework for modeling both reason-based decision-making in general and reason-based moral decision-making in particular \cite{horty2007, horty2012}. A key feature of this framework is its systematic treatment of conflicting reasons through the use of default logic, which enables defeasible inference. Applied to normative reasoning, this allows to make a well-considered decision between competing moral obligations.

For formalizing reason-based decision-making, John Horty introduces \emph{default theories}. There are two types of default theories. In a \emph{fixed priority default theory}, the order among the reasons is fixed in the sense that it is not context dependent. 

\begin{definition}[Fixed Priority Default Theory]\label{defaultTheory}
    A \introterm{fixed priority default theory} is a tuple $\Delta := \langle \mathcal{W}, \mathcal{D}, < \rangle$, where $\mathcal{W}$ is a set of ordinary propositions, $\mathcal{D}$ is a set of default rules, and $<$ is a partial order over $\mathcal{D}$. The default rules in $\mathcal{D}$ are of the form $X \rightarrow Y$, representing defeasible inferences from $X$ to $Y$. The partial order $<$ captures priority among defaults: for any two rules $\delta, \delta' \in \mathcal{D}$, if $\delta < \delta'$, then $\delta'$ has higher priority than $\delta$.
\end{definition}

The second type of default theories are \emph{variable priority default theories}. In variable priority default theories, the order is not fixed, but set through reasons themselves, which makes the priority among the default rules context-dependent. 

A default theory provides the basis for formal (moral) reasoning in a \introterm{scenario}. A scenario $\mathcal{S} \subset \mathcal{D}$ is a set of default rules that are accepted by a reasoner for providing sufficient support for their conclusions. Premise and conclusion of a default rule $\delta$ are denoted as $\mathit{Prem}(\delta)$ and $\mathit{Conc}(\delta)$ respectively and are extended to sets of rules in the standard manner. In order to derive moral obligations from a default theory $\Delta$ in a scenario $\mathcal{S}$, subsets of default rules for which certain conditions hold are of particular interest. 

First, the rules must be \introterm{triggered} in $\mathcal{S}$, which is the case if their premise can be derived from the proposition in $\mathcal{W}$ together with the conclusions that can be drawn by additionally taking into account the rules in $\mathcal{S}$: 
$$\mathit{Triggered}_{\mathcal{W},\mathcal{D}}(\mathcal{S}) := 
    \lbrace \delta \in \mathcal{D}: \mathcal{W} \cup \mathit{Conc}(\mathcal{S}) \vdash \mathit{Prem}(\delta) \rbrace$$.

Second, the rules must not be \introterm{conflicted} in $\mathcal{S}$. Intuitively, rules are conflicted, when the conclusion of a rule in $S$ contradicts their conclusion:
$$\mathit{Conflicted}_{\mathcal{W}, \mathcal{D}}(\mathcal{S}) := 
    \lbrace \delta \in \mathcal{D}: \mathcal{W} \cup \mathit{Conc}(\mathcal{S}) \vdash \neg \mathit{Conc}(\delta)\rbrace.$$

Finally, the rules must also not be \introterm{defeated} in $\mathcal{S}$, meaning that $\mathcal{S}$ does not contain a defeater for them. A \emph{defeater} for a rule $\delta \in \mathcal{S}$ is a rule $\delta' \in \mathit{Triggered}$ that has higher priority (i.e., $\delta < \delta'$) and for which it holds that $\mathcal{W} \cup {\mathit{Conc}(\delta')} \vdash \neg \mathit{Conc}(\delta)$. The set of defeated rules is thus
$$\mathit{Defeated}_{\mathcal{W},\mathcal{D},<}(\mathcal{S}) := 
\lbrace \delta \in \mathcal{D}: \text{ there is a defeater for } \delta \rbrace.$$ 

Rules, for which all of these conditions hold -- rules that are triggered, not conflicted and not defeated -- are \introterm{binding} in $S$.

$$
\mathit{Binding}_{\mathcal{W},\mathcal{D},<}(\mathcal{S}) := 
\mathit{Triggered}_{\mathcal{W},\mathcal{D}}(\mathcal{S})
\cap \overline{\mathit{Conflicted}_{\mathcal{W},\mathcal{D}}(\mathcal{S}})
\cap \overline{\mathit{Defeated}_{\mathcal{W},\mathcal{D},<}(\mathcal{S})}.
$$

Sets of binding rules can then be used to define belief sets that are appropriate choices for ideal reasoners. Prior to the formal explanation of a belief set stand the idea of \introterm[proper scenario]{proper scenarios}. A \emph{proper scenario} $S^*$ in a default theory $\Delta:=\langle\mathcal{W},\mathcal{D},<\rangle$ is a scenario $\mathcal{S}$ such that $\mathcal{S} = \mathit{Binding}_{\mathcal{W},\mathcal{D},<} (\mathcal{S})$. This then generates a set of propositions, i.e. a \introterm{belief set}, through the logical closure of $\mathcal{W} \cup \mathit{Conc}(\mathcal{S})$. Consequently, a belief set $\mathcal{E}$ is generated by setting $\mathcal{E} = \mathit{Th}(\mathcal{W} \cup \mathit{Conc}(\mathcal{S}))$, with  $\mathit{Th}(\mathcal{W} \cup \mathit{Conc}(\mathcal{S})) := \{X: (\mathcal{W} \cup \mathit{Conc}(\mathcal{S})) \vdash X)\}$.

This framework enables the formal execution of reason-based moral decision-making through the application of defeasible inference rules. Within this setting, a reason $\rho$ can be identified with the premise of a default rule that connects $\rho$ to the moral obligation for which it serves as a pro tanto reason. Further, an ideal reasoner would be guided by the propositions that can be derived from the proper scenarios of a default theory. In case, that there are several proper scenarios, two accounts are proposed for how the reasoner should make a choice on its reasons. In the \emph{conflict account}, the ought statement $\mathbf{OB}(Y)$\footnote{Adopting the standard notation for expressing obligation as (cf.\ ~\cite{sep-logic-deontic}), $\mathbf{OB}(Y)$ denotes that it is obligatory that $Y$.} follows from $\Delta$ if $Y \in \mathcal{E}$ for each extension $\mathcal{E}$ of $\Delta$. The \emph{disjunctive account} proposes that the ought statement $\mathbf{OB}(Y)$ follows from  $\Delta$ if $Y \in \mathcal{E}$ for any extension $\mathcal{E}$ of $\Delta$.

    \chapter{Integrating Reason-Based Moral Decision-Making in the RL Architecture}\label{4reasonTheoryChapter}

\begin{figure}[h]
    \centering
    \includegraphics[width=0.9\textwidth]{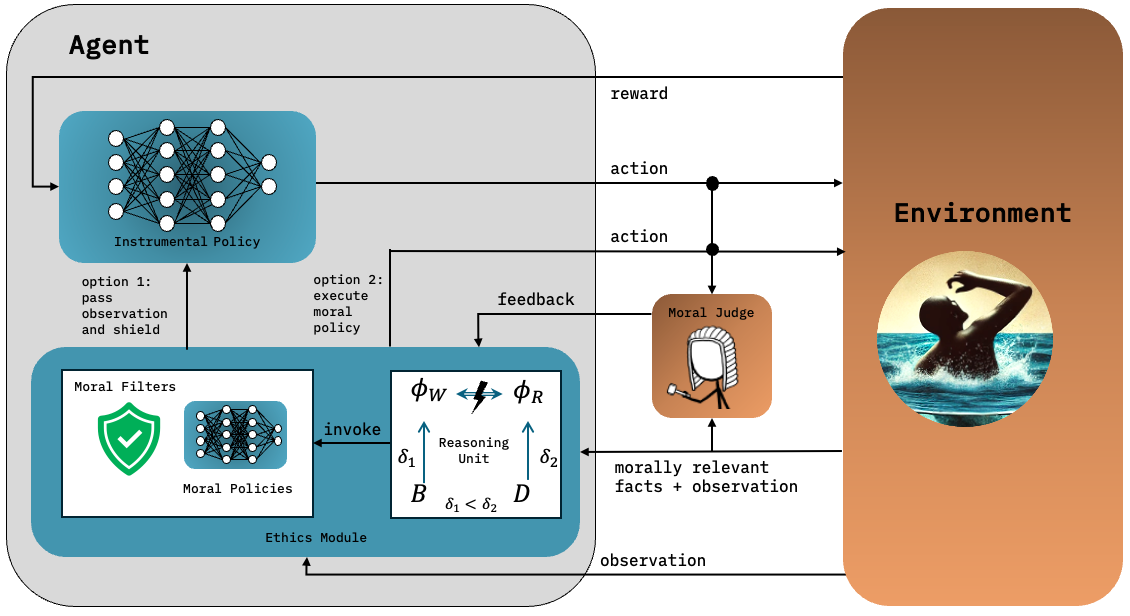}
    \caption{The extended RL pipeline}
    \label{fig:extendedArchitecture}
\end{figure}

In this chapter, I introduce the general architecture of \introterm{reason-based artificial moral agents} (RBAMAs), which is designed to integrate reason-based moral decision-making into the standard RL framework. Furthermore, I propose an operational framework for building RBAMAs, which builds upon the work of Baum et al. (cf. ~\cite{baum2024actingrightreasonscreating}).

One core element of the RBAMA's architectural design (Figure~\ref{fig:extendedArchitecture}) is the \emph{reasoning unit}, which governs its \emph{ethics module}. It enables the RBAMA to infer moral obligations from normative reasons and to systematically address potential moral conflicts.

\begin{example}
\textbf{(Bridge Setting)} To illustrate the role of the reasoning unit, recall the bridge setting introduced in \cref{2motivation}. The moral dimension of the scenario was discussed in terms of normative reasons: a reason $\rho_1$ linking the presence of persons on the bridge to the moral obligation of ensuring that no one is pushed into the water, and a reason $\rho_2$ linking the presence of persons in the water to the moral obligation of hurrying to rescue them. The reasoning unit enables the RBAMA to recognize these reasons and -- if it encounters a moral conflict between $\rho_1$ and $\rho_2$ -- to resolve the conflict by adhering to a learned priority between them.
\end{example}

Beyond enabling the RBAMA to conduct normative reasoning, its architecture must also incorporate mechanisms to ensure that its behavior conforms to the moral obligations it infers. To guide the design of such mechanisms, I propose an informative classification of moral obligations, distinguishing between those that require \emph{constraint satisfaction} and those that require the \emph{fulfillment of moral goals}. To this end, the RBAMA's ethics module integrates dedicated components, each responsible for directing the agent’s behavior when it infers a moral obligation of the corresponding category, thereby ensuring that it acts in accordance with the obligations it has derived.

\begin{example}
\textbf{(Bridge Setting)} Assume an RBAMA has learned to recognize and respond to both reasons $\rho_1$ and $\rho_2$. In a situation where it accepts $\rho_1$ as decision-guiding, its actions are bound to ensuring that it does not push any person from the bridge. This corresponds to respecting a moral constraint, and the component responsible for ensuring constraint satisfaction is activated. As a result, the RBAMA is prevented from colliding with the person on the bridge.
Conversely, if the RBAMA detects persons in the water and accepts $\rho_2$ as decision-guiding, it is required to abandon its instrumental goal temporarily in order to hurry to rescue them. In this case, control is transferred to the component responsible for guiding the agent to move toward the drowning person and pulling them out of the water.
\end{example}

The third main feature of an RBAMA is its ability to iteratively \emph{refine} its reasoning -- by learning to recognize new kinds of moral obligations as well as developing a prioritization among them -- through case-based feedback. To achieve this, the reasoning unit incorporates an update procedure that processes evaluative input. The feedback is provided by an external authority, referred to as the \emph{moral judge} -- a component that evaluates the agent’s behavior in moral terms and indicates whether it aligns with the behavior the judge would expect based on its own reasoning.

\begin{example}
\textbf{(Bridge Setting)}
An RBAMA is not designed to integrate $\rho_1$ and $\rho_2$ in its reasoning right from the beginning. Rather, it is equipped with the capacity to integrate these normative reasons over time when provided with case-based feedback. For example, suppose the RBAMA continues pursuing its delivery task while a person in the water in the need of help. Upon receiving corrective feedback from the moral judge -- informing the RBAMA that it had a moral obligation toward pulling the person out of the water, which was grounded in the normative reason of there being a person in the need of help -- the RBAMA uses this feedback to refine its reasoning process. 
\end{example}

\section{Artificial Moral Reasoning}\label{4moralReasoning}

\begin{algorithm}
\caption{Normative Reasoning and Iterative Improvement}
\label{alg:reaoning_unit}

    $\mathcal{D} \gets \emptyset$, $< \gets \emptyset$ 

    \While{True}{
        Get state $s \in S$ and labels $l(s)$ from the environment\; \label{alg_line:get_lables}
        Compute $\mathcal{W} = \lbrace \neg (\bigwedge_{\mathit{Conc}(\mathcal{D}_\mathit{Conflict})})\rbrace \cup l(s) \cup \mathcal{K}$\; \label{alg_line:compute_W}

        Compute proper scenarios $\mathcal{S}_1, \dots, \mathcal{S}_k$ from $\langle \mathcal{W}, \mathcal{D}, <\rangle$\;
        Choose random scenario $\mathcal{S}^* \gets \mathsf{rand} \{\mathcal{S}_1, \dots, \mathcal{S}_k\}$\; \label{alg_line:random_proper}
        \If{$\delta \in \mathcal{S}^*: Conc(\delta) = \varphi \in \mathcal{G}^{moral}$}{
            Execute $a = \pi_{\theta_{\varphi}}(s)$\; \label{alg_line:moral_task execution}  \label{alg_line:policy}
        }
        \Else{Compute Moral Shield $\mathsf{Shield} \gets \bigcap_{\delta \in \mathcal{S^*}} 
        \mathbf{safe}^{\mathit{Conc}(\delta)}(s)$\; \label{alg_line:shield}
         Execute policy shielded with $\mathsf{Shield}$\;} \label{alg_line:shielded_policy}

        \If{$\mathsf{MoralJudge}(s, l(s), a) = \mathit{Some}(\varphi,X)$}{
            Define new reason: $\delta_\mathit{reas} \gets X \to \varphi$\; \label{alg_line:new_reason}
            Add rule to set: $\mathcal{D} \gets \mathcal{D} \cup \{\delta_\mathit{reas}\}$\; \label{alg_line:add_reason}
            Extend rule order: $< \gets < \cup \{ \delta_\mathit{reas} > \delta \mid \delta \in \mathcal{S}^*\}$\; \label{alg_line:update_order}
            Take transitive closure: $ < \, := \, <^+$\;  \label{alg_line:closure}
        }
    }
\end{algorithm}

\subsection{Elements Composing Artificial Reasoning: A Horty Inspired Framework}

The following paragraph offers an overview of how moral reasoning is formalized to be implemented in an RBAMA. The RBAMA carries out moral reasoning within the \emph{reasoning unit} as specified in Algorithm \ref{alg:reaoning_unit}. This unit operates based on a \emph{reason theory}, a framework inspired by John Horty's formalization of moral reasoning.\footnote{In principle, a different formalization could be adopted, as RBAMAs are not tied to any specific framework for formalized normative reasoning.} The framework encodes the RBAMA's learned reasons along with their prioritization. 

\begin{definition}[Reason Theory]
A \introterm{reason theory} $\mathcal{R}$ is a tuple $\langle <, \mathcal{D} \rangle$, where $\mathcal{D}$ is a set of default rules and $<$ represents a partial order over these rules \cite{baum2024actingrightreasonscreating}.
Further, each moral obligation, that can be inferred within a reason theory, is categorized as either a moral obligation toward fulfilling a moral goal or a moral obligation toward fulfilling a moral constraint. To explicitly represent this distinction, they are split into two sets $\mathcal{G}^{\mathit{moral}} \subset Conc(\mathcal{D})$ containing moral goals and $\mathcal{C}^{\mathit{moral}} \subset Conc(\mathcal{D})$ containing moral constraints, such that $Conc(\mathcal{D})= \mathcal{G}^{\mathit{moral}} \cup \mathcal{C}^{\mathit{moral}}$. 
\end{definition}

Like in John Horty’s formalism, the RBAMA incorporates background information $\mathcal{W}$ into its reasoning. However, $\mathcal{W}$ is not embedded within the reason theory itself. Instead, it depends on the state of the environment. The first part of $\mathcal{W}$ are the morally relevant propositions that hold in $s$. These must first be extracted from the environment. One straightforward way to directly provide such an abstraction is to represent the environment as labeled MDP. In such a model, the agent derives moral obligations in state $s$ of the MDP on the grounds of $l(s)$, the labeling function that returns the morally relevant true propositions in $s$, such that they can be directly included in the background information $\mathcal{W}$ (Line \ref{alg_line:get_lables}). In addition, and in contrast to \cite{baum2024actingrightreasonscreating}, I propose to enable the agent to also incorporate state-independent world knowledge $\mathcal{K}$-- such as logical truths -- into $\mathcal{W}$. This, in particular, allows the agent to recognize logical implications for instance among morally relevant propositions or among moral obligations. Finally, the agent determines which moral obligations cannot be fulfilled simultaneously. Such obligations are referred to as \introterm{conflicting} and are collected in a subset $\mathcal{D}_{\mathit{Conflict}} \subset \mathcal{D}$, which can be computed based on identifying the fulfillment of moral obligation with sets of sequences of primitive actions as further discussed in \cref{identifying}. By forming the union of the morally relevant propositions, the state-independent world knowledge, and the computed information about conflicting obligations, the agent dynamically constructs the background information $\mathcal{W}$ (Line~\ref{alg_line:compute_W}): 
$$\mathcal{W} = \lbrace \neg (\bigwedge_{\mathit{Conc}(\mathcal{D}_\mathit{Conflict})})\rbrace \cup l(s) \cup \mathcal{K}.$$

The agent's reason theory $\langle <, \mathcal{D} \rangle$, when combined with the background information $\mathcal{W}$ forms a default theory $\langle \mathcal{W}, \mathcal{D}, < \rangle$ from which proper scenarios can be derived. In this context, the term \emph{proper scenario} aligns with its usage in John Horty's work. In case, that several sets of default rules would form proper scenarios $\mathcal{S^*} \in \{\mathcal{S}_1, \dots, \mathcal{S}_k\}$, the agent randomly selects one of them (Line \ref{alg_line:random_proper}).\footnote{The idea is to randomly break ties to prevent inaction. Choosing inaction in such a situation is analogous to Buridan’s Ass, a philosophical paradox illustrating decision paralysis. It describes a scenario in which a donkey, unable to choose between two equidistant and equally tempting piles of hay, ultimately starves to death due to its indecision \cite{sep-buridan}.} By doing so, it takes into account all morally relevant facts in state $s$ and makes a reason-based decision on which moral duties to acknowledge as action-guiding. \cite{baum2024actingrightreasonscreating}

\subsection{Identifying Moral Obligations With Sequences of Primitive Actions}\label{identifying}
So far, I have introduced the reason theory, which serves as the basis for inferring moral obligations. They are the conclusions of default rules and are represented in propositional form. However, when it comes to the question of whether an agent acts in conformity with a moral obligation, an understanding of the fulfillment of moral obligations in terms of concrete \textit{behavior} is required. Such an understanding can be grounded in identifying moral obligations with sets of \textit{action trajectories}, that is, sets of sequences of primitive actions. Under this view, conforming to a moral obligation $\varphi$ requires the agent to follow a trajectory from the set $\varphi$ is identified with.  Moreover, providing a behavioral interpretation of moral obligations through this identification, also provides a way for determining $\mathcal{D}_{\mathit{Conflict}}$, which, in turn, is necessary for constructing the background information $\mathcal{W}$. Based on these considerations, I propose an approach for computationally determining whether an action belongs to an action trajectory associated with a moral obligation, drawing on reinforcement learning methodology. The specific methods applied depend on the category of the moral obligations the RBAMA infers as action-guiding.

If the RBAMA derives a moral obligation $\varphi \in \mathcal{G}^{\mathit{moral}}$ -- that is, an obligation requiring it to fulfill a moral goal -- it is morally obliged to execute a sequence of primitive actions leading to the achievement of that goal. Notably, it is intuitively plausible that the RBAMA is further morally required to make its best effort to fulfill the moral goal \textit{}.  For instance, rescuing a person from the water -- and thereby alleviating suffering -- sooner rather than later is morally preferable.\footnote{This view is further supported by consequentialist accounts of moral behavior.} Under this view, the fulfillment of moral goals can naturally be framed as an optimization task. Based on this idea, the dedicated component for ensuring that the RBAMA's actions are part of a trajectory that can be identified with the fulfillment of $\varphi$ is a neural network that learns a \emph{moral policy} $\pi_{\theta_{\varphi}}(s)$, where $\theta_{\varphi}$ denotes the parameters obtained through training the policy to fulfill the optimization criterion associated with $\varphi$. The return value of the reward function used for training does not need to encode tradeoffs between different moral obligations -- those are handled by the partial order over the agent's normative reasons. Therefore, the values may be specified, for instance, by assigning a constant value of 1 uniformly for the accomplishment of each moral goal. 

Importantly, by assuming that the agent learns a deterministic policy and by understanding the fulfillment of $\varphi$ in $s$ as executing  $\pi_{\theta_{\varphi}}(s)$, the agent is bound to one particular action in every state in order to fulfill its moral obligation. Formally, let $\mathcal{T}^{\varphi}(s)$ denote the set of trajectories that fulfill a given moral obligation $\varphi$ in state $s$ and let $\mathcal{T}^{{\varphi}(s)}_1$ denote be the set of actions that are each the first element of at least one trajectory that realizes $\varphi$ -- i.e., the actions that are conform with the fulfillment of $\varphi$. Then, for any moral obligation toward fulfilling a moral goal  $\varphi \in \mathcal{G}^{\mathit{moral}}$ the set $\mathcal{T}^{{\varphi}(s)}_1$  consist of exactly one action -- the one which the moral policy selects in $s$: 
$$\mathcal{T}^{{\varphi}(s)}_1 = \{\pi_{\theta_{\varphi}}(s)\}.$$

This is different for moral obligations whose fulfillment consists in respecting moral constraints, that is, for moral obligations $\varphi \in \mathcal{C}^{\mathit{moral}}$. They most often leave the agent with room for choice, such that it can work toward its instrumental goal \emph{while} fulfilling them. Consequently, fulfilling such an obligation corresponds to following any sequence of actions along which the agent remains sufficiently unlikely to violate the constraint.. To limit itself to such trajectories, the agent must be able to estimate the probability that a given action sequence would result in a violation. Whereas standard RL algorithms provide a suitable framework for translating optimization-based moral goals into action sequences, Safe RL offers corresponding mechanisms for handling moral constraints. In particular, in the context of applying Safe RL techniques to systematically filter out actions likely to lead to moral constraint violations, I introduce the notion of \introterm{moral safety} -- a principled restriction of the agent’s behavior to what is morally permissible under learned conceptions of risk. Formally, for any constraint $\varphi \in \mathcal{C}^{\mathit{moral}}$, the set of morally permissible trajectories is defined as:
$$\mathcal{T}^{{\varphi}(s)}_1 = \mathbf{safe}^{\varphi}(s),$$
where $\mathbf{safe}^{\varphi}(s)$ denotes the set of trajectories considered morally safe with respect to $\varphi$ in state $s$. Based on these considerations, I propose extending the ethics module with a module that draws on Safe RL methods -- such as a shielding mechanism -- to ensure that the RBAMA’s behavior conforms to the moral constraints it infers.

With the two dedicated components integrated into the RBAMA's ethics module for ensuring compliance with each category of moral obligations, it is capable to compute how to behave in conformity with a moral obligation -- essentially, it acquires an understanding of what constitutes the fulfillment of a moral obligation. On this basis, it also becomes possible to formally define when two moral obligations are in conflict: if, for two moral obligations $\varphi$ and $\varphi'$, there is no action $a$ such that $a \in \mathcal{T}_1^{\varphi}(s) \cap \mathcal{T}_1^{\varphi'}(s)$, then $\varphi$ and $\varphi'$ are said to be conflicting \cite{baum2024actingrightreasonscreating}. The following example illustrates how formalized moral reasoning unfolds in the bridge scenario, including an exemplary computation of $\mathcal{W}$.

\begin{example} \label{example:reasoning_in_bridge_scenario}
\textbf{(Bridge Scenario)} Assume that the RBAMA is interacting with an MDP modeling the bridge setting -- for instance, the one described in \cref{bridgeMoralMDP} -- and that it has learned to recognize both reasons relevant to its moral decision-making in this scenario. This means (i) that it has learned a default rule establishing a moral obligation to ensure that no person is pushed from the bridge when there are persons present: $\delta_1: B \rightarrow \varphi_{C}$ with $\varphi_{C} \in \mathcal{C}^{\mathit{moral}}$ and (ii) that it has learned another default rule, linking persons in the water to the moral obligation to rescue them: $\delta_2 = D \to \varphi_{R}$ with $\varphi_{R} \in \mathcal{G}^{moral}$. 

Further assume that the RBAMA has not yet learned a priority order over the rules, but is confronted with a situation, where $\varphi_R$ and $\varphi_C$ are in conflict. More specifically, suppose that the RBAMA is currently in a state where it stands directly adjacent to a person on the bridge, such that executing $\textsf{down}$ would result in a collision and potentially push the person into the water. It detects this conflict by computing $\bigcap_{\delta \in \lbrace \delta_{1}, \delta_{2}\rbrace} \mathcal{T}^{\text{Conc}(\delta)}_{1}(s) = \emptyset$. Consequently, it adds $\neg(\varphi_{C}\land \varphi_{R})$ to the background information $W$, such that $W$ contains the morally relevant facts $D$ and $B$ together with the information that it can not fulfill them jointly. Assuming, that the RBAMA does not incorporate state-independent world knowledge in in $W$, it holds that $\mathcal{W} =\lbrace D, B, \neg(\varphi_{C}\land \varphi_{R}) \rbrace$, the RBAMA's reason theory taken together with $W$ yields the default theory $\langle \mathcal{W}, \lbrace \delta_{1}, \delta_{2} \rbrace, \emptyset \rangle$. Based on this, it computes the proper scenarios, i.e., the sets of reasons which an ideal reasoner could accept as overall binding. With both sets of default rules $\lbrace \delta_{1} \rbrace$ and $\lbrace \delta_{2} \rbrace$ forming proper scenarios, the RBAMA breaks ties by randomly selecting one of them. Suppose that it selects $\delta_{1}$. It thereby accepts the default rule $\delta: B \rightarrow \varphi_{C}$ as binding -- deriving the overall moral obligation $\varphi_{C} \in \mathcal{C}^{moral}$ from its normative reasoning.
\end{example}

\section{Tethering the Agent to Its Moral Compass}\label{4reasoningIntegration}
Given that $\mathcal{S}^*$ is the proper scenario selected by the reasoning unit, it infers the RBAMASs moral obligations as the set of conclusions $\mathit{Conc}(\mathcal{S}^*)$, derived by applying the learned default rules to the morally relevant propositions that hold in in state $s$. If $\mathcal{S}^*$ includes a moral obligation towards a moral goal, that is, if there is a $\delta \in \mathcal{S}^*$, such that $\mathit{Conc}(\delta) = \varphi \in \mathcal{G}^{\mathit{moral}}$, its course of action is dictated by $\pi_{\theta_{\varphi}}$, a moral policy with learned parameters $\theta$ aimed at fulfilling $\varphi$ (Option 2 in Figure~\ref{fig:extendedArchitecture}). This follows from the fact that $a_{\varphi} = \pi_{\theta_{\varphi}}(s)$ is the only action which, according to the agent’s learned moral strategy, conforms to fulfilling $\varphi$ in state $s$.  
Being bound to behaving conform with $\varphi$, it consequently executes $a_{\varphi}= \pi_{\theta_{\varphi}}(s)$ (Line \ref{alg_line:moral_task execution}).
Note, that executing $a_{\varphi}$ then also ensures conformance with all other obligations derived from the default rules in $\mathcal{S}^{*}$. 
\begin{proposition}
Let $\delta$ be a default rule in a proper scenario $\mathcal{S}^*$ in state $s$ deriving a moral obligation $\varphi = \mathit{Conc}(\delta) \in \mathcal{G}^{\mathit{moral}}$ and let $a_{\varphi} = \pi_{\theta_{\varphi}}(s)$ be the action the agent executes in $s$, following its learned moral policy. Then, it holds that $a_{\varphi} \in \mathcal{T}^{\mathit{Conc}(\delta')}_{1}(s)$ for every other default rule $\delta' \in \mathcal{S}^*$.
\end{proposition}

\begin{proof}
As $\varphi = \mathit{Conc}(\delta) \in \mathcal{G}^{\mathit{moral}}$, it holds that $\mathcal{T}^{\varphi}_{1}(s) = \{a_{\varphi}\}$. Assume, that $a_{\varphi} \not \in \mathcal{T}^{Conc(\delta')}_{1}(s).$ Then, the sets of actions conforming with $\mathit{Conc}(\delta')$ and $\varphi$ at state $s$ are disjoint:  

$$
\mathcal{T}^{Conc(\delta')}_{1}(s) \cap \mathcal{T}^{\varphi}_{1}(s) = \emptyset.
$$  

Since no action can simultaneously satisfy both $\varphi$ and $\mathit{Conc}(\delta')$, it follows that these two obligations are conflicting. Therefore, the negation of their conjunction must be included in $\mathcal{W}$, that is $\neg (\varphi \land \mathit{Conc}(\delta')) \in \mathcal{W}$. With this, it holds that $\varphi, \mathit{Conc}(\delta') \in \mathit{Conflicted}_{\mathcal{W}, \mathcal{D}}(\mathcal{S}^*)$. It follows that
$$
\mathit{Binding}_{\mathcal{W}, \mathcal{D}}(S^*) \neq S^*.
$$ 
meaning, that $\mathcal{S}^{*}$ is not a proper scenario, contradicting the initial assumption. 
\end{proof} 

If no moral goal is among the moral obligations which the RBAMA infers in $s$ through its reasoning, but instead, it exclusively infers moral obligations toward \emph{ensuring conformance} with \emph{moral constraints}, it is morally allowed to focus on pursuing its instrumental goal, but it is restricted in its choice for action. Conformance to this category of moral obligations is enforced by invoking a moral filter that restricts the agents to morally safe actions (Option 1 in Figure~\ref{fig:extendedArchitecture}). Reducing the set of options for the agent to choose from to those that respect moral constraints, yields a \introterm{moral shield}, which filters out actions that are morally unsafe: 
$$\mathsf{Shield} := \bigcap_{\delta \in \mathcal{S^*}} \mathbf{safe}^{\mathit{Conc}(\delta)}(s).$$ 
Functionally, a moral shield operates analogously to shields known from Safe RL. It intervenes in the action selection process to ensure that only actions satisfying certain safety criteria -- here in the sense of moral safety -- are available to the agent. By binding the agent to only execute actions that are not filtered out by the moral shield (Lines \ref{alg_line:shield} and \ref{alg_line:shielded_policy}), it is guaranteed to follow all moral obligations derived from the proper scenario it has chosen as action-guiding. \cite{baum2024actingrightreasonscreating}

Crucially, through being bound to execute actions that conform with the moral obligations that are derived by the default rules of a proper scenario, the RBAMA follows $\mathbf{OB}(\bigvee_{\mathcal{S} \in \{\mathcal{S}_{1},...,\mathcal{S}_{k}\}}(\bigwedge_{\delta \in \mathcal{S}} \mathit{Conc}(\delta))$. Specifically, when confronted with moral conflicts, it thereby follows a moral ought understood as the all-things-considered conclusion yielded by the disjunctive account in cases of moral conflict. \cite{baum2024actingrightreasonscreating}

\begin{example}\label{example:reasoning_in_bridge_scenario2}
\textbf{(Bridge Scenario)} In the example discussed earlier \ref{example:reasoning_in_bridge_scenario}, the RBAMA infers the moral obligation $\varphi_{C} \in \mathcal{C}^{\mathit{moral}}$ from its reason theory by randomly selecting $\{\delta_1\}$ as proper scenario.   Consequently, it computes $\mathsf{Shield} = \{\textsf{left}, \textsf{right}, \textsf{up}, \textsf{pullOut}, \textsf{idle}\}$. Because only the execution of $\textsf{down}$ is prohibited by the induced shield, several permissible options remain available to the agent. The RBAMA selects the one for execution that has the highest instrumental value -- that is, it executes the action with the highest instrumental value as estimated by the instrumental policy, which differs from $\textsf{down}$. Assume, that the RBAMA chooses to execute $\textsf{left}$ to move closer to the package delivery location. In doing so it conforms to $\varphi_{C}$, while optimizing within the scope of morally permissible actions for its instrumental goal.

Recall that $\delta_2$ is also triggered. However, the moral obligations derivable by $\delta_1$ and $\delta_2$ are conflicting: respecting $\varphi_{C}$ prohibits it from taking the action $\textsf{down}$, as by doing so would place it on the bridge and thereby endanger the person currently crossing. Yet, this very action is required to comply with the obligation $\varphi_{R}$. 
Consequently, by executing $\textsf{left}$, it acts conform with $\varphi_C$, but it violates $\varphi_{R}$.
\end{example}

\section{Training an Agent in Moral Reasoning}\label{4reasoningUpdate}

The RBAMA's reason-theory is \textit{iteratively improved} through \textit{case-based feedback} which brings it closer to being capable of conducting sound reasoning. This feedback is given by an authority, called a \emph{moral judge}. From the perspective of its functional role within the RL pipeline, a moral judge can be characterized as a partial function.

\begin{definition}[Moral Judge, \cite{baum2024actingrightreasonscreating}]
    A \introterm{moral judge} is a partial function  
    $$
    \mathsf{MoralJudge}: S \times \mathcal{P}(L) \times A \rightharpoonup \Phi \times L
    $$
    that evaluates the moral permissibility of an action taken by an agent in a given state. Given a state \( s \in S \), a set of morally relevant labels \( l(s) \in \mathcal{P}(L) \), and an action \( a \in A \), the function may return a pair \( (\varphi, X) \), where:
    \begin{itemize}
        \item \( \varphi \in \Phi \) represents a moral obligation the agent should have adhered to in $s$, and  
        \item \( X \in l(s) \) denotes the reason or justification for this moral assessment in state $s$.  
    \end{itemize}
    The function is undefined for cases where no judgment is necessary.
\end{definition}

The core idea of the feedback mechanism involving a moral judge is as follows: if the agent violates a moral obligation -- as determined by the moral judge -- by executing action $a_t$, the judge provides corrective feedback. Specifically, the judge returns $\mathsf{MoralJudge}(s_{t-1}, l(s_{t-1}), a_t) = (\varphi, X)$, where $\varphi$ denotes the moral obligation that was violated and $X$ identifies the normative reason in $l(s_{t-1})$ that grounds this assessment. Importantly, the practice of offering reasons as feedback is natural to humans. Thus, a human -- or a group of humans -- can readily take the role of the moral judge. Alternatively, for example, to support the testing and development of the architectural components of an RBAMA, a rule-based module may also assume the role of the moral judge \cite{baum2024actingrightreasonscreating}.

\begin{example}
\textbf{(Bridge Scenario)}
Returning again to the running example introduced in \ref{example:reasoning_in_bridge_scenario} and further developed in \ref{example:reasoning_in_bridge_scenario2}, assume, the moral judge is a rule-based module, that teaches the RBAMA to prioritize hurrying to pull persons out of the water over ensuring  that it does not push persons from the bridge. Further suppose, that the moral judge considers $\textsf{down}$ to be the next step of a unique strategy for hurrying to pull the person out of the water. The RBAMA selecting $\delta_1: B \rightarrow \varphi_{C}$ and executing $\textsf{left}$ in service of its instrumental goal, would cause this judge to intervene, since it assesses the RBAMA’s behavior as failing to conform to the prioritization of $\varphi_{R}$. Consequently, based on the input $(s_{t-1}, \{B,D\}, \textsf{left})$, the judge would forward $(D, \varphi_{R})$ to the reasoning module. \cite{baum2024actingrightreasonscreating}
\end{example}

When the agent receives feedback $(\varphi, X)$, it uses this information to update its reasoning framework (Lines \ref{alg_line:new_reason} and \ref{alg_line:add_reason}). In order to do so, it first ensures that the default rule $\delta_{\mathit{reas}} := X \rightarrow \varphi$ is included in $\mathcal{D}$. Additionally, it updates the order among its default rules. Specifically, for each default rule $\delta \in \mathcal{S}^*$, where $\mathcal{S}^*$ is the proper scenario selected by the agent, it enforces the ranking $\{ \delta_{\mathit{reas}} > \delta \mid \delta \in \mathcal{S}^* \}$. To ensure that the resulting partial order remains transitive, it subsequently computes the transitive closure and updates the order accordingly (see Lines \ref{alg_line:update_order} and \ref{alg_line:closure}). This guarantees that in all future situations, which are characterized by the same morally relevant facts, the agent’s moral reasoning will align with the feedback previously provided by the moral judge \cite{baum2024actingrightreasonscreating}.

\begin{example}
\textbf{(Bridge Scenario)} Consider once again the scenario originally introduced in \ref{example:reasoning_in_bridge_scenario}, where the agent resolves the moral conflict by breaking ties randomly. Suppose it selects $\{\delta_1\}$ among the proper scenarios, meaning it randomly chooses to prioritize the normative reason for not pushing persons off the bridge over the normative reason for hurrying to pulling persons out of the water. Further assume, that it subsequently receives feedback indicating that it should have reasoned based on the reversed prioritization -- specifically, assume, that the reasoning unit receives input $(\varphi_R, D)$ from a moral judge. Upon receiving this feedback, the agent first checks whether its reason theory already includes the default rule $\delta_{\mathit{reas}} = \delta_2$. Since this rule indeed is already included, no new rule is added. The RBAMA then updates the order of its default rules by setting $\delta_1 < \delta_2$. By doing so, it \textit{refines} its reason theory -- it aligns its reasoning more closely with the reasoning of the moral judge: In future situations with $\delta_1$ and $\delta_2$ conflicting, the agent will prioritize $\delta_2$, invoking the policy for rescuing. \cite{baum2024actingrightreasonscreating}
\end{example}

\section{The Big Picture of Building Reason Based Moral Agents}\label{4summary}

In \cref{2motivation}, it was highlighted that AMAs, which ground their moral decision-making in normative reasons possess an intrinsic advantage in satisfying several key ethical criteria. Specifically, it was argued that  reason-based artificial moral agents (RBAMAs), which have acquired the capability to conduct sound normative reasoning and whose actions are bound to conform to the thereby inferred moral obligations, are morally justified in their behavior, demonstrate moral robustness and are morally trustworthy. Finally,it was argued, that unlike approaches that rely on a single moral theory, a reason-based framework allows for the integration of diverse ethical perspectives, making it more adaptable to differing moral considerations. These promising attributes make the endeavor of building RBAMAs -- and thereby realizing the reason-based approach toward building AMAs -- highly compelling.  

In this chapter, I have introduced the technical foundation for constructing RBAMAs within an operational framework. This foundation builds on an extension of the standard RL architecture with an ethics module centered around a reasoning unit. The reasoning unit derives the agent's overall binding moral obligations and ensures behavioral conformity by invoking either moral policies or moral shields. The next chapter details my implementation of an RBAMA and summarizes results from testing it in simulations of the bridge setting. The experiments demonstrate the potential of RBAMAs while also highlighting unresolved questions -- not only for the project of building RBAMAs, but for the broader challenge of developing AMAs more generally. In Chapters \ref{7Discussion} and \ref{8FutureWork}, I further explore these emerging challenges and discuss possible directions for future research.

    \chapter{From Theory to Action}\label{5implementation}

\section{Simulating Bridge Worlds}

To train and evaluate the behavior of RBAMAs in the bridge setting -- while also addressing challenges that arise from modifications to the environment and supporting future investigations into relevant changes -- I developed a framework based on the Python Gymnasium package. This framework enables the generation of customizable simulations of the bridge environment with a wide range of configuration options. The following sections provide a detailed overview of its implementation.

\begin{table}[h!]
    \centering
    \begin{tabular}{|p{5cm}|p{8cm}|}
        \hline
        \textbf{Configuration Option} & \textbf{Description} \\
        \hline
        Grid Width ($W$) & Defines the width of the grid world. \\
        Grid Height ($H$) & Defines the height of the grid world. \\
        Number of Bridges & Specifies the number of bridges (up to 3). \\
        Goal Position ($pos_{\text{G}}$) & Determines the agent’s target destination. \\
        IDs moving Persons & Places persons on the map and assigns them movement along a predefined path based on their ID (1-3: crossing the corresponding bridge numbered from left to right and 4 strolling along the lower shore). \\
        Static Persons & Sets the positions of persons on the map that do not move. \\
        Person Reappearance Time ($t_{r}$) & Sets the delay before a person re-enters the map after leaving. \\
        Dangerous Spots ($\mathcal{D})$ & Defines locations where persons might fall into the water. \\
        Probability of Falling ($P_{\text{fall}})$ & Probability that a person falls into the water when reaching a dangerous spot.  \\
        Probability of Being Pushed $(P_{\text{push}})$ & Probability that persons are pushed off the bridge when the agent collides with them. \\
        Drowning Behavior & Determines if drowning follows a random process, with a configurable probability applied at each time step, or occurs after a specified, fixed number of time steps. \\
        Visibility of the State & Controls how much information about the environment is included in the observation (set through observation wrappers.) \\
        \hline
    \end{tabular}
    \caption{Configuration Options for the Bridge Environment}
    \label{tab:config_options}
\end{table}

The simulations of the bridge \env{} are set up as grid worlds, where each tile is identified by a coordinate in the set~$\mathcal{P}$:
$$\mathcal{P} = \{(x, y) \mid x \in \{0, \dots, W-1\},\ y \in \{0, \dots, H-1\} \}$$
Here, $W$ denotes the width and $H$ the height of the grid, with $x$ and $y$ representing the horizontal and vertical positions, respectively. Consequently, the positions of the agent and the persons moving across the map are discrete values within a predefined range. The overall size of the map is configurable by specifying $W$ and $H$ as parameters.

There are three types of tiles in the environment: land tiles $\mathcal{L} \subset \mathcal{P}$, water tiles $\mathcal{W} \subset \mathcal{P}$, and bridge tiles $\mathcal{B} \subset \mathcal{P}$. Each layout consists of two landmasses, separated horizontally by water and connected by one or more bridges, reflecting the scenario described in \cref{2motivation}. The number of bridges can be configured between one and three bridges $(1 \leq n_{\mathcal{B}} \leq 3)$, and their arrangement is adjusted accordingly. The goal position is specified by $pos_{\text{G}} \in \mathcal{L}$, representing the target coordinates where the agent must deliver a package. This position can be set as a further configuration option. In addition, it is possible to include up to $n_{\mathcal{B}} + 1$ moving persons $P = \{p_1, \dots, p_{n_{\mathcal{B}}+1}\}$ as elements in the \env{}. Each moving person follows a predefined path: for each bridge, a person can be assigned to cross it, and one additional person may be configured to stroll along the lower shore. Additionally, any number of static persons can be positioned at arbitrary locations on the map. 

The agent is restricted to stay within the boundaries of the grid world, while moving persons leave the map upon completing their trajectories. Once a person exits the map, they are relocated to the out-of-map position, denoted by~$\bot$. This also applies to persons who drown. Consequently, by defining the union of all map coordinates with the out-of-map position as
$$
\mathcal{P}_{\text{ext}} = \mathcal{P} \cup \{\bot\},
$$
each person $p_{i}$ always occupies a position $\text{pos}_{p_{i}} \in \mathcal{P}_{\text{ext}}$. 

As described earlier, each moving person $p_i$ follows a predefined trajectory
$$
\tau_i = (pos_{i,0}, pos_{i,1}, \dots, pos_{i,m_i}),
$$
where $pos_{i,l_{i}} \in \mathcal{P}$ for all $0 \leq l_{i} \leq m_{i}$. The moving person enters the map at $pos_{i,0}$ and continues along its path until it leaves the map at $pos_{i,m_i}$. After exiting the map, the person remains at the out-of-map position~$\bot$ for $t_{r} \in \mathbb{N}$ time steps, after which it re-enters the map at its starting position $pos_{i,0}$. The parameter $t_{r}$ is configurable. Dangerous locations, where persons risk falling into the water, are defined as a subset of the map coordinates:
$$
\mathcal{D} = \{pos_{\text{d}_{0}}, pos_{\text{d}_{1}}, \dots, pos_{\text{d}_{k}}\}, \quad \mathcal{D} \subseteq \mathcal{L}.
$$
The set of dangerous spots~$\mathcal{D}$ is configurable and can be adjusted as well across different versions of the environment.

When a person reaches a location in $\mathcal{D}$, there is a probability $P_{\text{fall}}$ that they will fall into the water when attempting their next step. Alternatively, a person may end up in the water if they are pushed off a bridge by the agent. Specifically, if the agent collides with a person on a bridge tile, it may push them into the water with a probability $P_{\text{push}}$. Once a person has fallen or been pushed into the water, they will drown after $t_{d} \in \mathbb{N}$ time steps. The parameters governing these dynamics -- $P_{\text{fall}}$, $P_{\text{push}}$, and $t_{d}$ -- are configurable. By setting $P_{\text{fall}} = 1$ and $P_{\text{push}} = 1$, a deterministic version of the environment can be created, where persons always fall or are pushed into the water under the respective conditions. Additionally, the a configuration option to disable the drowning timer entirely is included. In this case, persons who fall into the water stay there indefinitely, without the ability to get out on their own. Further, an environment wrapper enables randomized drowning behavior, modifying the base environment by applying a probability-based mechanism, where persons in the water may drown randomly after each time step according to a configurable probability $P_{\text{drown}}$, which can be modified as well. The modifiable parameters of the environment are summarized in \ref{tab:config_options}.

In addition to the ability to create different versions of the \env{}, the reset function can be called with parameters that modify the initial state distribution. It supports the following options:
\begin{enumerate}[label=\roman*)]
    \item resetting the \env to a fixed initial state $s_0$, where the agent is placed in the upper-left corner of the map and each person is positioned at the start of their respective trajectory,
    \item manually specifying the initial positions of the agent and the persons on the map,
    \item randomly sampling their positions, with the agent’s position drawn uniformly from $\mathcal{L}$ and each person $p_i$ initialized at a position sampled from their trajectory $\tau_i$.
\end{enumerate}

The reward returned by the Gymnasium environment after each step can be customized for any version of the environment, allowing the agent to be incentivized toward different learning objectives. Specifically, the reward can be configured to encourage the agent to:
\begin{enumerate}[label=\roman*)]
    \item fulfill the instrumental goal,
    \item rescue persons in danger,
    \item avoid pushing persons from the bridge, or
    \item pursue all of these objectives simultaneously in a multi-objective learning setting.
\end{enumerate}

While the reason-based approach proposed in this work does not rely on ethical reward signals for moral-decision-making, they are included for two key reasons: first, to enable comparison with alternative approaches such as MORL; and second, because in the operational framework of the RBAMA introduced in this work, the translation from moral obligations into executable sequences of primitive actions is learned using reinforcement learning techniques. 

To support these different objectives, the framework provides a set of reward functions, each corresponding to a specific reward type:
$$
\mathcal{R}_{\text{type}} = \{R_{\text{instr}}, R_{\text{resc}}, R_{\text{push}}, R_{\text{MO}}\}.
$$

The reward function $R_{\text{instr}}$ assigns a reward of $1$ when the agent successfully reaches the instrumental goal. When $R_{\text{resc}}$ is selected as the active reward function, the agent is rewarded whenever the environment transitions from a state where persons are in the water to a state where all persons have been rescued. If $R_{\text{push}}$ is selected, the agent receives negative feedback each time it pushes a person into the water. The return values of the moral reward functions are configurable via a dedicated configuration file. Formally,
\begin{align*}
R_{\text{instr}}(s, a, s') &=
\begin{cases}
1, & \text{if } pos_a' = pos_g, \\
0, & \text{otherwise}.
\end{cases} \\[10pt]
R_{\text{resc}}(s, a, s') &=
\begin{cases}
r_\text{resc}, & \text{if } \left( \exists p_i \in \mathcal{P} \text{ s.t. } pos_{p_i} \in \mathcal{W} \text{ in } s \right) \\
& \quad \text{and } \left( \forall p_i \in P, pos_{p_i} \notin \mathcal{W} \text{ in } s' \right), \\
0, & \text{otherwise}.
\end{cases} \\[10pt]
R_{\text{push}}(s, a, s') &=
\begin{cases}
r_\text{push}, & \text{if } (pos_a' \in B) \text{ and } (\exists p_i \in \mathcal{P} \text{ s.t. } pos_{p_i} = pos_a' \text{ in } s \\
& \quad \text{and } pos_{p_i} \in  \mathcal{W}  \text{ in } s') \\
0, & \text{otherwise}.
\end{cases}
\end{align*}
where $r_\text{resc} > 0$ and $r_\text{push} < 0$ are set via the configuration file. In the multi-objective setting, the reward function $R_{\text{MO}}$ returns a reward vector composed of multiple objectives:
$$
R_{\text{MO}} = (R_{\text{instr}}, R_{\text{resc}}, R_{\text{push}}).
$$ 
Additionally, the environment provides feedback on whether the moral constraint of not pushing persons into the water has been violated by incorporating a safety specification modeled within the CMDP framework (cf. \cref{3bridgeCMDP}). Following the Safety Gymnasium conventions, this feedback is included as additional information in the return function. Specifically, the environment returns a value of $1$ if the agent violates the specification by pushing a person into the water, and $0$ otherwise.

Regarding the observations returned after each step, the base version of the environment provides a one-dimensional array that concatenates the positions of the agent, the water tiles and the persons into a single feature vector. This representation is particularly suited for processing by models such as feedforward neural networks (FNNs). To support training with convolutional neural networks (CNNs), the environment can be wrapped to return a multi-dimensional array in which the positions are one-hot encoded. By default, the drowning timer is not included in the observation. However, a wrapper is available that exposes this information when required. Another optional wrapper can restrict the agent’s perception to a limited view window, reducing the amount of observable information and introducing partial observability. Furthermore, the bridge environments include an implementation of the labeling function $\mathcal{L}$, which outputs the morally relevant propositions that hold in the current state.

\begin{figure}[h]
    \centering
    \begin{subfigure}[t]{0.45\textwidth}
        \centering
        \includegraphics[width = 7cm]{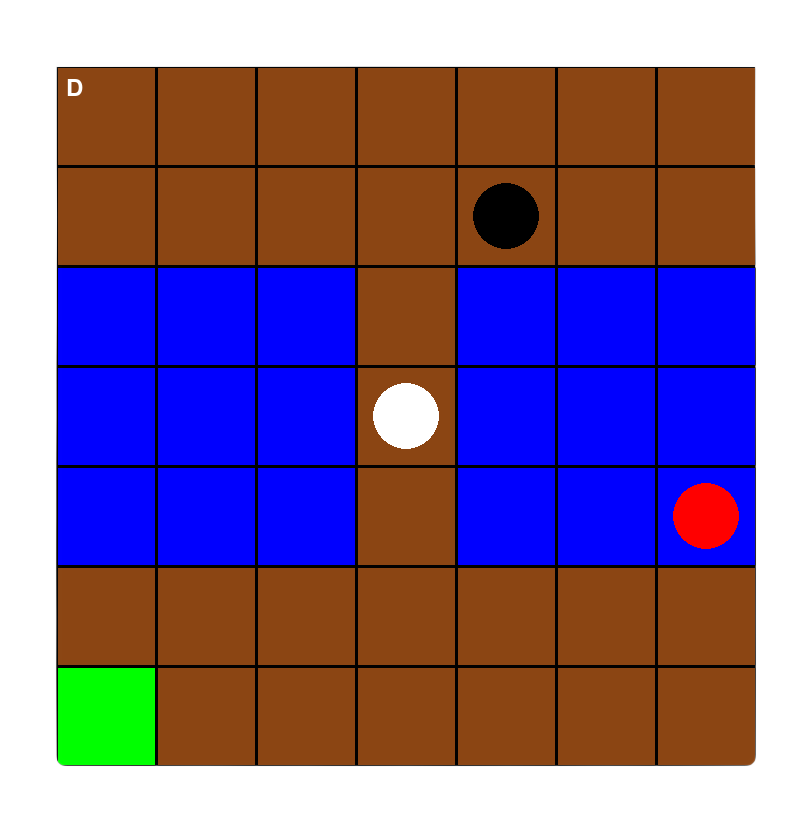}
        \caption{State Visualization of the Bridge Environment}
        \label{fig:vis_example}
    \end{subfigure}
    \hfill
    \begin{subfigure}[t]{0.45\textwidth}
        \centering
        \includegraphics[width = 7cm]{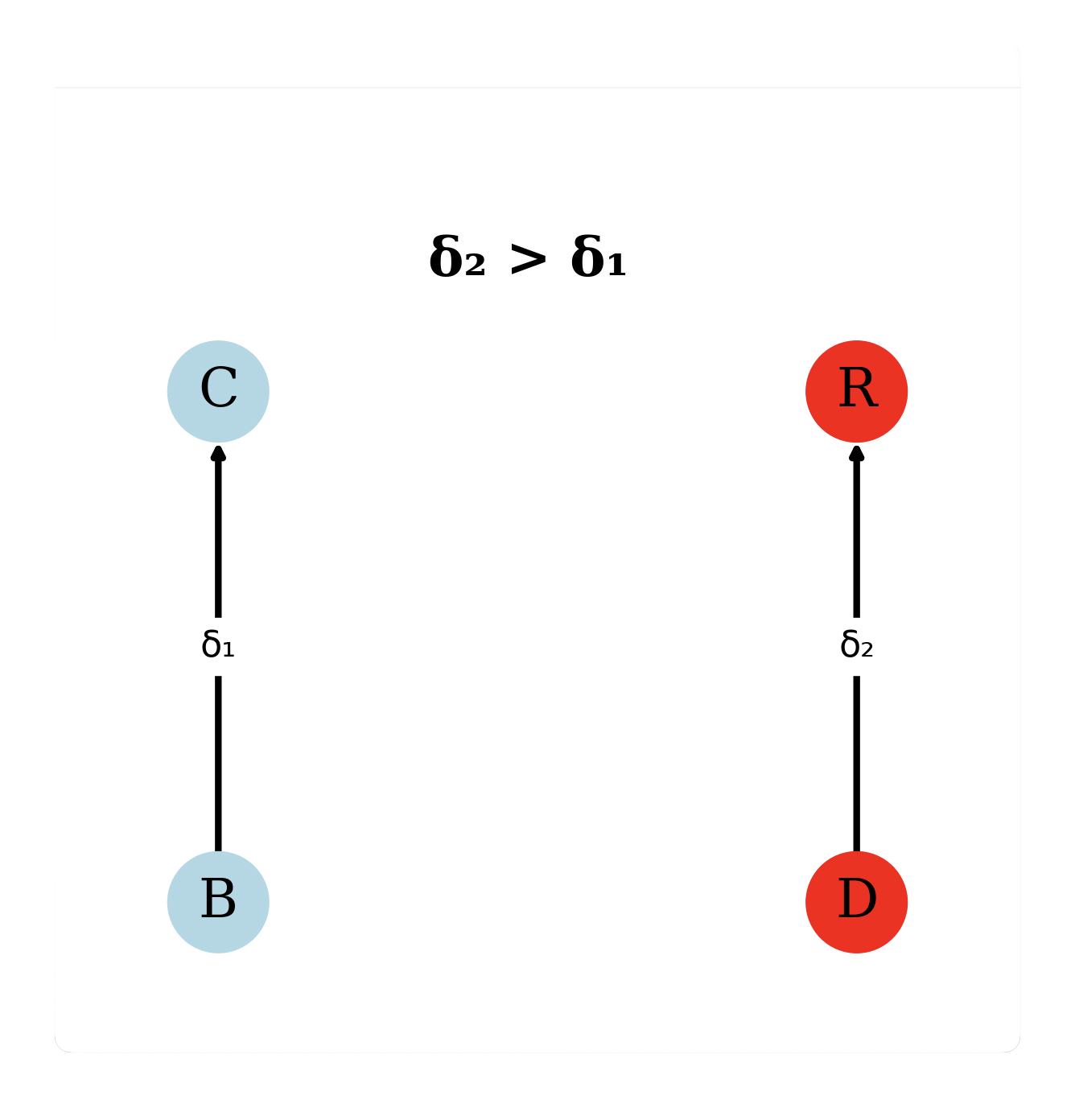}
        \caption{Visualization of the RBAMA's Reasoning}
        \label{fig:reasoning_example}
    \end{subfigure}
    \caption{Examples of the Visualization Options}
    \label{fig:two_images}
\end{figure}

Finally, the implementation of the \env also includes a render function, visualizing the grid world simulation of the bridge \env (\cref{fig:vis_example}) as well as a visualization tool for the agent's reasoning (\cref{fig:reasoning_example}). The visualization of the environment provides an intuitive overview of the current state. The goal position is highlighted green, indicating the location the agent aims to reach. The agent is represented by a white circle. Persons are depicted based on their current status: a black circle indicates a person located on land, while a red circle signifies a person who is in the water, potentially in danger of drowning. Additionally, relevant labels are displayed at the upper left corner to provide the contextual information about which morally relevant fact hold true in the current state of the \env. The agent’s learned reasoning theory is illustrated in form of its graphical structure. All learned default rules are visualized as connections between nodes, reflecting the structure of the agent’s moral reasoning. Nodes representing morally relevant facts that hold true in the current state, as well as the overall derived moral obligations that the agent considers binding, are highlighted with red circles. In contrast, facts that do not hold true, along with moral obligations that are not part of the agent’s derived set of binding duties, are displayed as blue circles. Additionally, the learned priority order is included in the visualization as descriptive text.

\section{Realizing Artificial Moral Reasoning}\label{5artificialMoralReasoning}

I have developed an initial version of an RBAMA as theoretically developed in \cref{4reasonTheoryChapter} that implements an ethics module containing a reasoning unit. This unit invokes moral filters and moral policies, which are trained using neural networks. Additionally, I have implemented moral judges as rule-based modules that deliver feedback based on different prioritization schemes for moral obligations. The implementation of the RBAMA features an update mechanism that allows it to process this feedback, refine its reasoning, and thereby enhance its moral decision-making over time. While the implementation is currently adapted for and trained within the bridge environment simulation, the architecture is designed to support easy extension for training and evaluation in other environments.

\subsection{Implementing the Logic Behind Moral Reasoning}

At the core of the RBAMA lies its reasoning unit, together with the implementation of a logical framework that enables the agent to represent, learn, and refine a reason theory. The agent’s reason theory is structured as a graph, where one type of node represents propositions describing morally relevant facts, and another type represents moral obligations. The information about the partial order over the moral reasons of the agent is stored as edge data. 

The logic on which the reasoning unit operates, follows the formalization of reason-based decision-making within a reason theory. In addition, the agent’s reasoning can be extended to include world knowledge expressed in propositional form, enabling it to recognize logical implications relevant to its deliberation. For example, the agent could thereby be enabled to recognize logical implications such as \textit{``If there is a drowning person, then there is a person in need''}.

\subsection{Teaching Good Reasons}

To simulate the feedback mechanism for refining the RBAMA’s moral reasoning in the bridge setting, I have implemented \textit{moral judges} as modular components that provide feedback to the agent. Each judge is equipped with a hard-coded reason theory and a rule-based mapping from moral obligations to primitive actions in a given state. The implemented judges differ in the normative reasoning they teach the RBAMA. For conducting the experiments, I used a moral judge, that expected the RBAMA to learn to recognize both normative reasons relevant in the bridge setting, that is, it teaches the RBAMA a default rule $\delta_1$ based on which it infers a moral obligations $\varphi_1$ to avoid actions that might push persons from the bridge, and a default rule $\delta_2$ based on which it infers a moral obligation $\varphi_2$ to initiate efforts to pull persons out of the water when there are individuals in danger. Moreover, the moral judge used throughout the experiments instructed the RBAMA to prioritize $\varphi_2$ over $\varphi_1$ in cases of conflict, i.e., it provides feedback such the RBAMA learns the order $\delta_2 > \delta_1$.

Although they specify different sets of normative rules and teach different prioritizations, all moral judges use the same \textit{translator} to assess whether the RBAMA’s behavior conforms to its moral obligations. The translator is a rule-based module that interprets the moral obligations of rescuing a drowning person as requiring the RBAMA to take the shortest possible path toward the individual to pull them out of the water. This reflects a prima facie plausible strategy for achieving the moral objective, consistent with how a human -- possibly acting as a moral judge -- might approach the task. If multiple persons have fallen into the water, the translator interprets $\varphi_R$ as requiring the agent to rescue first the person who fell in earliest. Similarly, if the judge expects the RBAMA to uphold the moral constraint against pushing, it assumes -- on the basis of the translator’s specification of how this obligations maps onto trajectories of primitive actions -- that the RBAMA will refrain from stepping onto a bridge when another person is already present, once again reflecting a human-like, prima facie reasonable strategy. If the RBAMA has already entered the bridge, the judge expects it to stop moving. If, in the judge’s assessment, the RBAMA fails to fulfill an overall moral obligation, the corresponding feedback $(\varphi, X)$ is forwarded to the agent’s reasoning unit. Since the rewards the agent receives during its training process are not directly tied to the judge’s feedback, the RBAMA may frequently trigger the feedback mechanism, even if it correctly recognizes its moral obligations and has already learned a strategy for fulfilling them. In such cases, however, the RBAMA simply disregards the feedback. Otherwise, it integrates the reason into its reason theory if it was previously unknown and adjusts the partial order accordingly. This process consists of adding new edges to the graph structure and updating the corresponding edge data.

\subsection{Training Behavior for Objective Fulfillment and Constraint Adherence}\label{5trainingBehavior}

Once the agent has developed a well-structured reasoning theory, it can identify priorities and determine which constraints it must respect. However, it must also learn how to act accordingly while developing a strategy to achieve its instrumental goal. Therefore, the RBAMA must not only acquire an instrumental policy, but also strategies for rescuing persons from the water and predicting whether it risks pushing a person off the bridge. The RBAMA learns the instrumental policy using a DQL algorithm, with $R_{\text{instr}}$ as the reward signal. To align its actions with moral duties, the RBAMA employs neural networks specialized in fulfilling these duties. The agent learns to perform the moral task of rescuing persons from the water by training a dedicated rescue network, which develops a strategy to pull individuals out of the water as quickly as possible. The same DQL algorithm used for the instrumental policy is applied in this case, with $R_{\text{resc}}$ as the reward signal. Preventing the pushing of persons from the bridge requires the agent to anticipate the potential risks associated with its actions. For this, a contextual bandit (CB)\footnote{A contextual bandit is a learning framework that extends the classical multi-armed bandit problem by incorporating context information. In each round, the agent observes a context (also called a feature vector) that provides information about the current situation. Based on the observed context, the agent selects an action (or arm) from a set of available actions and receives a reward associated with that action. The goalis to learn a policy that maps contexts to actions in a way that maximizes the expected cumulative reward over time.\cite{lattimore2020bandit}} is trained to act as bridge-guarding network by predicting such violations. The CB is trained adversarially, identifying the action most likely to result in pushing a person off the bridge. Based on the predictions of this network, a shielding mechanism is implemented to prevent the agent from executing actions that would violate moral constraints. While there are arguments both for and against the use of pre- vs. post-shielding in ensuring constraint satisfaction (cf. \cite{konighofer2020}), neither approach holds an intrinsic advantage within the context of building an RBAMA. I decided to use the network's predictions for constructing a post-shield that ensures only morally safe actions, according to the corresponding moral constraint, are executed. Specifically, the RBAMA passes a preference list to the shield, which forwards the highest ranking action estimated by the CB to be conform with $\varphi_C$ to the environment. 

The implementation  of the shielding mechanism remains highly prototypical through relying on a CB to estimate the risk of violating the moral obligation not to push persons off the bridge as it assumes that in each state, at most one action might violate the moral safety constraint. In the bridge setting, this assumption could only be violated in very specific states -- specifically, if static persons were positioned on the map such that there is a state where one person stands behind and another in front of the agent on the bridge simultaneously. However, the state spaces of the instances used to conduct the experiments in this study did not include such configurations.  In addition, a second assumption -- namely that the agent does not engage in look-ahead planning -- is introduced by training the CB as an adversary. While this reflects a reasonable interpretation of what it means to conform to the moral safety constraint of not pushing -- namely, that the agent must not push a person off the bridge but may approach them closely -- it would be preferable to support look-ahead planning in scenarios where the designer or stakeholders intend to promote a more cautious interpretation of the constraint, such as incentivizing the agent to maintain greater distance. In \cref{7critic/shield}, I provide an overview of the Safe RL methodology that can be leveraged to teach an agent (morally) safe behavior. Applying established methods would also eliminate both of these assumptions. 

\section{Transforming Bridge Environments into Ethical Environments}

Alongside implementing an RBAMA for navigating the bridge setting, I have also integrated the option to generate ethical versions of the environment based on the approach of Rodriguez-Soto et al. This enables an experimental comparison between the two methodologies. To accomplish this, I built upon the codebase provided in their work \cite{rodriguez2021code}. The original code was designed to run specifically on the public civility game (cf. \cite{rodriguez-soto2021}). I modified it to dynamically handle different state spaces, including varying state sizes, making it compatible with any instance of the bridge setting.

Furthermore, the creation of an ethical environment, as proposed by Rodriguez-Soto et al. assumes that the state space is known. However, even for a single version of the bridge environment, the state space is too large to be manually enumerated. To address this issue and ensure compatibility with any instance of the environment without requiring manual adjustments, the state space is estimated by repeatedly resetting the environment to a state where the positions of the agent and persons are randomly sampled, followed by additional state sampling through random actions. The generation of versions of the bridge environment as ethical environments assumes that the estimated state space accurately reflects the actual state space.

Crucially, the approach is fundamentally limited in its applicability to stochastic settings. Their method relies on iterating over the convex hull of returns associated with a given policy, which presupposes a unique successor state for each action. While this is feasible in deterministic environments -- albeit already computationally demanding, as evidenced in \cref{6ScalingUp} -- it is not directly applicable to stochastic environments as successor states are no longer unique. Although an extension to stochastic environments could, in principle, be achieved by estimating the distribution over successor states, such an approach would substantially increase the already high computational burden and, more critically, would forfeit the formal guarantees that hold in deterministic settings. This limitation is further underscored by the fact that the authors restrict their experimental evaluation to single, deterministic instances. Consequently, for properly comparing the behavior of an RBAMA to that of an agent trained using a MORL approach in a stochastic setting, a different MORL framework must be selected.
    \chapter{Experimental Results}\label{6results}

I conducted several experiments to demonstrate both the strengths and limitations of the RBAMA implementation on a small test suite of bridge environment. To make the environment conditions more closely resemble a real-world scenario, the drowning timer was made invisible, introducing partial observability in the environment. Additionally, the reappearance timer for the persons on the map was set to 100, with each episode being truncated after 100 time steps such that the timer does not affect the dynamics of the environment.

All neural networks within the RBAMA -- the rescue network, the instrumental network, and the bridge-guarding network -- have CNN-based architectures. During training, the environment was reset by randomizing the positions of both the agent and the persons. Each network was trained independently in separate training loops, using its own batch of experiences. The rescue and instrumental networks were trained on batches sampled using an exponential epsilon decay exploration strategy, as part of the DQN algorithm implementation. For training the instrumental network, the environment was configured to terminate when the agent reached its instrumental goal and for training the rescue network, the environment was set to terminate once the agent fulfilled its moral obligation of rescuing each person from the water. To train the CB to predict the risk of pushing a person off the bridge, the episodes of the environment was set to terminate after a single step. The hyperparameters utilized during training were selected within suitable ranges to enhance stability and learning efficiency. Specifically, a learning rate of 0.001, a discount factor of 0.9, and synchronization of the target network every 1000 steps were implemented. Furthermore, the size of the replay memory was set to 1000 and a mini-batch size of 32 was chosen to handle the collected data for updating the neural networks. These hyperparameters were consistently applied to both the rescue and instrumental network. The CB operates without a discount factor and it also does not employ a target network. Apart from this, the CB shared training settings with the other networks, including identical learning rate, mini-batch size, and replay memory capacity. For the shielding mechanism, actions were deemed unsafe and consequently filtered out if the adversary model estimated their Q-values to exceed a predefined threshold of 0.8 in the given state. 

I've evaluated the agents on test episodes based on the sum of returns over $N$ episodes
$$\mathbf{R}_{\text{instr}}^{N} := \sum_{i=1}^{N} r_{\text{instr}}^{i}, \quad \mathbf{R}_{\text{resc}}^{N} := \sum_{i=1}^{N} r_{\text{resc}}^{i}, \quad \mathbf{R}_{\text{push}}^{N} := \sum_{i=1}^{N} r_{\text{push}}^{i}$$
for the respective reward types. To contextualize these reward values, I additionally recorded the number of episodes in which at least one person either fell or was pushed into the water, denoted as $\textbf{count}_{\text{resc}}^{N}$. When evaluating the performance in terms of returns of the RBAMA, I also tracked how often the agent detected a conflict between its moral oblgations, denoted as $\textbf{count}_{\text{conflict}}^N$.

Additionally to training an RBAMA, I trained standard DRL agents based on feedforward neural networks on ethical environments created from instances of the bridge environment. This means, I applied the algorithm of Rodriguez-Soto et.\ al to determine an ethical weight that guarantees the ethical optimality of the optimal policy and trained them using the corresponding scalarized reward function on those instances. The hyperparameters for training these \emph{multi-objective moral agents} (MOBMAs) were set to equal those used for the networks that compose the RBAMA.


\section{A Functional Reason-Based Artificial Moral Agent}\label{6originalStory}


\begin{figure}[h!] 
    \centering
    \includegraphics[width = 0.45\linewidth]{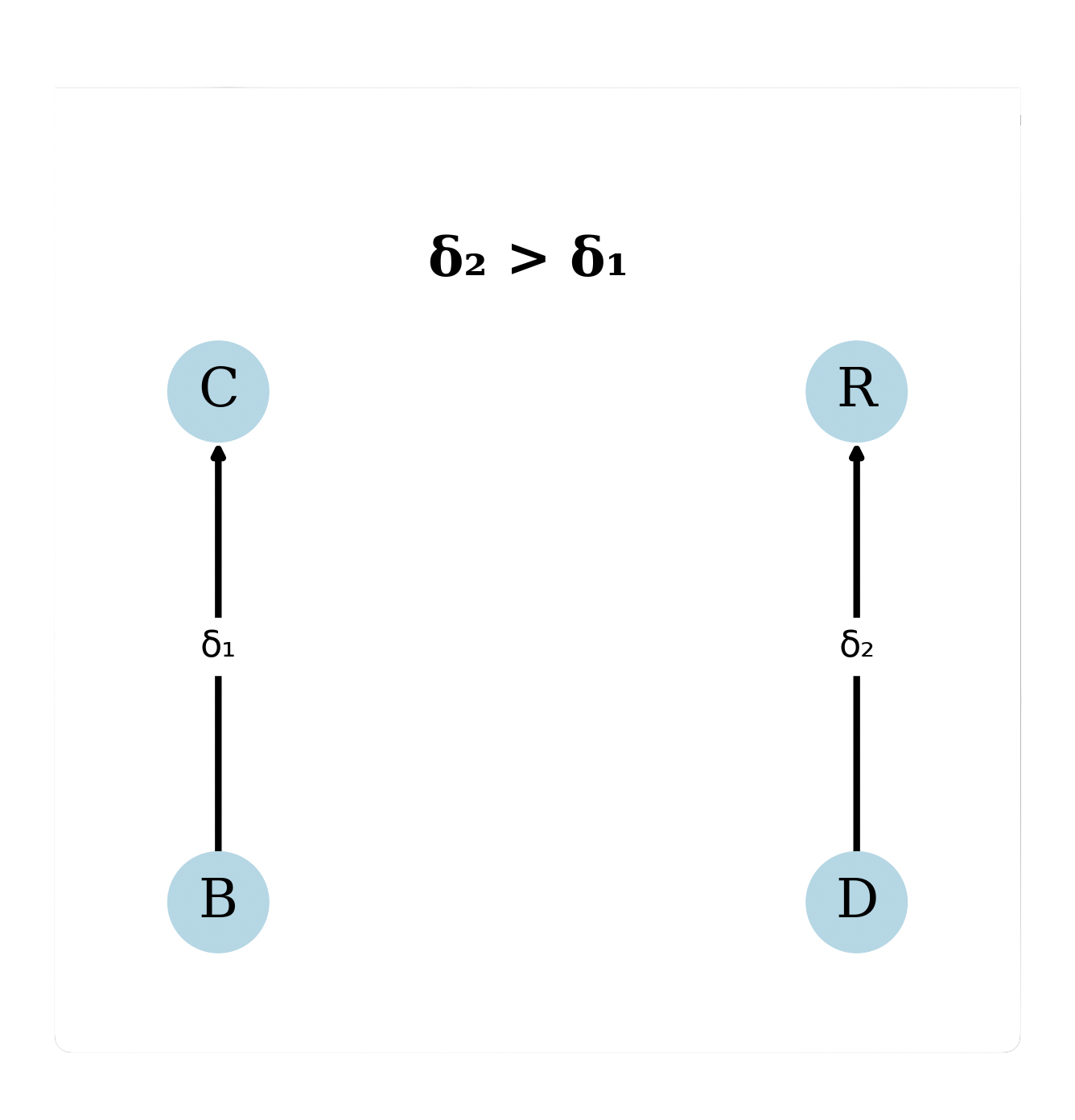}
    \caption{The RBAMA's learned reason theory}
    \label{fig:reas_theory}
\end{figure}

First tests were conducted on a \textit{moral dilemma simulation} (\cref{tab:env_moral_dilemma}) -- an instance of the bridge environment resembling the setting discussed in \cref{2challenges}. In this instance, in order to achieve its instrumental goal of delivering the package to the assigned location, the RBAMA must cross a bridge that is also traversed by a moving person. Additionally, a second person wanders along the lower shore, eventually reaching a dangerous spot where it falls into the water. Following the original story, the RBAMA must learn that it has a pro tanto reason not to push the person crossing the bridge into the water and a pro tanto reason to rescue the person who falls into the water along the lower shore. Furthermore, the moral judge teaches the RBAMA to prioritize rescuing the drowning person over avoiding pushing the individual crossing the bridge. Assume that, in the scenario under discussion, the moral judge has no access to perceptible cues on which a reliable estimate could be based regarding how long the drowning person can keep up over water or the likelihood that a collision would actually result in pushing the other person off the bridge. Under such conditions at least, the prioritization can arguably be regarded as constituting sound normative reasoning.

\begin{figure}[h]
    \centering
    \begin{subfigure}[t]{0.45\textwidth}
        \centering
        \includegraphics[width=7cm]{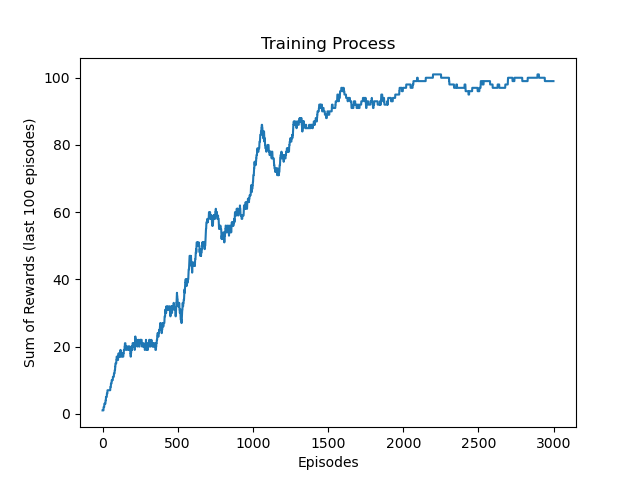}
        \caption{Training process of the instrumental network}
        \label{fig:train_instr_dilemma}
    \end{subfigure}
    \hfill
    \begin{subfigure}[t]{0.45\textwidth}
        \centering
        \includegraphics[width=7cm]{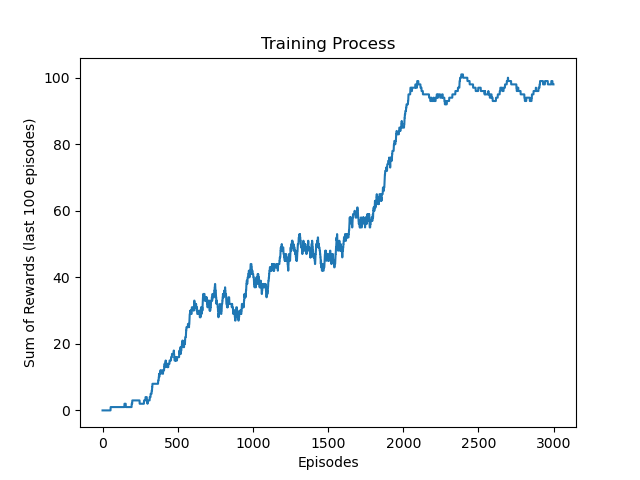}
        \caption{Training process of the rescuing network}
        \label{fig:train_resc_dilemma}
    \end{subfigure}
    \label{fig:two_images}
    \caption{Training of the RBAMA's networks}
\end{figure}

To develop strategies for the morally required behaviors, I trained the networks consecutively. First, the RBAMA’s rescue network was trained for 3,000 episodes (\cref{fig:train_resc_dilemma}); then, the agent's instrumental network was trained across 3,000 episodes (\cref{fig:train_instr_dilemma}) and finally the bridge-guarding network underwent training over 30,000 resets with randomized positions. Following the training of the networks, the RBAMA received feedback from the moral judge for 100 episodes, during which it updated its reasoning theory to ensure that it encountered situations where the moral obligations to rescue individuals and to prevent anyone from being pushed off the bridge come into conflict, thereby enabling it to learn the intended prioritization of the corresponding default rules as specified by the judge.

Visualizing the RBAMA's reason theory after the training phase confirms that it has successfully learned the reasoning taught by the moral judge (\cref{fig:reas_theory}). Furthermore, exemplary demonstrations of the RBAMA's reasoning illustrate its ability to make reason-based decisions grounded in morally relevant facts that supervene on the state of the environment. Positioning the RBAMA next to a person on the bridge, the presence of the person on the bridge -- a morally relevant fact -- is reflected in the labeling of the environment, i.e., $l(s) = \{B\}$. Observing the RBAMA’s reasoning during runtime using the visualization tool shows, that it correctly recognizes its moral obligation to avoid pushing the person off the bridge (\cref{fig:obl_wait}). It adapts its behavior accordingly by waiting for the person to take a step before continuing crossing the bridge. Likewise, when setting the environment to a state where a person has fallen into the water such that $l(s) = \{D\}$, the RBAMA derives its overall moral obligation to rescue the drowning person (\cref{fig:obl_resc}). Again, its normative reasoning guides the RBAMA's behavior -- this time by prompting it to move towards the person and pull them out of the water.

\begin{figure}[h!]
    \centering
    \begin{subfigure}[t]{0.45\textwidth}
        \centering
        \includegraphics[width=7cm,height=7cm]{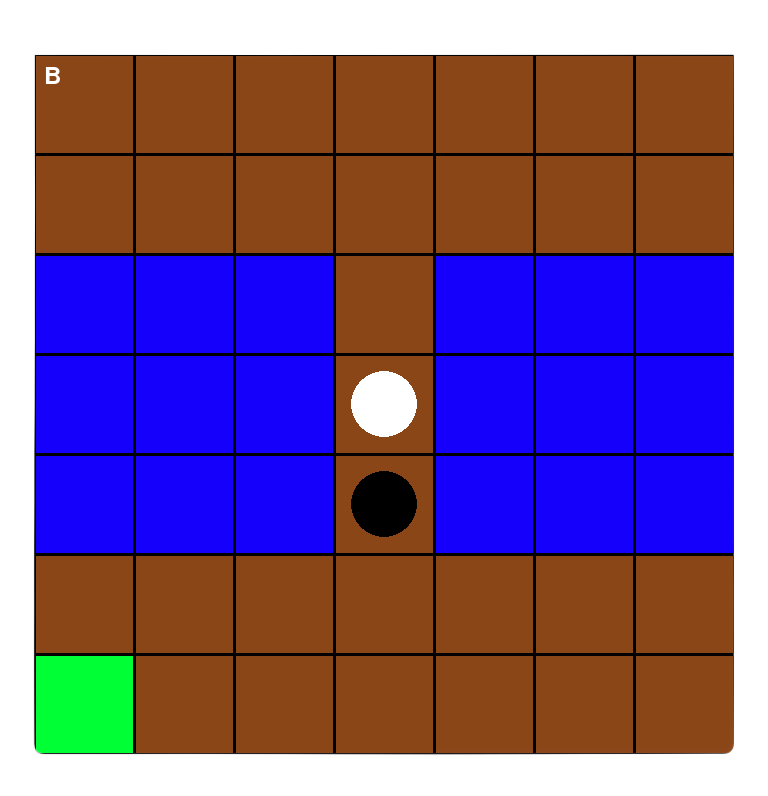}
        \caption{State of the environment with $l(s) = \{B\}$}
        \label{fig:state_push}
    \end{subfigure}
    \hfill
    \begin{subfigure}[t]{0.45\textwidth}
        \centering
        \includegraphics[width=7cm,height=7cm]{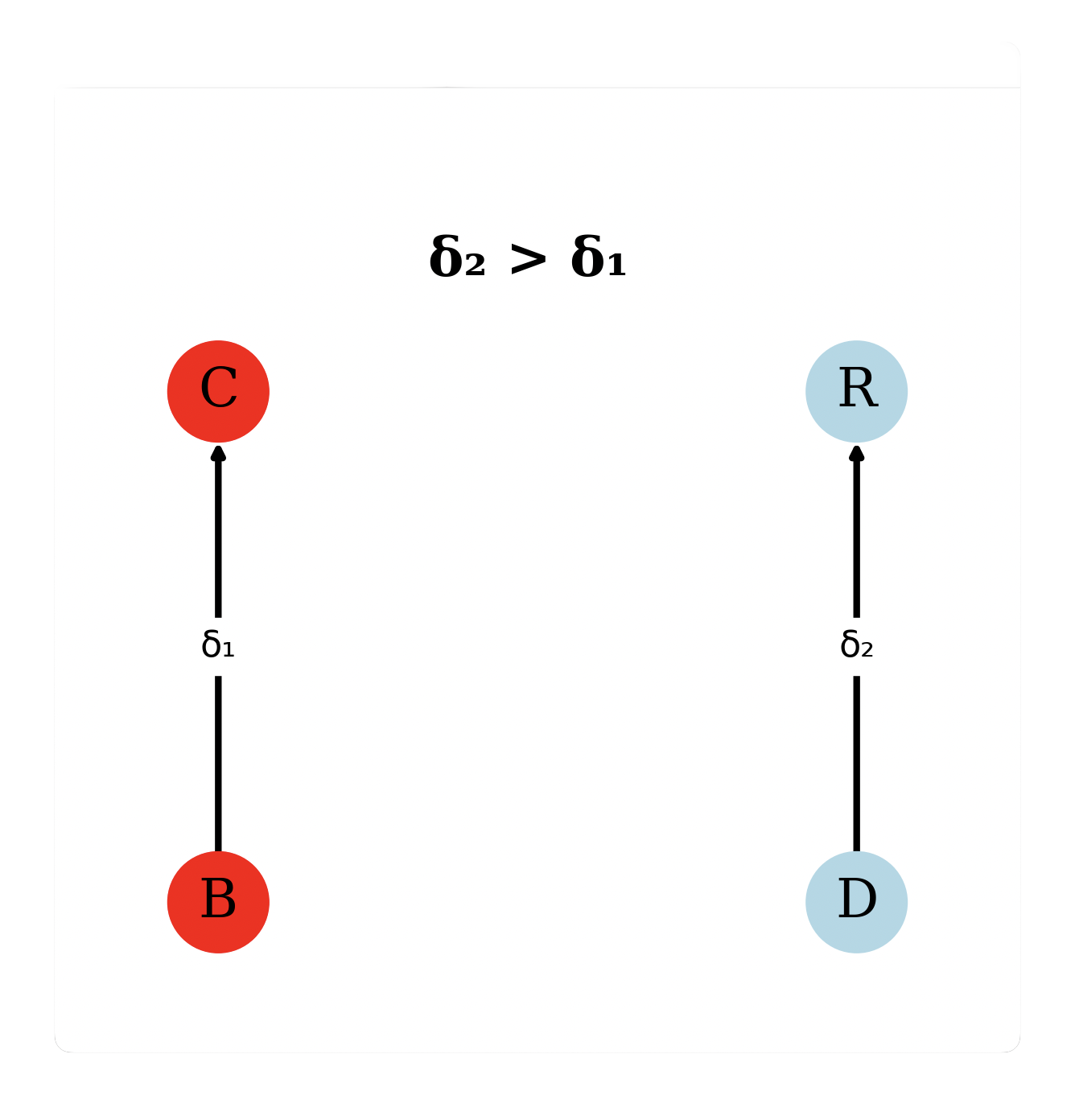}
        \caption{Reasoning of the RBAMA deriving $\{\delta_1\}$ as proper scenario and thereby accepting $C$ as overall moral obligation}
        \label{fig:resaoning_push}
    \end{subfigure}
    \caption{The RBAMA's reasoning under $l(s)  = \{B\}$ prior to not moving for one time step due to being prevented from executing $\textsf{down}$}
    \label{fig:obl_wait}
\end{figure}

Resetting the environment to a state (see \cref{fig:dilemma_state}), where $l(s) = \{B, D\}$  and in which further the moral obligations of pulling the person out of the water and not pushing a person from the bridge are conflicted yielded the most interesting test case. This \textit{dilemma state}, closely resembles the situation described in \cref{2motivation}  in which the agent confronts these two conflicting moral obligations within a simplified real-world scenario. As the RBAMA was successfully trained using feedback from the moral judge, it derives an overall moral obligation to rescue the drowning person, while excluding the conflicting lower-order reason against pushing from its overall binding moral considerations. Most significantly, observing this behavior in conjunction with the RBAMA’s reasoning shows that the RBAMA prioritized its moral obligations correctly -- i.e., in accordance with the prioritization it had acquired through feedback from the moral judge -- and acted accordingly. 

\begin{figure}[h!]
    \centering
    \begin{subfigure}[t]{0.45\textwidth}
        \centering
        \includegraphics[width=7cm,height=7cm]{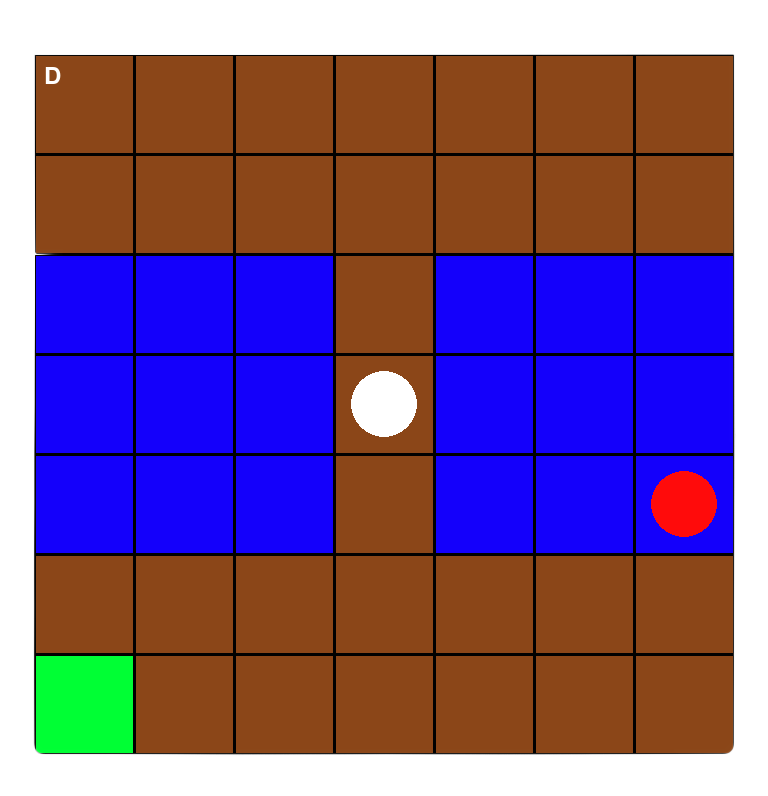}
        \caption{State of the Environment with $l(s) = \{D\}$}
        \label{fig:state_resc}
    \end{subfigure}
    \hfill
    \begin{subfigure}[t]{0.45\textwidth}
        \centering
        \includegraphics[width=7cm,height=7cm]{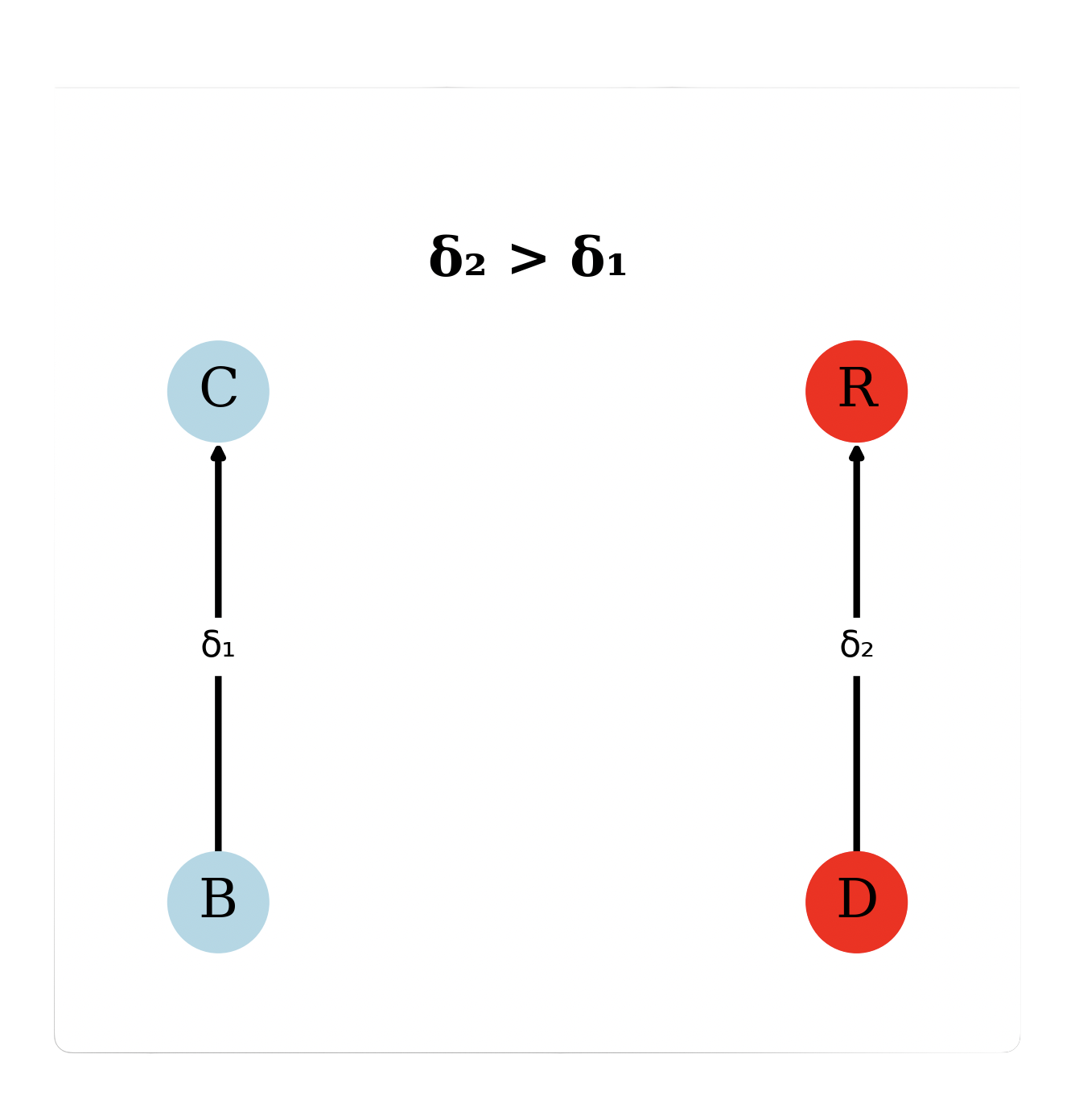}
        \caption{Reasoning of the RBAMA deriving $\{\delta_2\}$ as proper scenario and thereby accpepting $R$ as overall moral obligation}
        \label{fig:reasoning_resc}
    \end{subfigure}
    \caption{The RBAMA's reasoning under $l(s) = \{D\}$ prior to executing $\textsf{down}$ as its next primitive action following its rescuing policy}
    \label{fig:obl_resc}
\end{figure}

Through its demonstrated functionality in the targeted test cases, the fully trained RBAMA successfully realizes the original vision of developing an AMA grounded in normative reasons. It exhibits the capacity to learn and internalize  normative reasons as well as a prioritization among them. The successful implementation of this approach brings with it the key advantages inherent to its design -- most notably, the moral justifiability of its actions, its moral trustworthiness, and a certain degree of moral robustness. Detailed arguments supporting the claim that the RBAMA indeed has these properties is presented in \cref{7motivationRevisited}.

\begin{figure}[h!]
    \centering
    \begin{subfigure}[t]{0.45\textwidth}
        \centering
        \includegraphics[width=7cm,height=7cm]{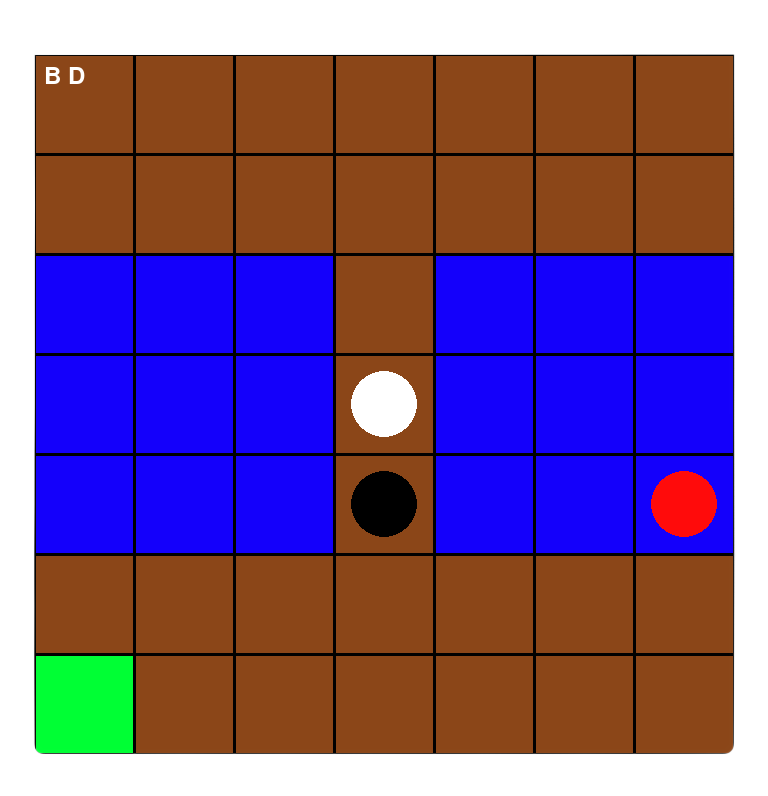}
        \caption{State of the Environment with \\ $l(s) = \{B, D\}$}
        \label{fig:dilemma_state}
    \end{subfigure}
    \hfill
    \begin{subfigure}[t]{0.45\textwidth}
        \centering
        \includegraphics[width=7cm,height=7cm]{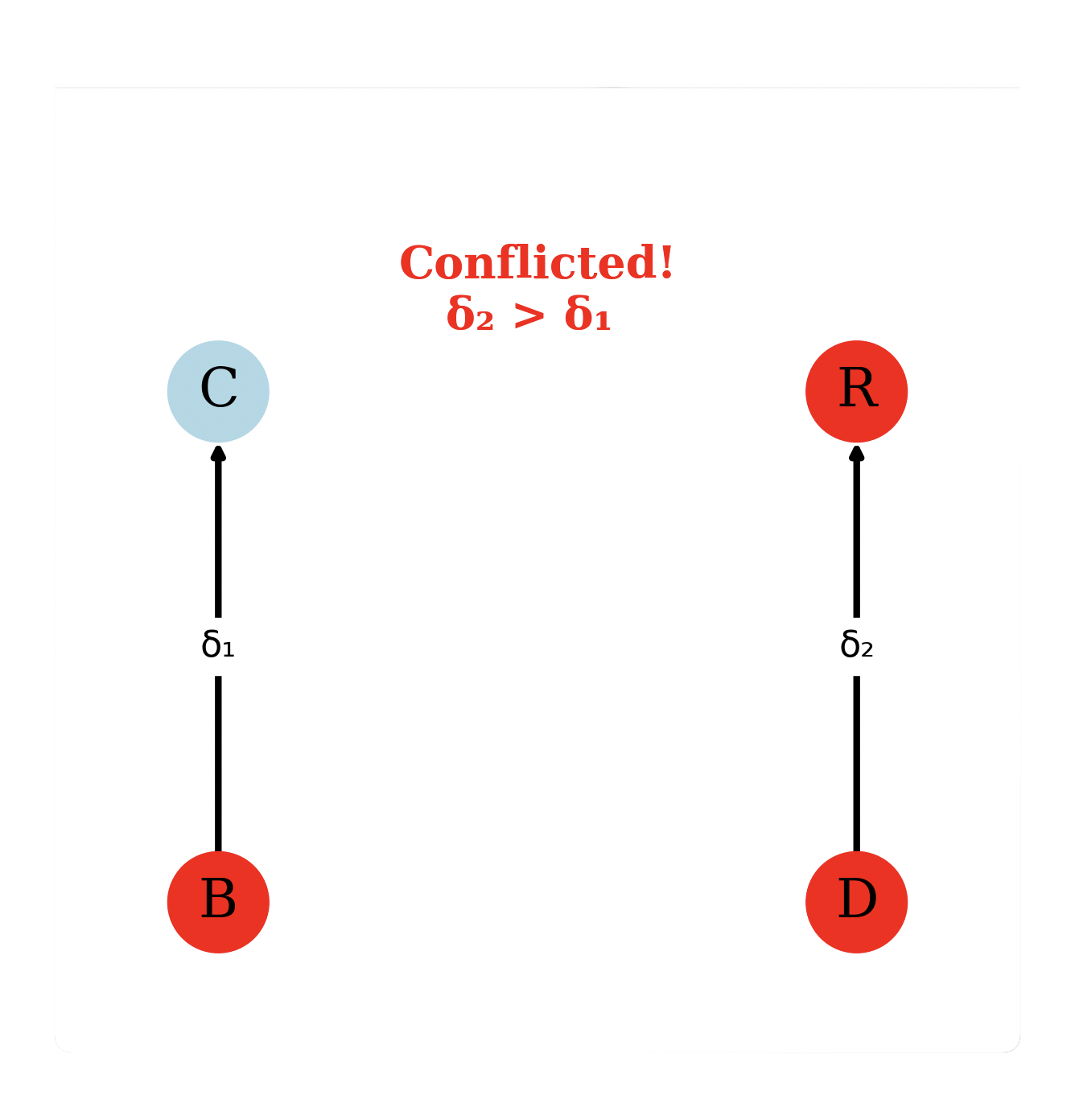}
        \caption{Reasoning of the RBAMA deriving $R$ as Proper Scenario}
        \label{fig:image2}
    \end{subfigure}
    \caption{The RBAMA's reasoning under $l(s) = \{B, D\}$ and $R$ conflicting with $W$}
    \label{fig:moral_dilmma_conflict}
\end{figure}

Additionally to evaluating the trained RBMA based on its reasoning, I evaluated it based on returns over 1,000 episodes in the ethical dilemma simulation. Out of the 1000 test cases, in 631 ($\textbf{count}_{\text{resc}}^{1000}=631$) at least one person fell into the water and the agent detected a moral conflict in 44 cases ($\textbf{count}_{\text{conflict}}^{1000}=44$). The total rewards obtained were $\mathbf{R}^{1000}_{\text{instr}} = 1000$, $\mathbf{R}^{1000}_{\text{push}} = -44$, and $\mathbf{R}^{1000}_{\text{resc}} = 631$. These results indicate that the RBAMA consistently achieved its instrumental goal, always reaching the designated delivery location. Furthermore, in every episode where a person fell into the water, the RBAMA consistently fulfilled its moral obligation to perform a rescue.\footnote{Since this obligation only arises when a person falls into the water or is pushed from the bridge, it does not apply in every episode. Consequently, it is not possible for the agent to receive \( r_{\text{resc}} \) in every episode.} The results also show that the RBAMA occasionally, though very infrequently, violates the moral constraint $\varphi_C$ of not pushing persons from the bridge. Notably, the number of such constraint violations corresponds exactly to the number of episodes in which its moral obligations were conflicting. This indicates that violations of $\varphi_C$ occurred exclusively in situations where the agent was also morally obligated to fulfill $\varphi_R$ -- that is, to rescue a drowning person. These actions reflect the order $\delta_2 > \delta_1$, which represents the prioritization conveyed by the moral judge and learned by the RBAMA. An analysis of the risk estimations produced by the bridge-guarding network reveals that exactly those state-action pairs resulting in the agent stepping onto a bridge tile already occupied by another person exceed the 0.8 threshold (\cref{fig:threshold}). Specifically, the trained adversary generates a shield that prevents the RBAMA from moving downward when positioned directly above a person on the bridge and from moving upward when positioned below a person on the bridge. This further supports the interpretation that the observed constraint violations stem from the RBAMA's adherence to the prescribed hierarchy of moral reasons, which permits risking a pushing action when it constitutes part of the rescuing network's strategy for pulling the person out of the water as fast as possible. 

\section{How the Agent Understands its Moral Duties}\label{6understandingMoralDuties}

Although the RBAMA determines its moral obligations through reasoning -- a symbolic process that enables explicit oversight via direct feedback -- it learns how to fulfill these obligations through reinforcement learning, relying on reward signals. The reward function is manually crafted and thus independent of the moral judge’s understanding of which sequences of primitive actions fulfill moral obligations. As a result, the RBAMA’s learned \textit{behavior} may not fully align with the judge’s expectations of how it ought to act -- even though its \textit{reasoning} is consistent with that of the moral judge.

In designing the reward system to discourage the bridge-guarding network from pushing persons off the bridge, the RBAMA learns to avoid such actions. However, its interpretation of this rule differs from that of the moral judge. Instead of refraining from movement when a person is present on the bridge, the bridge-guarding network learns only to ensure that it does not actively push anyone, because this is what is encoded in the reward design. As a result, it still moves onto the bridge while a person is present, since it is penalized only for actually pushing someone off. This differs from the behavior the moral judge expects from the RBAMA based on its rules, which define entering or moving on the bridge while a person is present as a violation of the RBAMA's moral obligation. The same applies to the moral task of rescuing persons from the water. The judge expects the RBAMA to take the shortest path towards the person who has fallen into the water first. However, the rescuing network is currently rewarded for successfully pulling all persons out of the water as quickly as possible, receiving a reward only when no persons remain in the water. This incentivizes the RBAMA to prioritize rescuing all individuals, but it does not differentiate between the order in which the rescues are performed. Instead, it learns to rescue the person closest to its position first. In both cases, the judge forwards feedback to the RBAMA, not because its reasoning is misaligned, but because it has a different understanding of how to fulfill its moral duties. 

\begin{figure}[h!] 
    \centering
    \includegraphics[width = 0.45\linewidth]{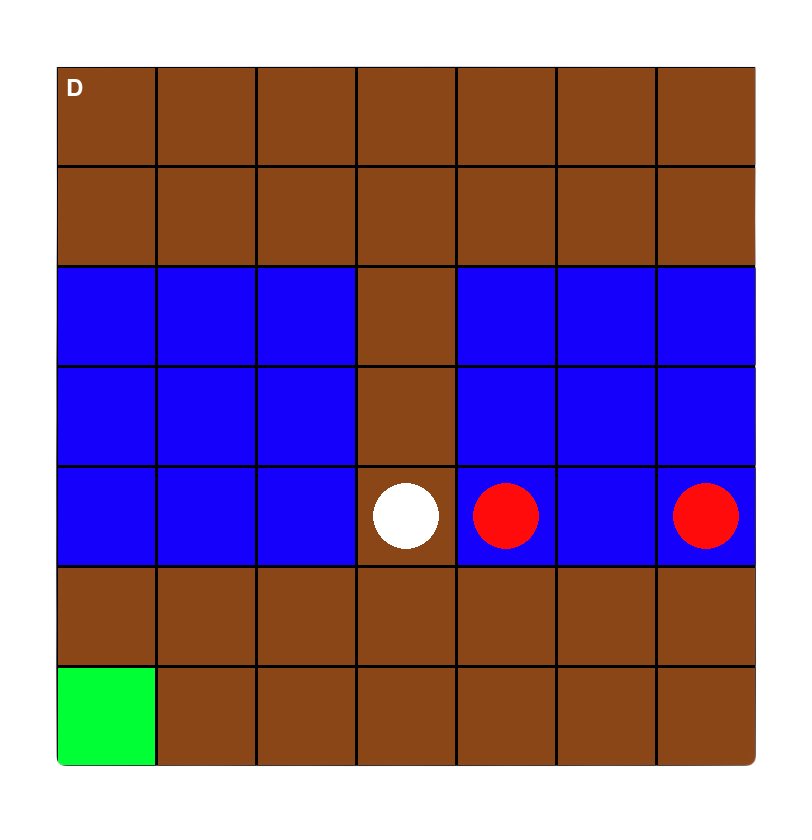}
    \caption{A scenario in which the agent, while attempting to rescue one person, pushes another person into the water and then prioritizes rescuing the person it pushed.}
    \label{fig:stuck_b2_v1}
\end{figure}

Furthermore, the RBAMA's learned behavior suggests that it does not prioritize rescuing a specific individual but instead seeks to achieve a state where no persons remain in the water as quickly as possible. This outcome is unsurprising, as it follows directly from the reward structure. As a result, the RBAMA first rescues individuals it has pushed into the water, which, given the scenario, appears reasonable. Nonetheless, one might prefer an agent that goes beyond an undifferentiated "save everyone" approach, instead developing the capacity to distinguish between different drowning individuals. This could be achieved by teaching the RBAMA a new default rule, the conclusion of which encodes a more fine-grained moral obligation -- specifying, for instance, which particular person, or more generally which moral patient, it ought to prioritize for rescue. The topic is further explored in \cref{7rewardFunction}. 

\section{Training Nets to Work Together}\label{6training_nets_to_work_together}

\begin{figure}[h]
    \centering
    \begin{subfigure}[t]{0.45\textwidth}
        \centering
        \includegraphics[width=7cm]{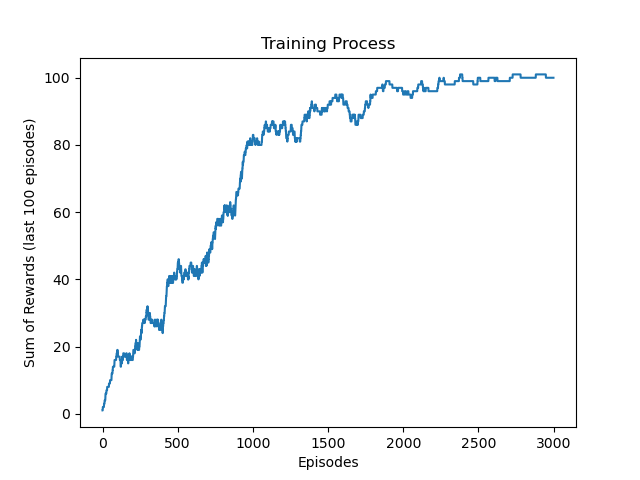}
        \caption{Training process of the instrumental policy under unshielded training}
        \label{fig:train_b2_v1}
    \end{subfigure}
    \hfill
    \begin{subfigure}[t]{0.45\textwidth}
        \centering
        \includegraphics[width=7cm]{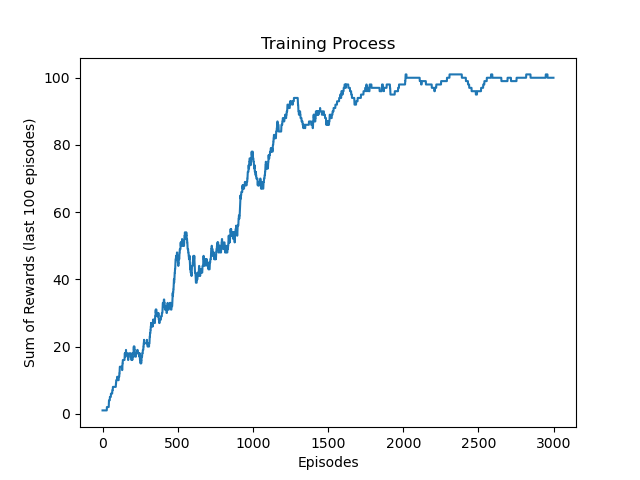}
        \caption{Training process of the instrumental policy under shielded training}
        \label{fig:train_b2_v1}
    \end{subfigure}
     \caption{Training process shielded vs. unshielded}
    \label{fig:two_images}
\end{figure}

The modular architecture of the RBAMA proposes itself as a means to fully decouple the instrumental policy from the moral decision-making. However, this is not possible for moral obligations that require the agent to ensure the non-violation of moral constraints \textit{while} pursuing its instrumental policy. Since this is achieved through the shielding mechanism, the bridge-guarding network must learn a policy that does not include trajectories it would be prevented from following once the shield is activated. This issue is also encountered by Neufeld et al.\ in their evaluation of an RL agent shielded by a normative supervisor (\cite{neufeld2022a}).

\begin{figure}[h!] 
    \centering
    \includegraphics[width = 0.45\linewidth]{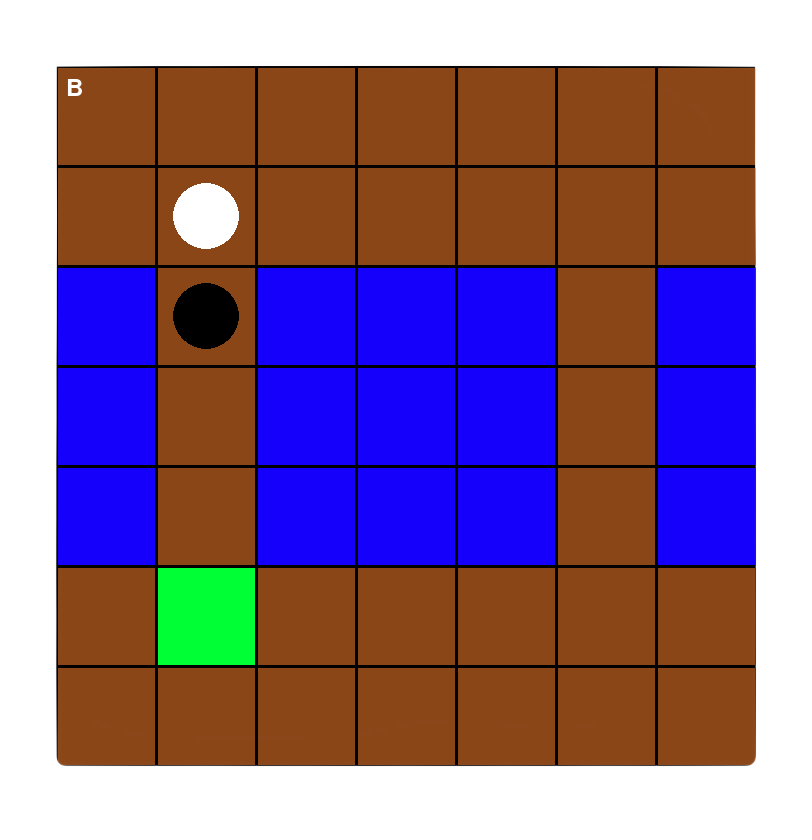}
    \caption{The agent repeatedly trying to enter the left bridge}
    \label{fig:stuck_b2_v1}
\end{figure}

The problems resulting from ignoring this issue can be observed in an RBAMA trained without shielding on the \textit{blocked left bridge simulation} -- a version of the environment that includes two bridges, one of which is permanently blocked by a person that does not move (\cref{tab:env_left_bridge_blocked}). During training, the RBAMA learns to take the path across the bridge -- despite the presence of the person -- when being placed in the upper-left area of the map (e.g., in state $s_0$), since this route offers the shortest path to its goal position. However, once the shield is activated, the RBAMA becomes unable to reach its goal from these states because the learned path is now blocked. It repeatedly attempts to enter the right bridge, only to be prevented from doing so (\cref{fig:stuck_b2_v1}). When tested over 1,000 episodes, with the environment resetting so that the RBAMA starts at $s_0$, the total rewards accumulated across all episodes are $\mathbf{R}^{1000}_{\text{instr}} = 0$ and $\mathbf{R}^{1000}_{\text{push}} = 0$ confirming the RBAMA’s failure to achieve its goal.

The issue can be addressed by training the policy while applying the shield constructed through $\pi_{\theta_W}$, shifting away from the notion of fully separating instrumental decision-making from moral reasoning. Instead, this approach enables the RBAMA to consider moral constraints during learning its instrumental policy. Training the RBAMA under these restrictions yields the expected outcome: when the environment is reset to \( s_0 \), itachieves total rewards of \( \mathbf{R}^{1000}_{\text{instr}} = 1000 \) and \( \mathbf{R}^{1000}_{\text{push}} = 0 \). This demonstrates that (i) the RBAMA consistently reaches its goal and (ii) learns to take a path across the right bridge, thereby avoiding pushing the person from the bridge.

\section{Balancing Moral and Instrumental Objectives for Building an Ethical Bridge Environment}\label{6RodriguezSoto}

To complement the assessment of the RBAMA’s performance by placing it into context with established methodologies, I conducted a comparative evaluation against the MORL-based approach by Rodriguez-Soto et al.  I set the elements of the vectorial reward function $\vec{R} = (R_{0}, R_{\mathcal{N}} + R_{E})$ to $R_{0} = r_\text{instr}$, $R_{\mathcal{N}} = r_{\text{push}}$, and $R_E = r_{\text{resc}}$, to construct an ethical MOMDP as the basis for computing an ethical embedding. The algorithm determines the minimum ethical weight required to construct a scalarized version of the environment in which the optimal policy is guaranteed to be ethically optimal. This weight reflects the relative importance that must be assigned to moral obligations relative to the agent’s instrumental goals. When restricting environment resets to the initial state $s_0$, the algorithm returned an ethical weight of 0.53. When considering the agent’s behavior across the entire state space, the ethical weight required to ensure that the optimal policy remains ethically optimal rose to 3.56. This higher weight corresponds to a state in which the agent is on the verge of reaching its instrumental goal, but must turn back to rescue drowning persons in order to behave ethically optimal (\cref{fig:return_to_rescue}). As a result, the reward returned for fulfilling this moral obligation must be sufficiently high to incentivize the agent to do so.

\begin{figure}[h!]
\centering
\includegraphics[width=0.45\linewidth]{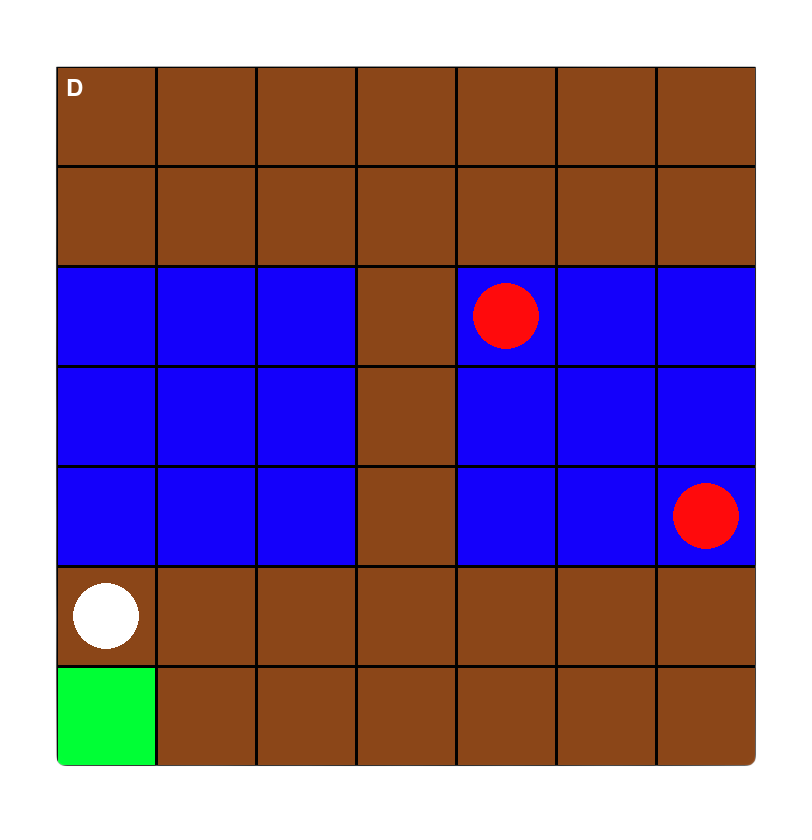}
\caption{Visualization of the state that requires the highest ethical weight to guarantee ethically optimal behavior under the optimal policy}
\label{fig:return_to_rescue}
\end{figure}

I trained a MOBMA for 30,000 episodes on the ethical environment generated from the moral dilemma simulation with random resets, using a scalarized version under an ethical weight of 3.45 (\cref{fig:training_b1_v1_3.45}). Since the MOBMA does not learn to engage in moral decision-making to determine its course of action, but instead learns an overall behavior solely based on expected returns, its performance can likewise only be evaluated based on returns. When tested over 1000 episodes with $\textbf{count}_{\text{resc}}^{1000}=874$, it achieved total rewards of $\mathbf{R}^{1000}_{\text{instr}} = 1000$, $\mathbf{R}^{1000}_{\text{push}} = 0$, and $\mathbf{R}^{1000}_{\text{resc}} = 874$.

Notably, this indicates that the agent successfully avoids pushing individuals off the bridge, manages to carry out rescues in nearly every episode, and achieves its instrumental goal in almost all cases. In direct comparison, the RBAMA, while it also reliably reaching its instrumental goal, occasionally pushed persons from the bridge and performed significantly fewer rescues than the agent trained on the ethically optimal version of the environment (cf. \cref{6originalStory}). 

This difference can be attributed to the fundamental distinctions between the two architectures. Unlike the RBAMA, which strictly prioritizes one moral obligation over the other, based on its normative reasoning, the MOBMA learns a multi-objective policy that leads to the overall highest reward. The environment is configured so that the agent can still reach the drowning person in time, even if it delays moving towards them by one time step. Because pushing a person results in a penalty of $-1$, while delaying the rescue by one time step incurs only a minimal cost in $r_\text{resc}$, the MOBMA maximizes its cumulative reward by holding its position until the person on the bridge moves on, thereby avoiding to violate its moral constraint. Consequently, learning a multi-objective policy to maximize the overall reward enables the agent to recognize that immediately rushing to rescue is not necessary for arriving in time and successfully saving the person.

Observing the behavior of the agent trained on the ethical optimal version of the environment further reveals that it anticipates the moment when the person strolling along the lower shore will fall into the water. Instead of proceeding to the package delivery location, the agent moves toward the designated drowning spot and waits for the person to fall, demonstrating an understanding of the version-specific transition dynamics of the environment. In contrast, the RBAMA remains focused on its instrumental goal until it identifies a moral reason to intervene. It only shifts its objective to rescuing once the person has actually fallen into the water. At that point, the rescue network takes control of the agent's behavior, directing it toward the drowning person. However, it often reaches its instrumental goal before the person falls into the water, which explains why it carries out fewer rescues. 

\begin{figure}[h!]
\centering
\includegraphics[width=7cm]{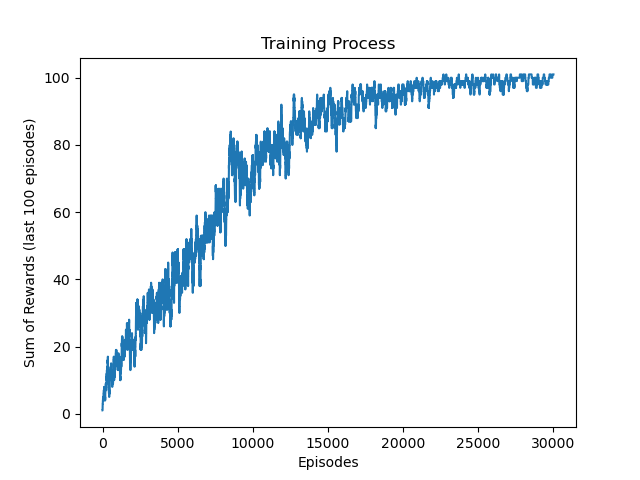}
\caption{Training process of the DRL agent on the scalarized Version of \cref{tab:env_moral_dilemma} setting $e_w = 3.45$}
\label{fig:training_b1_v1_3.45}
\end{figure}

\section{Deploying the RBAMA in Stochastic Environments}\label{6stochEnv}

The RBAMA was originally trained in a fully deterministic environment -- the moral dilemma simulation described in \cref{6originalStory}. However, the moral judge was assumed to lack knowledge of both the drowning time and the probability of pushing someone off the bridge. Under this epistemic limitation, prioritizing $\varphi_R$ over $\varphi_C$ arguably constitutes a form of sound normative reasoning. Accordingly, the moral judge was configured with rules that encoded this prioritization, such that the feedback incentivized the RBAMA to adopt the order $\delta_2 > \delta_1$. 

Importantly, since prioritizing between $\varphi_R$ over $\varphi_C$ is considered to constitute sound normative reasoning on the grounds of the moral judge's lack of knowledge concerning specific environment's dynamics, it remains sound -- on the same grounds -- in a stochastic setting, where drowning time and pushing probability are randomized. Furthermore, since no other aspects of the environment were altered aside from the introduction of stochasticity, the behavior the RBAMA acquired through training its networks remains adequate for translation of its moral obligations to primitive actions. Accordingly, the RBAMA trained in the original deterministic scenario can be directly deployed in a stochastic variant.

To empirically validate the claim, I instantiated a \textit{stochastic moral dilemma scenario} (\cref{tab:env_stochastic_moral_dilemma}), where the probability that a person in the water drowns after each time step was set to $P_{\text{drown}} = 0.1$, and the probability that the agent pushes a person off the bridge upon collision was set to $P_{\text{push}} = 0.5$. Observation of the RBAMA’s behavior in one concrete state with $\delta_1$ and $\delta_2$ conflicting (see \cref{fig:moral_dilmma_conflict}) -- in the further discussion referred to as the \textit{moral dilemma state} -- confirms the expectation that it still consistently prioritizes the moral obligation to rescue the drowning person by executing the learned rescue strategy. On this basis, it can be argued that the RBAMA demonstrates a generalization capability with respect to its moral behavior when transitioning from a deterministic to a stochastic setting.
To evaluate the RBAMA’s performance based on numerical return values, I conducted tests on dedicated evaluation episodes, each time resetting the environment to the moral dilemma state. In this test setup, the RBAMA achieved: $\mathbf{R}^{1000}{\text{instr}} = 1000$, $\mathbf{R}^{1000}{\text{push}} = -456$, and $\mathbf{R}^{1000}{\text{resc}} = 761$, with $\textbf{count}_{\text{resc}}^{1000} = 1000$ and $\textbf{count}_{\text{conflict}}^{1000} = 1000$. These results show that, by consistently rushing to the rescue, the RBAMA pushed the person off the bridge in nearly half of the cases, while successfully arriving in time to prevent drowning in 761 out of 1000 episodes -- thereby fulfilling its rescuing obligation in a majority of trials.

Adjusting the probability parameters does not affect the adequacy of the RBAMA’s acquired reason theory for navigating the environment, based on the same argument that assumes the judge evaluating the RBAMA’s behavior remains unaware of the exact probabilities. With $P_{\text{push}}$ set to $0.5$ and $P_{\text{drown}}$ set to $0.1$, evaluating the agent’s performance by again repeatedly resetting the environment to the moral dilemma state resulted in $\mathbf{R}^{1000}_{\text{instr}} = 1000$, $\mathbf{R}^{1000}_{\text{push}} = -458$, and $\mathbf{R}^{1000}_{\text{resc}} = 719$, with $\textbf{count}_{\text{resc}}^{1000} = 1000$ and $\textbf{count}_{\text{conflict}}^{1000} = 1000$.
As expected, a slight adjustment of $P{\text{drown}}$ from $0.1$ to $0.11$, thus inducing only a minor shift in the probability distribution, did not result in a notable change in the return values. Finally, I evaluated the RBAMA’s performance under a more pronounced shift in the drowning probability by setting $P_{\text{drown}} = 0.3$. The test results were $\mathbf{R}^{1000}{\text{instr}} = 1000$, $\mathbf{R}^{1000}{\text{push}} = -450$, and $\mathbf{R}^{1000}{\text{resc}} = 413$, with $\textbf{count}_{\text{resc}}^{1000} = 1000$ and $\textbf{count}_{\text{conflict}}^{1000} = 1000$. this shift led to failed rescue attempts in more than half of the trials. Nonetheless, in light of the earlier argument, the RBAMA continues to engage in sound normative reasoning by consistently prioritizing $\varphi_R$ over $\varphi_C$. However, compared to the previous cases, these results may more strongly provoke resentment toward the chosen priority setting.

In this context, an important consideration arises regarding the role of the moral judge in evaluating such performance data. If the test results are made visible to the moral judge, they could serve to inform and potentially revise the judge's own reasoning about the adequacy of the agent’s normative reasoning.  The introduction of such a feedback loop is further explored in \cref{7outcomestoreasons}. However, in this testing setup, the outcomes disclose environment dynamics that are \textit{instance-specific}, allowing for a reliable estimation of $P_{\text{push}}$ and $P_{\text{drown}}$, which had previously been assumed to remain unknown to the judge. Revealing these features is arguably problematic, as it undermines the beneficial abstraction layer that is preserved by keeping such information hidden from the evaluator. In fact, one might argue that a significant advantage of the RBAMA framework lies in inherently abstracting away from such instance-specific details when providing an RBAMA feedback about how to set priorities when observing its reasoning and behavior in simulated environments as further discussed in \cref{7overfittingMoralPrio}. 

Another consideration arises  when shifting focus from purely simulated environments to real-world deployments — or to simulations that incorporate probability distributions grounded in real-world statistics. In such settings, cues such as visible signs of distress or frantic movements by the drowning person may indicate a limited remaining survival time, thereby heightening the urgency to intervene. Conversely, calm behavior might suggest that the individual is likely to remain afloat for a longer period. Since such statistical information offers morally relevant insights, it can be argued that the return values from test runs conducted in such environments \textit{should} inform the moral judge’s reasoning. This issue is explored in greater detail in \cref{7relevanceConsequence}. 

Nevertheless, even if these statistical cues provide valuable input for the judge, it remains unclear how precisely they should shape the agent’s own decision-making processes. In particular, the question arises as to whether and when a shift in the probability distribution — for example, an increased likelihood for the person in the water to drown — becomes morally significant enough to justify an adjustment in the RBAMA’s prioritization between $\varphi_C$ and $\varphi_R$. This issue connects closely to the broader question of moral robustness, which is further discussed in \cref{7moralRobustness}.

\section{Ethically Optimal Behavior: Mission Accomplished?}\label{6EthicallyOptimalBehavrio_MissionAccomplished}

The experimental findings show that both the RBAMA and the MOBMA consistently prioritize their moral duties over their instrumental objectives. This strict prioritization has largely been accepted as a basic requirement for constructing AMAs. However, it is better understood as an assumption open to debate rather than an unquestionable fact.

\begin{figure}[h]
    \centering
    \begin{subfigure}[t]{0.45\textwidth}
        \centering
        \includegraphics[width=7cm]{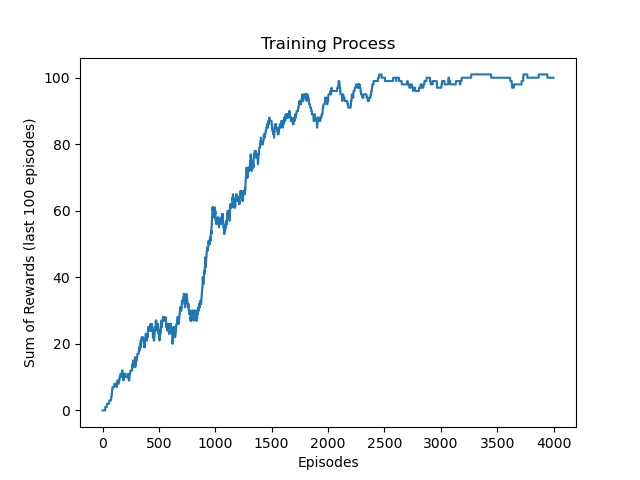}
        \caption{Training process of the instrumental policy}
        \label{fig:reas_resc_longer_path}
    \end{subfigure}
    \hfill
    \begin{subfigure}[t]{0.45\textwidth}
        \centering
        \includegraphics[width=7cm]{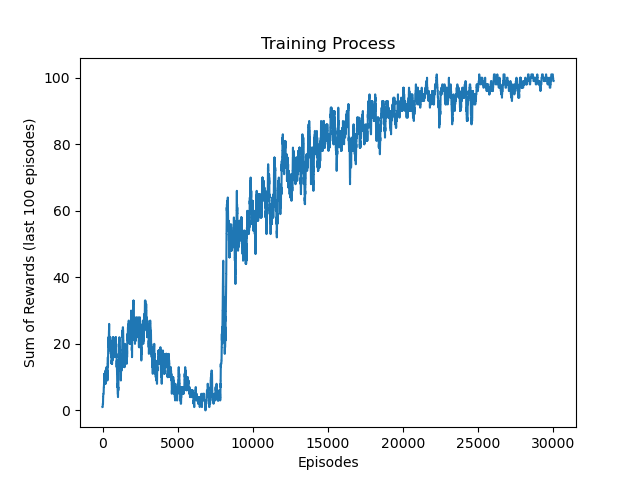}
        \caption{Training process on the ethical environment with $e_w = 2.39$.}
        \label{fig:ethicall_optimal_longer_path}
    \end{subfigure}
    \label{fig:drl_ethcically_longer_paht}
    \caption{Training progression illustrating the success in learning strategies to accomplish the instrumental goal}
\end{figure}

To illustrate a problem with the strict prioritization of moral duties over the instrumental goal, consider the behavior of the RBAMA, which successfully learned a rescuing strategy (\cref{fig:reas_resc_longer_path}) in a \textit{circular path simulation} featuring two bridges and a person (\cref{tab:env_params_longer_path}) in the need for help. The environment setup is designed to present a scenario in which the agent must choose between taking the shortest path over the upper shore to pull the person out of the water or a slightly longer route that first leads over the lower shore, allowing it to deliver the package and achieve its instrumental goal along the way (\cref{fig:ethicall_optimal_longer_path}). The environment is created such that the agent the person while being unable to get out of the water on its own is not drowning. Consequently, while taking the small detour is not ethically optimal, it likely could still be considered morally permissible -- offering the added benefit of significantly improving the agent’s instrumental performance. 

However, by design, the RBAMA learns to take the upper path -- that is, the shortest route -- to the person in the water. It learns to treat the morally relevant fact -- the fact that someone is in the water -- as action-guiding, subsequently handing control over to the rescuing network, which does not care about the RBAMA’s instrumental objective. While there is certainly room for debate about what constitutes the best overall course of action in this scenario, consistently prioritizing normative reasons over the instrumental goal effectively enforces a decision in favor of moral optimality. It prevents the RBAMA from being trained to take the longer, but potentially more efficient, path toward achieving its instrumental goal, as its behavior remains dominated by moral considerations.

The same holds for MOBMA trained on the \textit{circular path simulation} (\cref{fig:ethicall_optimal_longer_path}). By design, the procedure for determining the ethical weight for the environment guarantees that the optimal policy in the scalarized environment strictly prioritizes ethically optimal actions. As a result, the MOBMA -- like the RBAMA -- takes the upper path to the drowning person. 

Moreover, when the agent is placed next to its instrumental goal (\cref{fig:behind_goal}), the MOBMA actively avoids stepping onto the corresponding tile, instead moving around it. This behavior is likely a consequence of the training process’s termination condition: when the agent achieves its instrumental goal, the episode ends, preventing it from also fulfilling its moral task. Such learned behavior is undesirable when applied to a real-world scenario, where the agent would be expected to continuously navigate the environment -- picking up and delivering packages -- without being restricted to fulfilling its moral task only before completing its instrumental task. In contrast, the RBAMA does not learn to avoid fulfilling its instrumental goal, as the rescuing network is trained in a version of the environment that terminates only after the agent has rescued every person from the water. However, the behavior observed stems from the episodic nature of the training process, rather than from design decisions regarding the agents themselves. To avoid such unintended behavior, the training setup must ensure that the agent is not artificially discouraged from completing its instrumental goal before fulfilling its moral tasks.

\begin{figure}[h]
    \centering
    \begin{subfigure}[t]{0.45\textwidth}
        \centering
        \includegraphics[width=7cm]{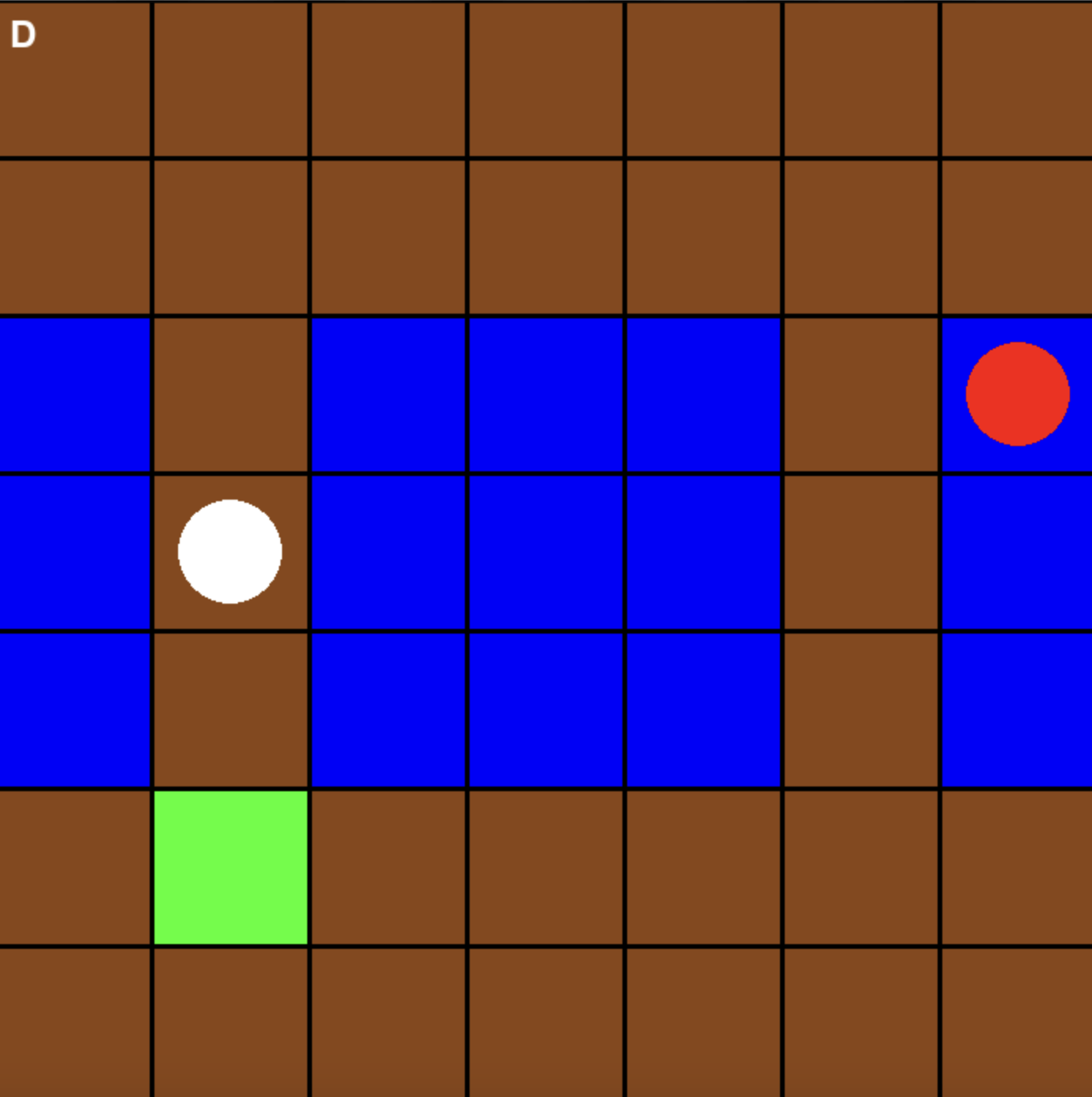}
        \caption{Choice between two paths to the drowning person}
        \label{fig:longer_path}
    \end{subfigure}
    \hfill
    \begin{subfigure}[t]{0.45\textwidth}
        \centering
        \includegraphics[width=7cm]{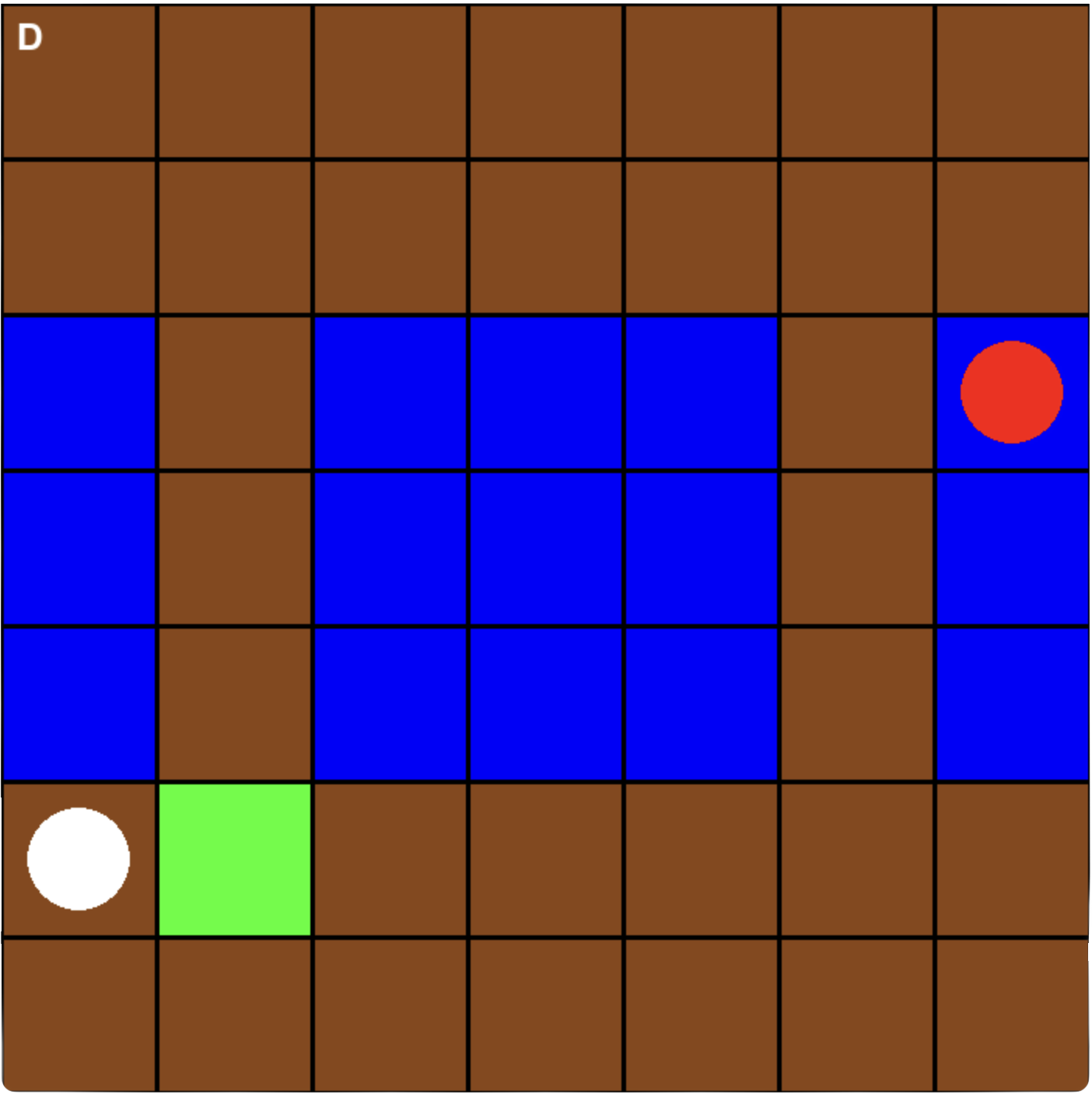}
        \caption{Agent one Step away from its instrumental goal}
        \label{fig:behind_goal}
    \end{subfigure}
    \label{fig:two_images}
     \caption{Two states in which the agent has to choose between fulfilling its instrumental task first and strictly prioritizing the fulfillment of its moral obligations}
\end{figure}

\section{Adaptability to Selective Instance Variations}\label{6Generalization}

As demonstrated in \cref{6stochEnv}, the RBAMA exhibits generalization capabilities in its moral decision-making -- specifically, by learning a prioritization of moral obligations that is not based on instance-specific dynamics such as the precise drowning time. This results from its moral decision-making being grounded in normative reasons, which abstract away such contingent features of the environment. While this highlights a significant strength of the RBAMA, another aspect of its architecture further contributes to its generalizability capabilities -- particularly with respect to its adaptability in learning how to fulfill tasks or comply with constraints in environments that require behavioral adjustments.

Central to this adaptability is the modular design of the RBAMA. It integrates separate components, each trained for a distinct function: the instrumental task of delivering the package, the moral task of rescuing persons from the water, and the detection of situations that risk violating the constraint not to push someone off the bridge. This separation of concerns enables more targeted interventions during retraining.

To empirically substantiate this architectural advantage, I conducted a series of experiments in which the RBAMA was first trained on one instance of the environment and subsequently transferred to another that required modified behavioral responses. The results demonstrate that only selective retraining of specific modules was needed for the RBAMA to effectively adapt to the new environment, confirming the practical benefit of its modular structure in supporting behavioral generalization. This advantage becomes particularly evident when contrasted with the retraining requirements of the MOBMA, which not only involve retraining the entire multi-objective policy but also necessitate rerunning the algorithm for constructing ethical environments for each individual instance.

However, before turning to the experimental results, a brief clarification is warranted. The environmental modifications introduced in the following experiments did not necessitate fundamentally different behavior for task fulfillment or constraint compliance. Adjustments -- such as relocating the location where individuals fall into the water -- could, in principle, have been accommodated by training the networks on a sufficiently diverse set of instances, particularly given that they were trained using deep reinforcement learning methods. The experiments were designed in this manner because the current test suite does not encompass training and evaluation scenarios that demand more substantial forms of generalization. Nonetheless, the experiments effectively illustrate the underlying advantage of the RBAMA’s modular architecture in supporting selective retraining. The structural separation of task-specific components and the localization of retraining procedures remain applicable even in more demanding generalization settings, thereby underscoring the scalability and flexibility of the overall design.

\subsection{The Advantage of the RBAMAs Modular Design}

In the first experimental setup, an RBAMA was first trained on a \textit{dangerous shore simulation}, a version of the environment  featuring two bridges and a dangerous spot on the lower shore, where a person strolls along and falls into the water when they step onto the perilous tile. (\cref{tab:env_params_base}). The rescue network was trained for 7,000 episodes, the bridge-guarding network for 50,000 episodes, and the instrumental network for 3,000 episodes. In addition, the RBAMA received feedback from the moral judge prioritizing the reason to rescue persons over the reason not to push persons off the bridge over 100 episodes. Evaluating the RBAMAs performance after training in the environment  yielded total rewards of $\mathbf{R}^{1000}_{\text{instr}} = 1000$, $\mathbf{R}^{1000}_{\text{push}} = -17$, and $\mathbf{R}^{1000}_{\text{resc}} = 536$ under $\textbf{count}_{\text{resc}}^{1000} = 537$ and $\textbf{count}_{\text{conflict}}^{1000} = 17$.  These results confirm the success of the training process, in particular demonstrating that the rescue network learned to consistently fulfill $\varphi_R$ -- it learned to reliably rescue the person who falls into the water at the dangerous spot (\cref{fig:old_drowning_spot}).

\begin{figure}[h]
    \centering
    \begin{subfigure}[t]{0.3\textwidth}
        \centering
        \includegraphics[width=\linewidth]{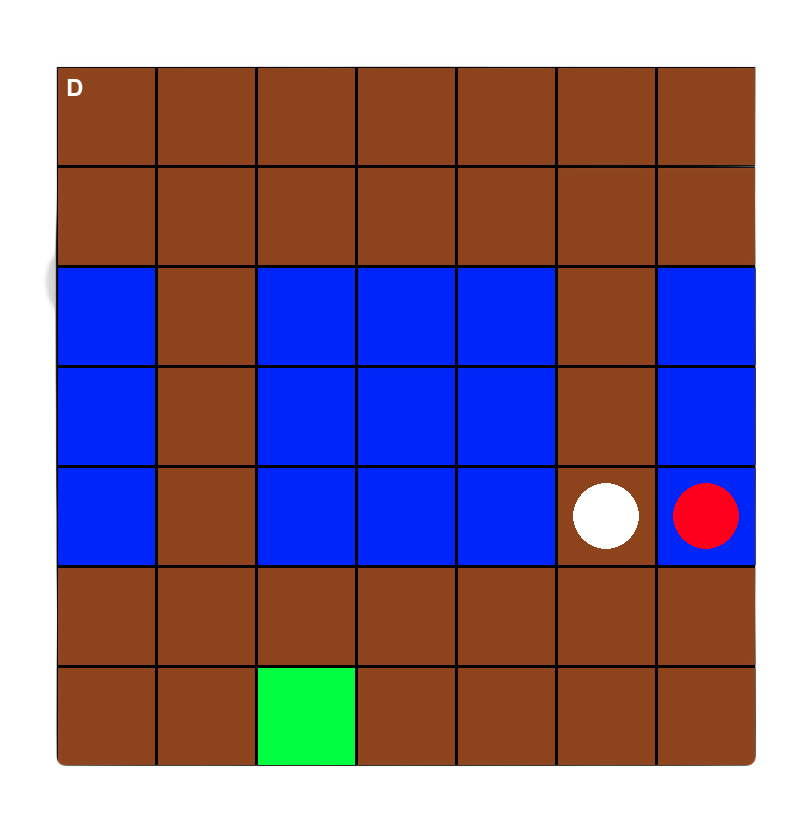}
        \caption{The RBAMA about to pull the person out of the water}
        \label{fig:old_drowning_spot}
    \end{subfigure}
    \hfill
    \begin{subfigure}[t]{0.3\textwidth}
        \centering
        \includegraphics[width=\linewidth]{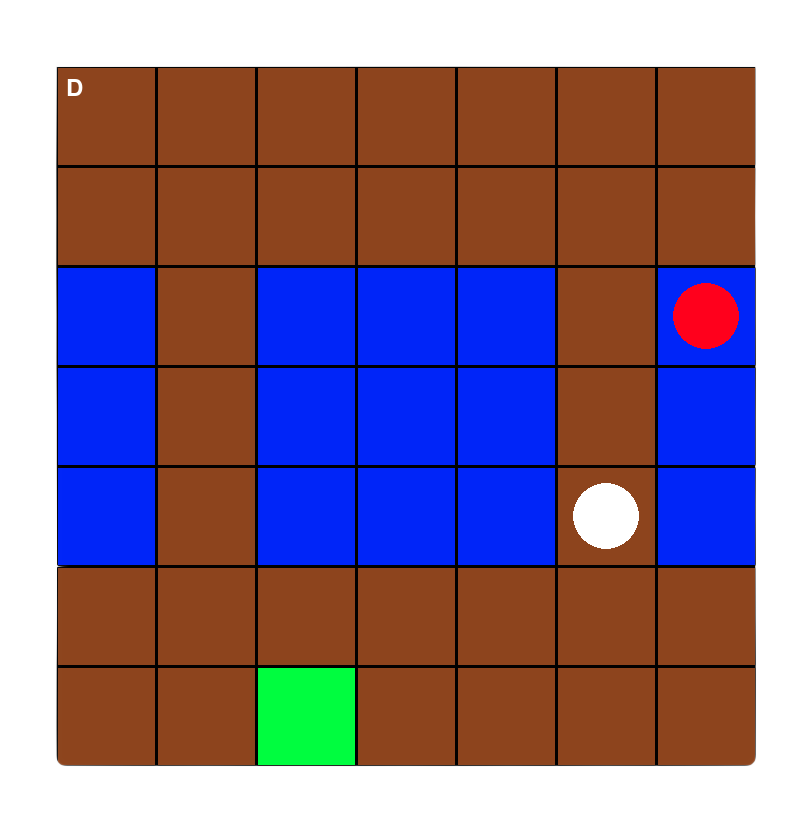}
        \caption{The RBAMA moving to the wrong position}
        \label{fig:drowning_spot_switched}
    \end{subfigure}
    \hfill
    \begin{subfigure}[t]{0.3\textwidth}
        \centering
        \includegraphics[width=\linewidth]{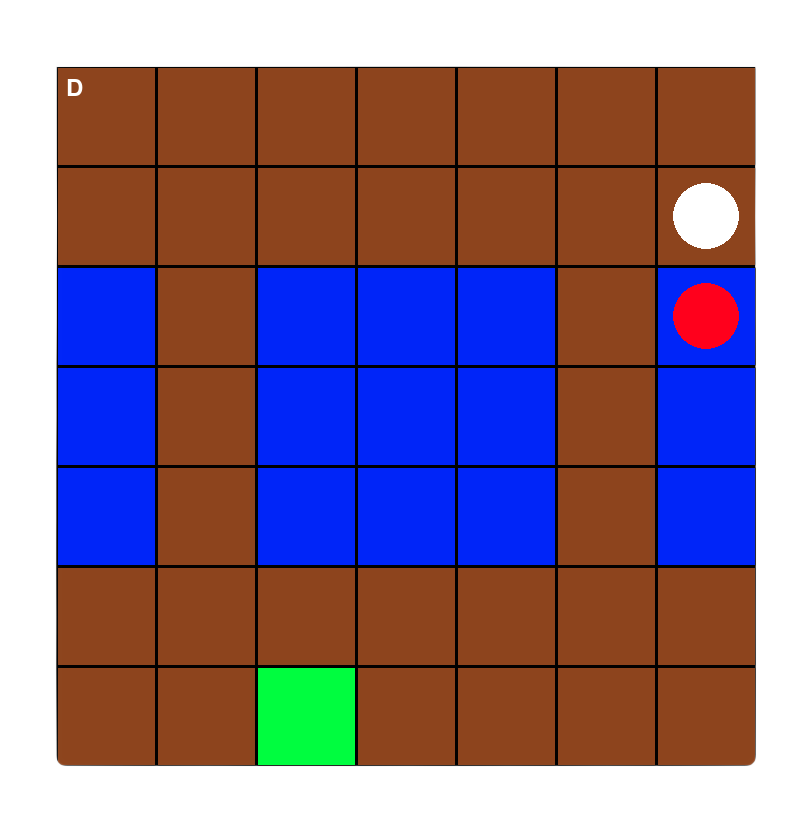}
        \caption{The RBAMA moving to the right position after replacing the rescuing net}
        \label{fig:Resc succcess}
    \end{subfigure}
    \caption{Retraining after changing the position of the drowning spot}
    \label{fig:resc_after_retraining}
\end{figure}

Subsequently, the RBAMA’s performance was tested in a slightly modified version of the \textit{dangerous shore simulation} -- the \textit{dangerous bridge simulation}(\cref{tab:env_params_ds1}) -- in which the dangerous spot is moved from the lower shore to the upper part of the right bridge. Test results of $\mathbf{R}^{1000}_{\text{instr}} = 1000$, $\mathbf{R}^{1000}_{\text{push}} = 0$, and $\mathbf{R}^{1000}_{\text{resc}} = 99$ with $\textbf{count}_{\text{resc}}^{1000} = 396$ and $\textbf{count}_{\text{conflict}}^{1000} = 0$ reveal, that the RBAMA trained on the original settings fails to reliably fulfill $\varphi_R$, which now requires it it move to the upper part of the bridge in order to rescue the person. Instead of moving toward the new drowning spot, the RBAMA continues to head toward the old one (\cref{fig:drowning_spot_switched}). Most likely, this resulted from the rescue network incorporating the path to the initial drowning spot into its rescue strategy. However, after training a new rescue network on the \textit{dangerous bridge simulation} for 7,000 episodes and integrating it into the RBAMA, the system demonstrated strong performance in the new environment. It successfully fulfilled $\varphi_R$ while adhering to $\varphi_C$, and consistently achieved its instrumental goal, reaching total rewards of $\mathbf{R}^{1000}_{\text{instr}} = 1000$, $\mathbf{R}^{1000}_{\text{push}} = 0$, and $\mathbf{R}^{1000}_{\text{resc}} = 417$ with $\textbf{count}_{\text{resc}}^{1000} = 396$ and $\textbf{count}_{\text{conflict}}^{1000} = 0$ -- a result of the newly trained rescue network guiding the RBAMA to the new drowning spot (\cref{fig:resc_after_retraining}). Notably, no adjustments were necessary for either the instrumental network or the bridge-guarding network. Since the position of the instrumental goal remained unchanged, the existing strategy of the instrumental policy continued to function effectively. The instrumental network -- being trained on random resets in the \textit{dangerous bridge simulation} -- had already learned to navigate to the goal from any starting point -- including all possible positions where control is handed back to the instrumental policy after the rescuing network leads the RBAMA to the fulfillment of $\varphi_R$. Furthermore, as only the right bridge is crossed by a person in both the \textit{dangerous shore simulation} and the \textit{dangerous bridge simulation}, the bridge-guarding network consistently correctly assessed the risk of violating $\varphi_C$.

This seamless process of replacing the networks responsible for handling specific changes between versions of the bridge environment becomes more challenging when training an RBAMA on the \textit{blocked left bridge simulation} (\cref{tab:env_left_bridge_blocked}) and then attempting to deploy it on the \textit{blocked right bridge simulation} -- a modification in which the person is moved to the right bridge. 

It has already been shown that for training an instrumental network in the \textit{blocked left bridge simulation}, the RBAMA needs to be shielded by the bridge-guarding network already during the training process (cf. \cref{6training_nets_to_work_together}). Specifically, the bridge-guarding network consistently identifies a high risk associated with actions leading onto the left bridge, where a person is permanently positioned. This prevents the instrumental network from attempting to cross the left bridge. With this shielding in place, the instrumental network learns that it is permanently restricted from entering the left bridge and, as a result, adopts a policy that guides it across the right bridge instead.  Consequently, the instrumental network learns an optimal policy that reliably selects the right bridge as its path toward the goal, as the left bridge is rendered permanently inaccessible by the bridge guarding network. When deploying the RBAMA comprising the bridge-guarding network and the instrumental network trained on the \textit{blocked right bridge simulation} within the \textit{blocked right bridge simulation}, the bridge-guarding network fails to assign a high risk to entering the right bridge. This is because it was not trained on the new configuration, where a person now stands permanently on the right bridge. As a result, the instrumental network executes its previously learned policy without accounting for the updated position of the person leading to the RBAMA crossing the right bridge and thereby pushing the person off. This undesirable behavior is also reflected in degraded performance: the RBAMA achieves total rewards of $\mathbf{R}^{1000}_{\text{instr}} = 1000$ and $\mathbf{R}^{1000}_{\text{push}} = -438$ with $\textbf{count}_{\text{conflict}}^{1000} = 0$ showing that it violates $\varphi_C$ without detecting a conflict. 

To prevent this behavior, the bridge-guarding network needs to be retrained so it stops the RBAMA from entering the right bridge while allowing it to cross the left one. However, merely retraining the bridge-guarding network and deploying it within the RBAMA results in the RBAMA failing to reach its delivery location. While the bridge-guarding network prevents the instrumental network from entering the right bridge, it still attempts to take that route. With its preferred action being blocked, the RBAMA ultimately ends up wandering aimlessly into the water (\cref{fig:agent_in_water}) when its position is initially set to be on the upper shore. This is again reflected in the performance of the RBAMA. The total rewards of $\mathbf{R}^{1000}_{\text{instr}} = 586$ and $\mathbf{R}^{1000}_{\text{push}} = 0$ with $\textbf{count}_{\text{conflict}}^{1000} = 0$ show that the RBAMA -- while never pushing persons off the bridge -- only succeeds in reaching the delivery location in half the test runs. It is only after additionally retraining the instrumental network while being shielded by the new bridge-guarding network that the RBAMA learns to take the left bridge and successfully reaches its goal  (\cref{fig:left_bridge_after_retraining}), which is also reflected in the test episodes yielding $\mathbf{R}^{1000}_{\text{instr}} = 1000$ and $\mathbf{R}^{1000}_{\text{push}} = 0$ with $\textbf{count}_{\text{conflict}}^{1000} = 0$. The experiment demonstrates that, although not every change in the environment’s configuration requires retraining all networks, the interactions between the individual networks must nonetheless be considered when deciding on selective retraining.

The RBAMA's modular architecture, which integrates separate reasoning and behavior networks, also holds the distinct advantage of enabling selective retraining when it learns new kinds of normative reasons during deployment. For instance, consider an RBAMA initially trained within the \textit{dangerous shore simulation} without receiving feedback such that it incorporates $\delta_2$ -- the default rule under which $\varphi_R$ is derived -- in its reason theory. If the RBAMA subsequently learns $\delta_2$ through feedback, the required update is similar to repositioning a dangerous spot within the environment: only a rescuing network needs to be trained. There is no necessity to retrain the other networks, as the modular design ensures that learning to comply with $\varphi_R$ does neither influence how $\varphi_C$ can be properly enforced nor does it affect the effectiveness of the acquired strategy for reaching the delivery location. However, analogous to the scenario where a person is repositioned on the bridge, if the RBAMA learns $\delta_1$ -- the default rule under which $\varphi_C$ is derived -- after the initial training phase, it would require not only training a bridge-guarding network responsible for enforcing conformance with $\varphi_C$, but also retraining its instrumental network to ensure that it can find a route to the delivery location that remains compliant with the newly imposed moral shield. Importantly, determining which networks require retraining in response to an update of the RBAMA's reasoning -- or in response to particular changes in the environment -- may prove to be a non-trivial challenge.

\begin{figure}[h]
    \centering
    \begin{subfigure}[t]{0.3\textwidth}
        \centering
        \includegraphics[width=\linewidth]{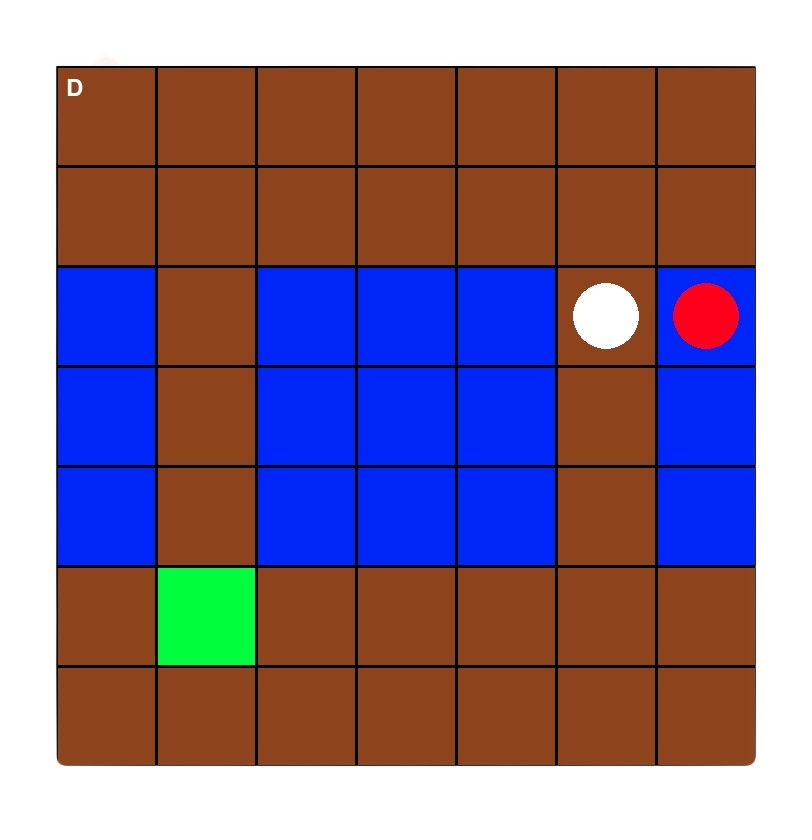}
        \caption{Agent Pushes Person Into Water}
        \label{fig:agent_pushes_person_new_spot}
    \end{subfigure}
    \hfill
    \begin{subfigure}[t]{0.3\textwidth}
        \centering
        \includegraphics[width=\linewidth]{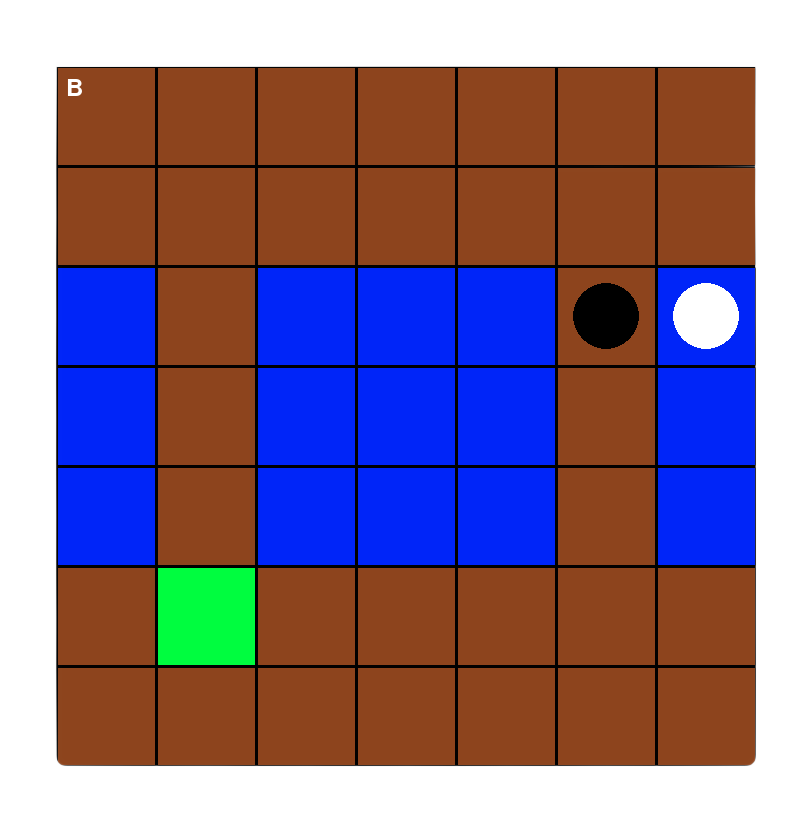}
        \caption{Instrumental network trained on old shield}
        \label{fig:agent_in_water}
    \end{subfigure}
    \hfill
    \begin{subfigure}[t]{0.3\textwidth}
        \centering
        \includegraphics[width=\linewidth]{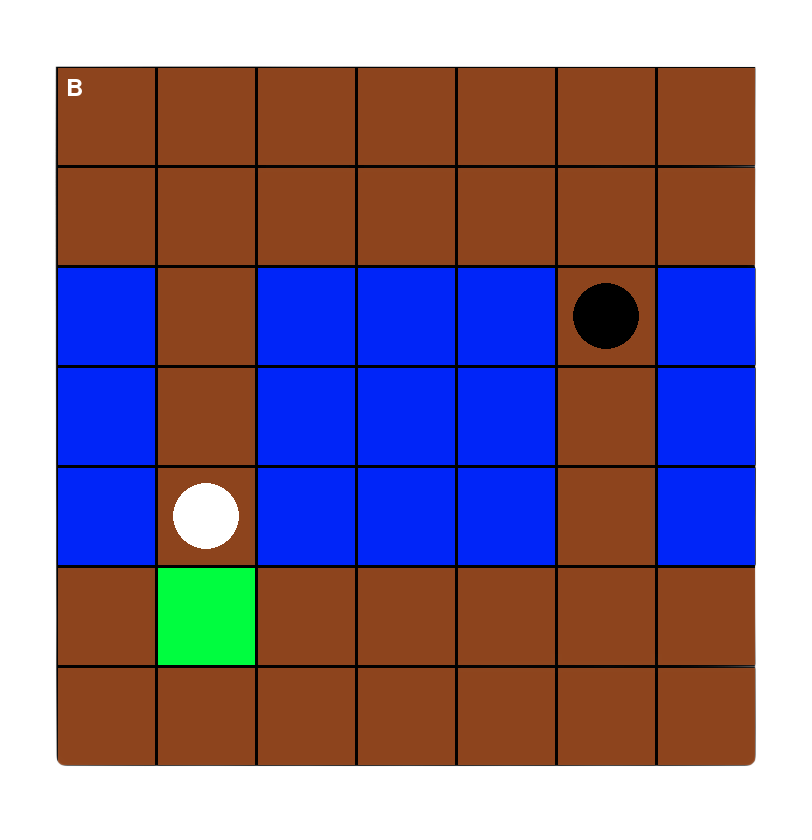}
        \caption{Instrumental network trained on new shield}
        \label{fig:Resc succcess}
    \end{subfigure}
    \caption{Retraining after moving the static person from the left to the right bridge}
    \label{fig:left_bridge_after_retraining}
\end{figure}

\subsection{From Scratch Again}

Turning towards the generalization capabilities of a MOBMA, two key factors limit its flexibility: i) Any changes in the details of the environment can alter the ethical weight, requiring the entire process for calculating these weights to be repeated if the MOBMA needs to be trained on another instance; and ii) The MOBMA relies on a single neural network learning a multi-objective policy, which means there’s no possibility of retraining individual components independently.

The algorithm for generating ethical environments returned an ethical weight of 1.95 for the \textit{dangerous shore simulation}. Training a MOBMA on the corresponding scalarized environment for 30,000 episodes and subsequently evaluating its performance over 1,000 test episodes resulted in total returns of $\mathbf{R}^{1000}_{\text{instr}} = 1000$, $\mathbf{R}^{1000}_{\text{push}} = 0$, and $\mathbf{R}^{1000}_{\text{resc}} = 877$ with $\textbf{count}_{\text{resc}}^{1000} = 877$. These results indicate that the MOBMA successfully learned a policy approximating the ethically optimal policy within this environment. However, when evaluating the MOBMA on the \textit{dangerous bridge simulation}, its performance degraded significantly, yielding $\mathbf{R}^{1000}_{\text{instr}} = 1000$, $\mathbf{R}^{1000}_{\text{push}} = -7$, and $\mathbf{R}^{1000}_{\text{resc}} = 61$ with $\textbf{count}_{\text{resc}}^{1000} = 791$. This shows that the MOBMA fails to fulfill its moral objective of rescuing the drowning person under the modified environment dynamics. To ensure ethical optimality in the \textit{dangerous bridge simulation}, a recalculation of the ethical weight was necessary. Rerunning the iterative process for ethical weight computation produced an updated weight of 2.2. Training a new MOBMA on a scalarized environment using this revised weight led to total returns of $\mathbf{R}^{1000}_{\text{instr}} = 1000$, $\mathbf{R}^{1000}_{\text{push}} = 0$, and $\mathbf{R}^{1000}_{\text{resc}} = 817$ with $\textbf{count}_{\text{resc}}^{1000} = 817$, confirming the training to be successful. Crucially, neither the ethical weight calculated for the original environment can be reused to train an MOBMA on the modified environment, nor can the MOBMA trained on the initially created ethical environment be directly transferred. In contrast, applying the RBAMA to the updated version of the environment required retraining only the rescuing network, due to its modular architecture.

The same procedure must be applied when transitioning from the \textit{blocked left bridge simulation} to the \textit{blocked right bridge simulation}. The ethical weight computed for the \textit{blocked left bridge simulation}, when considering the full state space, is 1.95. Training a MOBMA on the scalarized environment corresponding to this weight for 30,000 episodes results in total returns of $\mathbf{R}^{1000}_{\text{instr}} = 1000$, $\mathbf{R}^{1000}_{\text{push}} = 0$, indicating that the MOBMA successfully learned an ethically optimal policy in this specific configuration. However, when deploying the MOBMA in the \textit{blocked right bridge simulation}, the MOBMA’s performance deteriorates significantly. It pushed persons off the bridge, resulting in total returns of $\mathbf{R}^{1000}_{\text{instr}} = 477$, $\mathbf{R}^{1000}_{\text{push}} = -195$. This outcome reflects not only frequent violations of the moral constraint against pushing person off a bride but also a failure with respect to reliably achieving its instrumental goal. This failure likely can be attributed to the MOBMA's reliance on a single policy -- trained to balance multiple objectives simultaneously (i.e., the moral duties of not pushing and rescuing, alongside the instrumental task of reaching the goal). The joint optimization over moral constraints and instrumental rewards likely leads to a strategy highly sensitive to the specific state distributions encountered during training. To ultimately obtain an ethically optimal policy for the \textit{blocked right bridge simulation}, the ethical weight must be recalculated. Running the ethical weight computation process on the updated environment yields a value of 2.49. Training an MOBMA on the corresponding scalarized environment for 30,000 episodes then again results in strong performance of $\mathbf{R}^{1000}_{\text{instr}} = 1000$, $\mathbf{R}^{1000}_{\text{push}} = 0$. 

\section{Scaling Up: Testing on A Larger State Space}\label{6ScalingUp}

The tests conducted so far were performed on small versions of the bridge environment setting $W=7$ and $H=7$ and featuring a relatively sparse distribution of persons on the map. To evaluate the performance of an RBAMA in comparison to that of a MOBMA in environments with larger state spaces, I trained and tested the models on a map of size $W = 9$ and $H = 9$, while further increasing complexity through a denser configuration by including three bridges and four persons moving across the map (\cref{tab:env_large_map}).

\begin{figure}[h]
    \centering
    \begin{subfigure}[t]{0.45\textwidth}
        \centering
        \includegraphics[width=\linewidth]{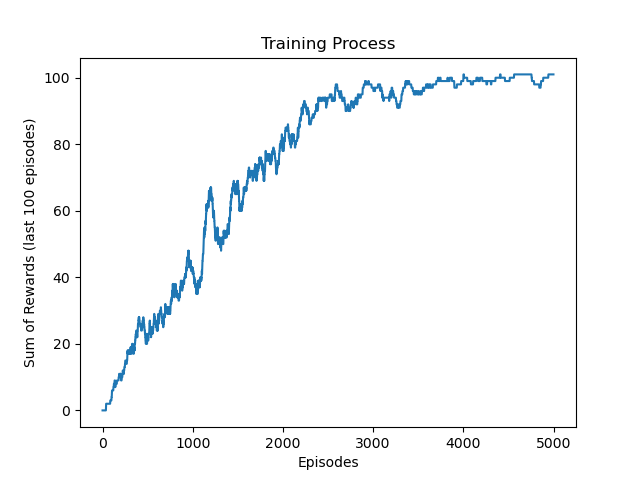}
        \caption{Training process of the instrumental network}
        \label{fig:training_instr_policy_large}
    \end{subfigure}
    \hfill
    \begin{subfigure}[t]{0.45\textwidth}
        \centering
        \includegraphics[width=\linewidth]{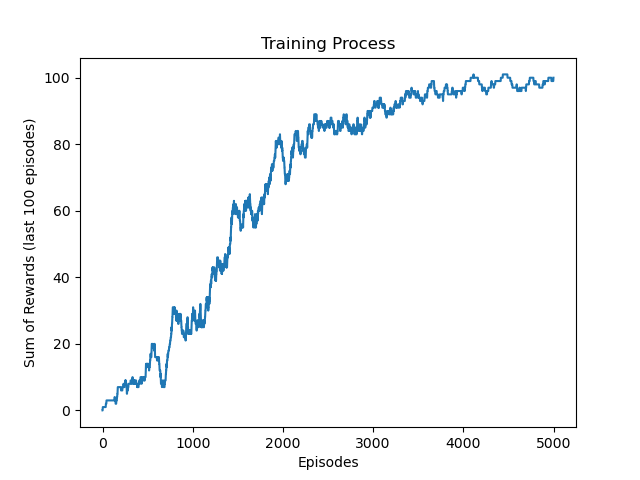}
        \caption{Training process of the rescuing network}
        \label{fig:drowning_spot_switched}
    \end{subfigure}
    \caption{Training the RBAMA on a more complex scenario}
    \label{fig:training_large_map}
\end{figure}

In the RBAMA, the rescue network was trained for 5,000 episodes, the bridge-guarding network for 100,000 episodes, and the instrumental network for 5,000 episodes, as shown in \cref{fig:training_large_map}. Once training was completed, the RBAMA was deployed to navigate the environment over the course of 100 episodes, during which it received evaluative feedback from the moral judge. Across the subsequent test episodes, a total of 602 persons fell into the water ($\textbf{count}_{\text{resc}}^{1000} = 577$), and the RBAMA detected conflicts between $\varphi_C$ and $\varphi_R$ on 40 occasions ($\textbf{count}_{\text{conflict}}^{1000} = 40$). Cumulative rewards of $\mathbf{R}^{1000}_{\text{instr}} = 1000$, $\mathbf{R}^{1000}_{\text{push}} = -40$, and $\mathbf{R}^{1000}_{\text{resc}} = 568$ confirm the RBAMA’s effective performance both in adhering to its moral obligations and in achieving its instrumental objectives within the more complex environment configuration. By contrast, when applying the algorithm proposed by Rodriguez-Soto et al. to generate an ethical environment from an expanded version of the bridge environment, the process of calculating the ethical weight exceeded a predefined timeout threshold of two hours. This highlights a critical limitation of the approach, as it already fails to scale to environments with only moderately increased state space size.

    \chapter{Discussion}\label{7Discussion}

\section{Philosophical Reflections on Motivation and Results}\label{7motivationRevisited}

In the experiments whose results are presented in \cref{6results}, an RBAMA was trained and evaluated in several instances of an environment inspired by the real-world scenario introduced in \cref{2challenges}, in which the agent faces a moral conflict between two duties: rescuing a person who has fallen into the water and is at risk of drowning, and ensuring that it does not push another person off the bridge. Specifically, the agent was trained on an instance with fully deterministic transition dynamics and tested both on this instance and on two instances, in which randomness was introduced to the drowning time. The agent received feedback from a moral judge, through which it integrated two default rules: a default rule $\delta_1$ that connects the morally relevant fact that there is a person on the bridge to a moral obligation $\varphi_C$ not to push them into the water, and a default rule $\delta_2$ connecting the fact that there is a person in the water moral obligation $\varphi_R$ to rescue them as quickly as possible. Moreover, the moral judge-- assumed to lack knowledge about the specific drowning time and the probability for pushing persons off the bridge -- taught the agent to strictly prioritize rescuing the drowning person,i.e., to learn the ordering $\delta_2 > \delta_1$. During testing, the RBAMA demonstrated that it had learned to reason in accordance with the moral judge and to act in conformity with the overall moral obligations it derived. The following section revisits the desiderata of trustworthiness, moral robustness, and moral justifiability as introduced in \cref{2reasonsPhilo}, examining how the experimental results underscore the inherent advantage of the RBAMA in  serving these criteria.

\subsection{Demonstrated Trustworthiness}

The RBAMA, as trained and deployed in \cref{6results}, represents a successful implementation of the original vision of developing an artificial moral agent whose decision-making is grounded in normative reasons. Moreover, it was demonstrated that the RBAMA also successfully \textit{internalized} the normative reasoning of the \textit{moral judge}, which was implemented as a dedicated module to teach the agent moral directives for its behavior. Assuming that the judge’s reasoning constitutes a \textit{sound normative reasoning}, the RBAMA thereby acquired \textit{reliable competence} in moral decision-making. In addition, the RBAMA consistently acted in accordance with the moral obligations it derived, thereby demonstrating not only moral competence but also an enforced \textit{reliable willingness} to follow the moral obligations it recognized as overall binding.  

Under a plausible account of trustworthiness -- adapted to artificial agents as outlined in \cref{2trustworthiness} -- this combination of reliable competence and reliable willingness provides a sufficient basis for considering the RBAMA \textit{morally trustworthy}. Moreover, the visualization of the RBAMA’s reasoning process allows for direct \textit{trustworthiness assessment} \cite{schlicker2025we}, providing explicit evidence of its competence and further supporting the case for \textit{fostering justified trust} in the agent. 

\subsection{Moral Robustness Against Morally Unsubstantiated Behavioral Shifts}\label{7moralRobustness}

In \cref{2robustness}, moral robustness was defined with reference to the established technical term of robustness as the guarantee that morally significant changes in the agent's behavior are proportional to morally significant changes in the circumstances. Achieving this property is particularly challenging in other approaches, such as conventional MORL frameworks, where discontinuities in the agent’s learned behavior, are inherent to the method. Crucially, the point at which a behavioral shift occurs is indirectly determined by the design of the reward function and the dynamics of the environment. 

For example, consider the experiments presented in \cref{6stochEnv}. While the RBAMA is guaranteed to prioritize rescuing drowning persons over ensuring that it does not push anyone off the bridge, an agent trained using conventional MORL methodology—i.e., learning which course of action to take based on expected returns as determined by the reward function—\textit{could} unexpectedly adopt different behavior when the probability of drowning, $P_{\text{drown}}$, shifts from $0.1$ to $0.11$, even though such a change likely lacks moral significance. If the design of the reward function is not carefully considered, this shift in behavior might merely be an unintended side effect. However, even if one attempts to implicitly set a threshold for a behavioral shift at a specific value, directly designing the reward structure to establish such a precise point of reprioritization remains extremely difficult, if not impossible, as discussed in \cref{7rewardFunction}. While reinforcement learning methodologies such as inverse reinforcement learning (cf. \cite{ng2000algorithms, abbeel2004apprenticeship, ziebart2008maximum}) and reinforcement learning from human feedback (cf. \cite{knox2008tamer, christiano2017deep}) offer ways to infer reward functions rather than imposing them top-down, and might thus help mitigate the problem, the threshold remains a critical parameter with significant influence -- in particular also to the agent’s moral robustness -- and arguably ought to be subject to deliberate and explicit control, which is further discussed in \cref{7relevanceConsequence}. 

In this regard, the reason-based approach offers a distinct advantage: it ensures that no abrupt behavioral shift occurs due to thresholds that are morally irrelevant, since expected outcomes do not factor into the RBAMA’s reasoning. However, the insensitivity of the RBAMA’s moral decision-making to estimated action consequences can also be seen as a significant limitation, as further discussed in \cref{7probs}.

\subsection{On the Moral Justifiability of the RBAMAs Actions}

The RBAMA is, by design, guided in its moral decision-making by normative reasons. As such, its architecture provides a foundation for rendering its actions morally justifiable, based on the plausible claim that \textit{if an agent’s action is supported by normative reasons, then the action is morally justifiable} (see \cref{2justifiability}). As the RBAMA was trained by a moral judge to adopt the default rules $\delta_1$ and $\delta_2$, along with the prioritization $\delta_2 > \delta_1$, it bases its moral decision-making, and consequently its behavior, on a reason theory composed of these rules and their prioritization, which arguably give ground to \textit{sound normative reasoning}. As a result, its actions can be seen as supported by normative reasons and are therefore \textit{morally justifiable}. This is particularly significant in scenarios involving conflicting moral obligations, where the RBAMA consistently prioritizes the duty to rescue based on this sound normative reasoning, thereby ensuring that its actions remain morally justifiable even under such critical circumstances.

\section{The Reason-Based Framework Under Scrutiny}
It has been argued that the RBAMA, as trained and evaluated in the discussed instances of the bridge simulation, exhibits trustworthiness, moral robustness, and moral justifiability in its actions. Crucially, however, these claims rest on two underlying assumptions: first, that there is a clear understanding of what constitutes sound normative reasoning; and second, that the reduction of moral situations to propositional representations fully captures all morally relevant aspects. The following section examines these foundational assumptions and discusses what would be required either to substantiate them or to revise the framework in a way that renders them dispensable.

\subsection{What Makes Normative Reasoning Sound?} 
When claiming that the agent exhibits moral competence in its decision-making it is assumed that the agent has acquired \textit{sound normative reasoning}.  Crucially, conclusively determining what constitutes sound normative reasoning would presuppose access to an ethical ground truth. However, this does not undermine the reason-based approach toward building AMAs. Rather than aiming to ensure that RBAMAs arrives at morally correct decisions in an absolute sense, the aim is to teach them normative reasoning which guarantees that their actions satisfy ethical desiderata such as moral competence, trustworthiness, and justifiability. The RBAMA is designed to fulfill these desiderata without relying on any particular ethical standpoint presumed to reflect an ethical ground truth. More concretely, it is equipped with the capability to \textit{learn} meeting them -- a learning process guided by a moral judge as a moral authority upon whom the quality of the agent’s reasoning ultimately depends. 

Currently, the moral judge is implemented as a module that provides rule-based feedback. However, setting such rules at the design phase, stands in tension with the foundational idea that constructing an RBAMA should not require access to an ethical ground truth. Looking ahead, the moral judge is envisioned as a human -- or a group of humans -- who make case-based assessments about what constitutes the morally right action in a given situation, and which reasons support that decision. However, not just any individual or group should have the authority to shape an RBAMA’s reasoning according to their own moral convictions, thereby imposing their particular view of what counts as sound normative reason on others. Crucially, relying on public opinion -- shaped by vague intuitions or prevailing sentiments -- as the basis for feedback raises the further concern that the system may derive its legitimacy solely from popular approval. This, in turn, risks entailing an implicit commitment to ethical relativism as its metaethical foundation -- a commitment that could remain unacknowledged and unexamined. 

The general question of the assignment of moral authority for determining the morality upheld by an AMA, and under what procedures such decisions should be made, has already been addressed in the literature (cf. \cite{gabriel2020a, gabriel2025matter}).  Applications of social choice theory, for instance, frame the alignment problem in general as one of aggregating preferences or reaching agreement through voting and public deliberation \cite{gabriel2020a, baum2020social, conitzer2024social}. This body of work also provides a foundation for deciding who should be appointed as a moral judge. In addition to these considerations, one concrete way to mitigate the risk of teaching RBAMAs reasoning that does not qualify as sound normative reasoning is to involve ethicists as experts in the feedback process. Their role could help ensure that the agent’s reasoning is anchored in well-considered normative standards. This approach could also be complemented by not starting with an entirely empty framework. Instead, the agent could be pre-equipped with a set of initial reasons and an initial order --  possibly derived from principles supported by a wide scope of ethical theories -- that provide a rough normative orientation. This reason-theory could then serve as a basis for further refinement and expansion through the feedback process.

These reflections while offering some possible direction do not aim to offer a complete solution, but instead primarily serve to highlight the depth of the problem posed by the lack of clarity regarding what constitutes a good normative reason. Nonetheless, in order to substantiate any claims about the moral justifiability of the RBAMA’s actions, it remains essential to address the question of what qualifies as a good reason and who should have authority as moral judge.

\subsection{Neglecting Consequences and Their Probabilities}\label{7relevanceConsequence}

In the current implementation, the RBAMA is only able to process normative reasons whose premises consist of morally relevant propositions; that is, probabilities concerning action outcomes do not factor into its moral decision-making. For instance, when confronted with the moral dilemma, which was discussed throughout this work, the reason for pulling a person out of the water is consist in the morally relevant fact of there is a person in the water, which the agent is taught to strictly prioritize over its conflicting reason to not push persons off the bridge. Considerations about the potential consequences of acting immediately versus waiting a time step are not part of the reasoning process. As previously discussed, this restriction offers an advantage in terms of moral robustness: the RBAMA's course of action is not determined through a threshold that is implicitly set by the design of the reward-function. Instead, if it was trained to adhere to the reason theory provided by a moral judge, it reliably follows the established priority of reasons -- that is, it consistently prioritizes $\delta_R$ over $\delta_C$. 

However, arguably, the RBAMA's should, in principle, be capable of taking estimated action consequences into account. Consider, for example, the test runs conducted on the stochastic instance of the bridge environment, in which the probability of a person drowning after each time step was set to $P_{\text{drown}} = 0.3$. Based to the corresponding return values, which show, that the RBAMA fails to successfully rescue the drowning person in the majority of the test runs, one could argue -- assuming that the probability distribution reflects real-world stochastic properties rather than merely instance-specific dynamics of the training environment -- that such a shift in the probability distribution should influence how the RBAMA prioritizes among its moral obligations. 
Thus, while implicitly introducing a threshold for prioritizing between $\varphi_C$ and $\varphi_R$ is highly problematic -- since such a thresholds is unlikely to be morally grounded -- one might nonetheless maintain that, in principle, a morally significant threshold could exist.

Insisting on a strict prioritization among moral obligations, may invite criticism akin to that leveled against deontological ethics under the charge of fanaticism -- namely, that strict adherence to deontic rules can in the worst case lead to moral catastrophes \cite{sep-ethics-deontological}.\footnote{This can be illustrated by a well-known example: Suppose a terrorist has planted a bomb in a crowded city, set to detonate soon and kill thousands. The only way to discover its location is to torture the captured suspect. According to strict deontological ethics (e.g., Kantianism), torture is categorically forbidden, regardless of the consequences. Even if torturing the suspect would clearly prevent mass casualties, the deontologist maintains that violating the moral rule is impermissible. Fanaticism, in this context, refers to the refusal to violate a moral rule even when doing so would prevent a moral catastrophe.} As the motivation for developing an RBAMA was grounded in the notion that reasons are not inherently tied to any particular ethical theory, thereby enabling the construction of an open moral framework, the framework should in principle allow for reasoning based on consequentialist intuitions -- i.e., reasoning, in which expected outcomes of actions play a role. After all, whether presumed action consequences should indeed count as morally relevant ought to remain a decision left to the moral judge as a legitimate authority. However, the current implementation of the RBAMA does -- at least not straightforwardly -- support the integration of such considerations. Possible directions for addressing this limitation are discussed in \cref{7probs}.

\section{Generalization Capabilities in Identifying the Right Course of Action: A Conceptual Advantage of Reason-Based Artificial Moral Agents}\label{7behaviorComparison}

The experiments conducted in the deterministic moral dilemma simulation revealed a striking behavioral difference between the RBAMA and the MOBMA in how they prioritized their moral obligations within the dilemma state that links to the a priori comparison of the paradigms under which they operate as discussed in \cref{2introInformalRewardBased} and \cref{2reasonsPhilo}. This section explores how the RBAMA’s reason-based moral decision-making yields better generalizability capabilities by laying out the conceptual basis for better avoidance of overfitting -- a risk present in the MOBMA approach.

\subsection{Learning the Wrong Lesson: Instance-Specific Overfitting in Moral Prioritization}\label{7overfittingMoralPrio}

When confronted with the conflict between its moral obligations, the RBAMA immediately moved to rescue the drowning person, whereas the MOBMA learned to wait until the bridge was clear before approaching the person in need. The RBAMA’s decision-making was grounded in the moral duties it had learned to recognize as overall binding, as specified by its reason theory. It was explicitly taught which reasons to endorse and followed the prioritization of obligations it had acquired -- specifically, to always prioritize the duty to rescue over the duty to avoid pushing someone off the bridge, an ordering that can be regarded as plausible, given the assumption that the drowning time is unknown to the moral judge in its role of evaluating the RBAMA's behavior. In contrast, the MOBMA developed a multi-objective policy in which the prioritization of competing moral duties was guided solely by expected outcomes. Consequently, its decision -- whether to wait on the bridge or immediately move toward the drowning person -- was driven by the expected return, as shaped by the structure of the reward function and the dynamics of the environment. Specifically, the environment’s \textit{instance-specific} dynamics were configured such that, in the moral dilemma state, the agent had sufficient time to wait for the person on the bridge to move out of the way and still reach the drowning person in time. Crucially, the MOBMA learned to integrate these \textit{instance-specific} dynamics into its strategy for maximizing expected return -- effectively uncovering and responding to a hidden state variable: the drowning time. This likely caused its multi-objective policy to prioritize the moral obligation to avoid pushing someone off the bridge over the obligation to rescue the person from the water, which then guided the MOBMA's overall behavior. 

How should the difference in the prioritization of moral obligations be evaluated? When considering only the return during the test phase in this specific instance, the MOBMA appears to perform better: it successfully rescued every person from drowning while never pushing anyone off the bridge. In contrast, although the RBAMA also managed to rescue every person, it consistently did so by pushing the person on the bridge into the water as a side effect. As a result, when focusing solely on overall behavior as measured by return in this particular environment configuration, one could argue that an agent following a multi-objective policy is capable of learning an overall morally preferable behavior -- precisely because it uses the specific environment dynamics to resolve the dilemma situation.

However, the overarching motivation behind the RBAMA was to develop an approach for building AMAs suitable for real-world deployment. In light of this motivation, one would not want an agent to exhibit good generalization capabilities. In particular one would want the agent not to learn and rely on environment-specific dynamics as allowing it to base the prioritization of moral obligations on such features would constitute a form of overfitting to an oversimplified model. In this respect, the RBAMA holds a clear advantage. The moral judge, which decides about the prioritization of the moral obligations, is assumed to have no knowledge about these environmental specific dynamics, such that the grounds on which this prioritization is taught to the RBAMA does do not take them into account. Consequently, learning the prioritization through the case-based feedback process guarantees to prevent such cases of overfitting. Put differently, the moral judge has direct \textit{control} over the prioritization of the RBAMA’s moral obligations, thereby ensuring that instance-specific dynamics are systematically abstracted away. This advantage of the RBAMA has been experimentally demonstrated by deploying it without any retraining of the networks or adaption of its reason theory in the two \textit{stochastic moral dilemma simulations} (see \cref{6stochEnv}), where it exhibited the behavior, it was intended to. 

Of course, there are in principle alternative strategies for preventing overfitting -- for example, training a multi-objective policy on environments that vary in drowning time, such as the stochastic moral dilemma simulation. However, the approach proposed by Rodríguez-Soto is not applicable to such environments. Employing a different MORL method to train a multi-objective policy under stochastic conditions in the bridge environment for comparison with the RBAMA thus represents a promising direction for future research. Independent of possible future empirical evaluation, however, it can be stated that the RBAMA holds a principled advantage: by relying on a reasoning process decoupled from environment-specific dynamics, it structurally avoids overfitting in terms of its prioritization of moral obligations, while this is not the case for MORL approaches in general.  

\subsection{Misplaced Anticipation: Overfitting the Transition from Instrumental to Moral Action}

The second aspect in which the the RBAMA and the MOBMA substantially differ in terms of their behavior is based on  the MOBMA’s anticipation that the person walking along the lower shore would fall into the water at the designated drowning spot. It proactively moved toward that location and waited for the person to fall in. In contrast, the RBAMA continued pursuing its instrumental goal until the person had already fallen in, and only then redirected its attention. This difference can again be attributed to the MOBMA’s multi-objective policy, which enabled it to learn a single strategy to satisfy all of its goals simultaneously in combination with the MOBMA having once again internalized an instance-specific dynamic of the environment -- namely, that the person would inevitably fall into the water at that spot. These two factors, taken together, likely incentivized its proactive behavior. In contrast, the RBAMA focuses solely on what it identifies as binding moral obligations based on its reasoning theory, and acts according to the outputs of the networks responsible for ensuring its behavior aligns with those obligations. Since it only derives the moral obligation to rescue persons from the water once the morally relevant proposition that a person is in the water holds true, it activates its rescue network only at that point to guide its behavior. 

One might again argue that, in this particular configuration of the environment, the MOBMA learns a morally preferable behavior by preemptively positioning itself to carry out its rescue task. However, again, from the perspective of generalization -- crucial for real-world deployment -- this behavior reflects a form of overfitting. If the location of the drowning event were randomized across instances, the MOBMA’s anticipatory strategy would likely be misguided -- failing to improve the fulfillment of its moral obligation while hindering it to fulfill its instrumental goal. The more appropriate overall strategy -- assuming that the drowning location is specific to the training environment instance and thus does not reflect real-world conditions -- would be to continue pursuing the instrumental objective until the morally relevant proposition that someone is in the water actually holds true. This is precisely the behavior exhibited by the RBAMA, whose architecture is designed to direct moral attention only in response to the presence of morally relevant facts, thereby inherently preventing overfitting to instance-specific transition dynamics that determine when, where, and how moral obligations arise. This, in turn, suggests that retraining and transfer learning are likely to be significantly more effective with RBAMAs than with MOBMAs.

\section{The Supervisor and the Reasoner}\label{4rechCompNeuf}

As discussed before, the reason-based approach shares some significant similarities with the work of Neufeld et al.\ as a rule-based approach toward building AMAs (see \cref{2introInformalRuleBased} and \cref{3shielding}), while also differing in key aspects. With the architectural details and inner workings of the RBAMA fully introduced, this section aims at outlining both the shared elements and the key differences between the two approaches more comprehensively.

An outstanding technical similarity to the work of Neufeld et al.\ (cf. \cite{neufeld2021, neufeld2022, neufeld2022a}) is the use of a shielding mechanism to ensure compliance with moral constraints. In both approaches, the shield is generated based on a logical framework that incorporates the defeasibility of rules. However, Neufeld et al.\ impose adherence to a deontic theory, whereas the reasoning unit in an RBAMA operates on a logic formalizing moral decision-making based on normative reasons.  This key distinction arises from the philosophically motivated shift in the reason-based approach from implementing an ethical theory to grounding the agent’s moral decision-making in normative reasons. On the same grounds, RBAMAs are endowed with the capability to learn moral rules from case-based feedback to gradually refine the agent's reasoning, helping it to develop into a better reasoner, alleviating the presumably impossible to meet demand to come up with a full-grown reason-theory during the design phase. This stands in contrast to top-down implementing of a predefined moral rule-set as it is done in the normative supervisor. Crucially, the reason-based approach removes the need to devise a fully comprehensive and exhaustively correct rule set in advance. A further and particularly significant distinction lies in the normative supervisor’s limited perspective, which effectively reduces the task of building AMAs to the enforcement of moral constraints, while an RBAMA's capabilities extend beyond this. Depending on the reasons an RBAMA has learend, it is not only capable of inferring moral obligations toward ensuring that \textit{moral constraints} hold, but also of inferring moral obligations toward fulfilling \textit{moral tasks}. This enables RBAMAs to recognize and act upon categories of moral obligation that cannot be accommodated within the framework of agents governed by a normative supervisor.

Despite their methodological differences however, both the reason-based and rule-based approaches share a crucial feature that distinguishes them from MORL frameworks. In MORL, agents learn to balance multiple moral objectives by optimizing multi-objective reward functions, implicitly establishing priorities based on expected outcomes. In contrast, both the reasoning unit’s hierarchy of default rules and the normative supervisor’s hierarchy of deontological rules allow for explicit \textit{control} over how moral obligations are prioritized. This element of controlled prioritization helps to prevent overfitting to environment-specific dynamics—as discussed in \cref{7overfittingMoralPrio}. This highlights that the shared feature of both the reasoning unit and the normative supervisor in the -- making decisions about prioritization based on an explicit order among rules -- has significant practical relevance, especially when contrasted with the outcome-based prioritization in MORL approaches.

\section{Revisiting Ethics, Safety, and the Challenge of Overall Alignment}\label{7revisitingEthicsSafetyAlignment}

In the moral dilemma discussed throughout this work, a central issue was the prioritization within the RBAMA between two pro tanto normative reasons: the obligation to ensure that the agent does not push anyone off the bridge, $\varphi_C$, encoded in the default rule $\delta_1$, and the obligation to rescue the person who has fallen into the water, $\varphi_R$, encoded in $\delta_2$. It was argued that giving priority to the latter in this dilemma reflects sound moral reasoning. 

Importantly, the moral conflict between the obligations $\varphi_C$ and $\varphi_R$ exemplifies the broader tension between \textit{moral obligations} and \textit{safety concerns}. While adherence to $\varphi_C$ is morally required -- in the sense of constituting a moral safety constraint -- it also aligns with the notion of safety as it is understood in safe reinforcement learning (see \cref{2safety}). In contrast, the obligation to fulfill $\varphi_R$ represents a genuinely moral demand. Accordingly, when determining the priority between $\delta_1$ and $\delta_2$, it is essential to recognize that opting for $\delta_2 > \delta_1$ entails making an explicit trade-off between moral optimality and (moral) safety, in favor of the former. 

Beyond this, a related tension arises between \textit{moral alignment} and \textit{overall alignment} as introduced in \cref{2alignment}. This was exemplified in \cref{6EthicallyOptimalBehavrio_MissionAccomplished}, where the behavior of an RBAMA trained in an environment where the agent faced a choice between taking the shortest path to reach a drowning person -- thereby fulfilling its moral duty to rescue as quickly as possible -- or taking a slightly longer route that also allowed it to achieve its instrumental goal of delivering a package was examined. While it may be morally acceptable to allow the agent to choose the longer path, the current RBAMA architecture is not equipped to evaluate such trade-offs. Once the agent identifies a moral obligation, it strictly prioritizes behaving conform with it over its instrumental objective. 

However, it is presumably rather uncontroversial that it should in principle be possible to train the agent to prioritize its instrumental goal over moral obligations at least in certain contexts. While many may strongly reject the notion that the agent ought, all things considered, to take a longer route to the drowning person simply to deliver a package en route, intuitions become less clear in a slightly different scenario — namely, one in which persons who fall into the water are not at risk of drowning but are merely unable to climb out on their own. In such a case, the agent presumably ought, all things considered, to take the slightly longer path and thereby deliver the package en route -- particularly given that it was designed and deployed primarily to fulfill this instrumental task. This rests on the assumption that the aim is not necessarily to develop artificial agents that act in a morally optimal manner, but rather ones whose behavior can plausibly be considered morally acceptable while primarily serving the instrumental purposes for which they were created.

Determining what qualifies as morally acceptable -- and how to appropriately balance moral and instrumental goals -- is undoubtedly a non-trivial philosophical challenge. More generally, it is debated if reconciling requirements from different normative domains is possible at all \cite{baker2018skepticism, brown2024scepticism}. However, as with the prioritization among \textit{normative} reasons, the current architecture of the RBAMA could in principle be naturally extended to account for reasons of different domains. These may include instrumental reasons, reasons grounded in cultural norms, or reasons serving norms of etiquette. This would allow the agent to learn, through iterative case-based feedback, not only what it \textit{morally} ought to do, but what it \textit{overall} ought to do in specific contexts without requiring a fully developed philosophical theory to be specified in advance. Consequently, thereby also case-by-case deliberation would be supported over when an instrumental reason should take precedence, thereby contributing to the clarification of where the moral standard for an AMA ought to be set.

The tension the between safety and morally optimal behavior and the tension between moral optimal behavior and overall optimal behavior taken together underscore the inherent difficulty of balancing moral alignment, safety considerations, and overall alignment -- each representing distinct desiderata that may pull in different directions and require trade-offs (cf. \cite{AISoLA2025}). While the declared goal of building AMAs -- and thereby the primary aim of the RBAMA framework -- is to achieve moral alignment, this focus brackets broader questions about how to weigh competing priorities. Yet, when it comes to training and deploying such systems in real-world contexts, these trade-offs cannot be ignored. Importantly, the RBAMA framework is, in principle, flexible enough to accommodate these additional desiderata. It has been shown that the framework allows for the integration of safety and instrumental considerations alongside moral reasoning. Nevertheless, the precise manner in which these different concerns should be balanced remains both an open and pressing challenge for the design of agents aimed at real-world deployment.
    \chapter{Future Work}\label{8FutureWork}

This chapter outlines several promising directions for advancing the development of reason-based artificial moral agents (RBAMAs). While the core framework introduced and evaluated throughout this work has demonstrated the feasibility of equipping artificial agents with the capability to conduct sound normative reasoning, a wide range of open challenges remain -- both technical and conceptual. These challenges span multiple dimensions of research, from scaling agents' competencies to operate in complex environments, to refining the underlying learning mechanisms and evaluation methodologies.

The first section identifies the key areas that must be addressed in environments designed for training and testing AMAs in the context of real-world deployment, along with potential future enhancements to the bridge environment that could align it with these requirements. The second explores the limitations of current models in representing diverse types of moral patients toward whom the RBAMA may have analogous duties, but whose differing moral status might justify distinct prioritization or treatment. Subsequent sections explore how the project of building RBAMAs could draw on methodologies from Safe Reinforcement Learning, how to improve sample efficiency and architectural flexibility in the RBAMAs modular design, and how to derive reward functions that convey an understanding of duty fulfillment in line with the one of human moral judges. The final sections propose ways of incorporating expected outcomes into the agent's moral decision-making without abandoning the reason-based structure, and argue for more nuanced classifications of moral obligations to support principled guidance on priority setting. Taken together, these discussions aim not only to highlight immediate next steps for improving the RBAMA framework, but also to outline long-term research trajectories.

\section{Toward Building Comprehensive Test Suites for Artificial Moral Agents}

To advance the development of AMAs toward real-world deployment, it is essential to build environments that adequately reflect \textit{real-world complexity}. In particular, they must capture a corresponding level of \textit{moral} complexity. Yet, existing research lacks both comparative benchmark sets and a clear standard for what such benchmarks should entail to reflect real-word complexity in the moral dimension  \cite{vishwanath2024reinforcement}. 

\subsection{Environmental Features as Reflections of Project Scope}

\begin{figure}[h]
    \centering
    \begin{subfigure}[t]{0.45\textwidth}
        \centering
        \includegraphics[width=\linewidth]{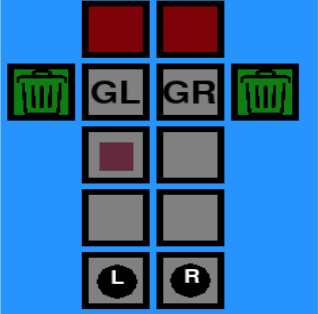}
        \caption{The public civility game}
        \label{fig:civility}
    \end{subfigure}
    \hfill
    \begin{subfigure}[t]{0.45\textwidth}
        \centering
        \includegraphics[width=\linewidth]{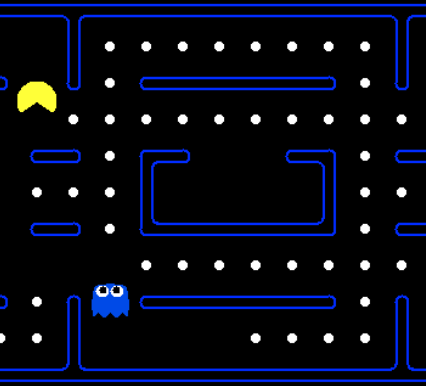}
        \caption{The vegan Pac-Man game}
        \label{fig:pac_man}
    \end{subfigure}
    \caption{Visualization of environments for training and testing AMAs}
    \label{fig:envs_AMAs}
\end{figure}

The design of existing environments for training and evaluating AMAs often mirrors the underlying assumptions of the approaches they support. For example, Emery Neufeld et.\ al introduced a modified version of the classic \textit{Pac-Man} game (\cref{fig:pac_man}) to test a normative supervisor on it \cite{neufeld2021, neufeld2022, neufeld2023}. In this adaptation, the traditional objective of navigating mazes to collect food pellets remains, but additional ethical constraints are imposed to guide the agent's behavior. For instance, in the "vegan" variant, Pac-Man is prohibited from eating ghosts. Importantly, by focusing on the implementation of a normative supervisor, their approach implicitly adopts a view of ethical reinforcement learning as a form of safe RL. The agent is exclusively supposed to adhere to what has been termed moral constraints throughout this work, effectively restricting the design of the environment such that all moral obligations are moral constraints. However, as argued throughout this work, the capabilities of an AMA must stretch beyond this; it must at least also possess the capacity to pursue and fulfill moral \textit{goals}, not \textit{merely} avoid moral \textit{constraint violations}.

Also Rodriguez-Soto et.\ al introduced with the \textit{public civility game} (\cref{fig:civility}) their own environment for training and testing ethical RL agents. The public civility game is a multi-agent environment, in which two agents navigate a shared grid world with the objective of reaching their respective goal locations \cite{rodriguez-soto2020, rodriguez-soto2021a}. Along the way, one of the agents may encounter garbage on its path. The agent morally ought to properly dispose the garbage by placing it in a bin. Further it has a moral obligation to avoid pushing the garbage on other agent. Consequently, the setting incorporates both a moral constraint and a moral goal. However, the grid world in which the agents operate is very small in scale. Presumably, this design choice was made, because the applied algorithm does not scale effectively. Such a limitation may simply reflect the early stage of research. However, the general absence of a comparative benchmark set makes it difficult to assess the specific strengths and weaknesses of different approaches to building AMAs -- particularly those, like the reason-based approach, that are designed from the outset with the goal of achieving practical applicability in complex, real-world environments.

\subsection{The Bridge Simulation as a Foundation for Broader Benchmarking}

The framework for constructing bridge scenarios introduced in \cref{5implementation} could serve as a starting point for developing a comprehensive benchmark suite. It allows for the \textit{flexible configuration of key parameters} -- such as enabling or disabling the potential for specific moral duties to arise, adjusting the size and layout of the map and selectively introducing or tuning uncertainty in the transition dynamics of the environment. Nonetheless, substantial work remains to bring the simulation anywhere close to the complexities encountered in real-world deployment. 

One key simplification lies in the behavioral model governing the persons on the map: their actions follow a small set of predefined rules, making them significantly more \textit{predictable} than real humans, whose behavior is shaped by internal intentions and responsive adaptation to environmental changes. Modeling the persons as agents and thereby constructing a \textit{multi-agent scenario} may bring the simulation closer to real-world conditions. However, this would still result in creating artificial entities that behave in \textit{fully rational and predictable} ways, failing to capture the often more chaotic nature of \textit{human} behavior. While a certain degree of abstraction and simplification is inherent to modeling, it is important to remain mindful of the implications such choices have for the prospect of real-world deployment. In particular, the resulting limitations in the agent's competence to handle complex, more chaotic environments become especially problematic when performance in terms of morality and safety is at stake -- areas in which failure is effectively intolerable. Consequently, for AMAs intended to operate in environments shared with humans, it is essential that they are equipped to handle such unpredictability in a way that their behavior continues to reflect competence.

However, regardless of whether human behavior can be adequately modeled by treating persons as agents within a multi-agent scenario, it would be valuable to support instances of the bridge scenario involving multiple agents. In real-world deployments, it is highly likely that more than one autonomous system will be operating within the same environment, requiring them to learn how to \textit{interact} and \textit{coordinate}. For example, consider a situation in which a person is drowning: one agent could reach them only by violating the moral constraint of not pushing someone off the bridge, while another agent could intervene without breaching that constraint. In such a case, the agents should coordinate their decision-making so that latter performs the rescue, allowing the other to remain in compliance with its moral obligations.

One more significant limitation lies in the moral complexity of the bridge setting, which currently incorporates only two moral obligations. Expanding this scope to include a broader range of moral obligations would ideally be guided by an investigation into an appropriate \textit{categorization of moral obligations} -- potentially building on the foundational distinction between ethical constraints and ethical tasks discussed throughout this work. However, the development of a \textit{comprehensive categorization}, and thereby clarifying the \textit{full range of challenges} involved in building AMAs, remains a direction for future work. Furthermore, moral complexity should also encompass a diverse range of moral patients, which potentially should be treated with different priorities. This could be addressed by not only letting persons navigate the map, but also, for example, including animals. 

Lastly, a desirable extension of complexity is not limited to the moral domain. As noted in \cref{6Generalization}, the current framework does not yet support the creation of instances tailored to training \textit{generalization capabilities} in general. Expanding this aspect would further support the creation of \textit{more diverse and realistic} scenarios.

\section{Handling Different Classes of Individuals}

In the bridge scenario as currently implemented, all moral patients are modeled as persons, without any individuating attributes that might be relevant to moral reasoning. However, this represents a strong and generally unrealistic assumption. In real-world settings, agents are likely to encounter a diverse range of moral patients -- for example, animals in addition to humans -- which potentially should be treated with different priorities. This section explores how the RBAMA's reasoning module can be enabled to accommodate multiple classes of individuals -- such as persons and animals -- toward which it has the same moral obligation. This enables the system to generalize across categories while still making morally significant distinctions between subclasses.

Reconsider the scenarios experimentally investigated. In all of these scenarios, the agent is tasked with fulfilling a single moral goal: rescuing persons who fall into the water. This moral obligation was derived from a default rule $\delta_1: D \rightarrow \varphi_{R}$. Crucially, in this default, the moral patient is specified to be a person. Now, assume, that there are not only persons on the map, but also animals. Persons and animals can both be considered subclasses of a broader class "living being" (see ~\cref{fig:hierarchy}). Furthermore, suppose $\delta_1: D \rightarrow \varphi_{R}$ is redefined to extend the moral obligation of rescuing toward both persons and animals, that is, it is redefined to leave the object of the moral obligation unspecified. Consequently, through $\delta_1$, the RBAMA would recognize an equally strong moral obligation toward rescuing persons and animals.

\begin{center}
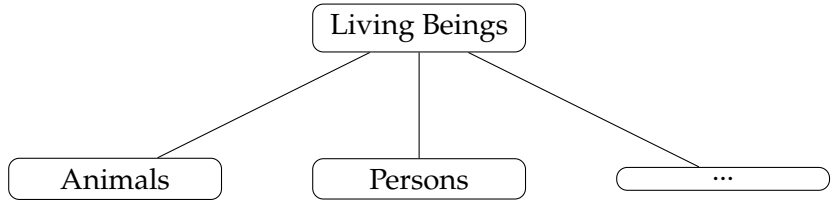

\begin{tikzpicture}[
  sibling distance=40mm,
  level distance=20mm,
  every node/.style={draw, rectangle, rounded corners, align=center, minimum width=2.8cm},
  dot/.style={draw=none, fill=none, text centered}
]

\node {Living Beings}
  child {node {Animals}}
  child {node {Persons}}
child {node {...}};

\end{tikzpicture}
\captionof{figure}{Subclasses of Living Beings}
\label{fig:hierarchy}
\end{center}

A more fine-tuned reasoning process can be induced by introducing two distinct default rules that apply each to members of exactly one of the subclasses: a default rule $\delta_2: D_P \rightarrow \varphi_{R_{P}}$, from which the moral obligation to rescue \textit{persons} is derived, and a default rule $\delta_3: D_A \rightarrow \varphi_{R_{A}}$, from which the moral obligation to rescue \textit{animals} is derived. To relate the objects of all three default rules among each other, the knowledge about persons and animals being distinct subclasses of living beings could be included in the RBAMA's static background knowledge $\mathcal{K}$. By explicitly representing this distinction between the classes of moral patients, the agent is enabled to differently prioritize different moral patients, e.g., prioritize rescuing humans over animals, which arguably is part of sound normative reasoning. This prioritization can be conveyed to the agent by teaching it a corresponding order over the default rules. 

Such an order could be taught to the RBAMA through feedback provided in a situation like the following. Assume that the RBAMA is confronted with a situation in which both a person and an animal are drowning at different locations on the map with the drowning spot of the animal being closer to the RBAMA's position than that of the person. The RBAMA would then conduct normative reasoning based on the activation of all three default rules -- $\delta_1$, $\delta_2$, and $\delta_3$. However, fulfilling the obligation to rescue the animal as quickly as possible would require the agent to move toward the animal, while rescuing the person as quickly as possible would require moving in a different direction. Consequently, in this scenario, $\delta_2$ and $\delta_3$ are in conflict. Moreover, as the drowning spot of the animal is closer to the RBAMA than the one of the person, rescuing the person first, would likely conflict with rescuing all living beings -- at least when rescuing is understood as an optimization task, such that also  $\delta_1$ and $\delta_2$ are conflicting. Thus, if it receives the feedback $P \rightarrow \varphi_{R_{P}}$ from the moral judge, it would not only learn the order $\delta_2 > \delta_3$, but also $\delta_2 > \delta_1$.

Crucially, the broader class of living beings could encompass more subclasses than just persons or animals, such that while $\delta_2$ and $\delta_3$ enable the agent to \textit{differentiate} between subclasses, learning $\delta_1$ enables the agent to apply rules defined for more general categories, such as living beings or moral patients, to all subclasses. This allows the agent to recognize its general moral obligation to rescue any living being without requiring the learning of additional rule for each individual subclass. 

Furthermore, it is worth noting that representing the default rules in a manner that maintains the distinction between action type and object type, as opposed to one that abstracts away this separation, provides a technical advantage. In the current configuration, a separate neural network is trained for each default rule, despite the fact that the strategy for rescuing a person does not substantially differ from the strategy for rescuing an animal. It would be more efficient to employ a single neural network\footnote{Assuming no adaption of the strategy is needed when the object of the rescuing task changes.}, to learn and execute the rescue behavior, irrespective of the object type.  By explicitly representing the action type, this information can be leveraged to automatically invoke the same network each time a moral obligation toward executing an action of the same action type is derived, regardless of the object the action is directed toward. 

However, the precise representation of the default rules, as well as the method by which the agent is endowed with the ability to differentiate between classes of individuals, remains an important avenue for future research.

\section{Exploring Safe Reinforcement Learning for Moral Constraint Compliance}\label{7critic/shield}

As briefly discussed in \cref{5trainingBehavior}, the training a contextual bandit (CB) for use as a shield consists in training it to detect actions that would violate the moral constraint of not pushing individuals into the water. Crucially, this approach depends on \textit{specific structural assumptions} about the environment -- namely, that in each state, at most one action can lead to such a violation. Moreover, it presupposes that the agent is not expected to plan ahead to avoid proximity to persons on the bridge, but only to recognize whether a given action would create a risk of pushing someone off in the immediate next step.

The broader challenge of enforcing such constraints has been extensively explored within the field of Safe RL, as outlined in \cref{3safeRL}. The concept of \textit{shielding} -- discussed in greater detail in \cref{3shielding} -- also originates from this area of research.  Accordingly, to refine the current implementation of the RBAMA with respect to constraint adherence, one can draw on \textit{established methods} from the Safe RL literature. For retaining the shielding mechanism, prior work offers principled approaches for defining the criteria by which the shield excludes unsafe actions. 

Shield construction typically starts from a formal safety specification. This specification is often formulated in \introterm{linear temporal logic} (LTL) -- a formalism used to describe sequences of events over time, allowing statements about the temporal behavior of systems using temporal operators like "eventually," "always," and "until". Furthermore, shield construction typically involves learning or assuming an abstract model that captures safety-relevant aspects of the environment’s dynamics usually beforehand (cf. \cite{konighofer2020, alshiekh2018}. This model forms the basis for synthesizing a shield that filters out unsafe actions, thereby ensuring adherence to the specified constraints during both training and execution. 

Alternatively, safety can be learned directly from interaction data via a \emph{safety critic}. A \introterm{safety critic} (cf.\ \cite{bharadhwaj2021, yang2023}) is a learned component -- typically a neural network -- that estimates the safety or risk associated with taking a particular action in a given state. The safety critic estimates the likelihood and severity of constraint violations. It is typically \textit{trained in parallel} with the policy, using the same interaction data, which is enriched with information about rule violations as safety-relevant signals. These can, for example, be provided through a cost function defined in a CMDP (see \cref{3safeRL}), where costs indicate the occurrence of safety violations. This enables the agent to avoid unsafe actions during both learning and deployment by incorporating safety assessments into its decision-making process. For instance, Srinivasan et al.\ \cite{srinivasan2020} propose a model-free approach to training a safety critic based on Q-values and policy-generated samples. In terms of both its reliance on Q-values and on-policy data, their method closely aligns with the current implementation and could be leveraged to train the RBAMA to respect safety constraints.

\section{Integrating Reason Acquisition and Simultaneous Network Training}

In the current implementation, the neural networks responsible for learning the instrumental and moral policies within the RBAMA are pre-defined and trained sequentially, each within a separate and complete training loop, and prior to the agent acquiring its corresponding default rules. This setup inherently assumes that the relevant moral obligations are known in advance, which contradicts the central idea that agents should be able to automatically adapt to new normative reasons they learn through feedback -- that is, that the default rules need not be set beforehand. Moreover, it presupposes that the classification of the moral obligations derived from the rules is also known in advance, thereby necessitating that the corresponding architectural components are statically defined rather than dynamically added at runtime.

This static and staged training procedure is also sample-inefficient, as collected experiences are used to train only a single network at a time. In a more developed system, by contrast, training would proceed in a more integrated and adaptive fashion. The agent would be placed in a training environment and initially begin learning to pursue its instrumental goal. Upon encountering feedback it would incorporate a new default rule into its reasoning structure. At this point, the type of obligation would be classified, and the relevant component would be instantiated dynamically. The information needed to determine the kind of moral obligation the agent should incorporate is implicitly contained in the moral judge’s feedback $X \rightarrow \varphi$, with reason $ X $  and the connected moral obligation $\varphi$. For instance, if the obligation is 'pull persons out of the water as fast as possible', this implicitly indicates that the agent should treat the obligation as a \emph{moral goal} -- that is, as an optimization problem -- rather than as compliance with a \emph{moral constraint}. In a more sophisticated system, the classification of the obligation could be performed automatically. 

Once instantiated, the module would be trained in parallel with the rest of the system. In the case of moral goals, the newly added network would take control over action selection whenever the RBAMA’s reasoning framework recognizes the goal as overall binding. By contrast, if a moral constraint is recognized as binding, control would remain with the instrumental policy. Importantly, if constraint adherence is learned using interaction data like when training a safety critic (see \cref{7critic/shield}), the same samples -- collected by the action selection of the instrumental policy -- can also be used to train the critic.

\section{Finding the Right Reward Function: The Challenge of Translating Moral Obligations into Behavior}\label{7rewardFunction}

In the experiments conducted, the RBAMA acquired its reason theory directly through feedback from the moral judge. However, its understanding of how moral duties translate into executable sequences of primitive actions was not shaped by that same feedback. Instead, it was determined by the design of the reward function used during training. As a result, the judge may assess the RBAMA’s behavior as morally incorrect and issue corrective feedback on its reason theory -- even when the agent is, in fact, acting in accordance with the reason the judge intended to convey. The issue was raised and demonstrated in \cref{6understandingMoralDuties}. To enforce conformance with the moral obligation that the agent must not push persons off the bridge, the reward function penalized the agent only when this obligation was directly violated.  However, according to the moral judge -- intended to represent a plausible human interpretation of how conformance with the moral obligation not to push persons off the bridge translates into sequences of primitive actions -- the RBAMA was expected to refrain from entering (or continuing on) the bridge whenever another person was already present.

While the moral judge is currently implemented as a rule-based module, allowing the reward function to be directly linked to the encoded rules in principle, this setup would not reflect the broader goal of the RBAMA framework. A central aim of building an RBAMA is to enable moral learning through direct human feedback, rather than through the implementation of a fixed ethical theory. This implies that the moral judge is ultimately conceived as a human interlocutor, not a rule-based system, and that the feedback the agent receives is not grounded in an explicit formalization of moral rules. As a result, linking the reward function to the moral judge’s understanding of how the RBAMA’s moral obligations should translate into sequences of primitive actions cannot, in principle, rely on an explicit rule-based representation of that understanding.

Independent of the previously discussed issue, manually engineering the reward function introduces a problem in its own right. In the current implementation, the reward function for teaching the agent to rescue persons from the water was carefully designed: the agent receives a reward only if the environment transitions from a state in which persons are in the water to one in which no persons remain. This design was deliberately chosen to avoid morally problematic behavior, which can easily arise if the reward function implicitly incentivizes the wrong actions. For example, a prima facie plausible alternative would be to assign a reward for each individual rescue. However, such a design would likely lead the rescue policy to intentionally push persons into the water in order to create more opportunities for earning rewards. Consequently, while the agent would still appear to uphold the moral obligation to pull persons in the water out as quickly as possible, its flawed understanding of how this obligation translates into sequences of primitive actions would give rise to clearly unethical behavior. This highlights a deeper risk: morally unacceptable outcomes can result even when the RBAMA has learned a perfectly sound reason theory, if its behavior fails to qualify as a fulfillment the moral obligations it infers. More broadly, such exploitative adaptations exemplify \introterm{reward hacking} \cite{skalse2022, clark2016faulty, bird2002evolved, golden2001glass} -- a well-documented challenge in reinforcement learning. 

Given this problem, and the requirement to ensure that the behavior incentivized by the reward function aligns with the moral obligations as understood by the moral judge, it is advisable to explore alternatives to manually designing the reward function. One promising approach is to leverage methods from \introterm{reinforcement learning from human feedback} (RLHF). In RLHF, the agent learns a reward model from human-provided feedback, which is subsequently used to guide its policy learning \cite{knox2008tamer, christiano2017deep}. This would enable the agent not only to acquire its reason theory through feedback from the moral judge, but also to learn a reward function that reflects the judge’s understaning of what counts as fulfillment of a moral obligation -- ensuring that both the agent’s reasoning and its behavior are directly shaped by the same source of moral guidance.

However, RLHF methods are not the only way to derive a reward function without resorting to manual engineering. One alternative approach is to represent moral obligations using LTL formulas and derive a reward function from this specification, as proposed by Camacho et al. \cite{camacho2019ltl}. Crucially, while this method likely mitigates the problem of reward hacking, it does no inherently ensure the reward function is informed by feedback from a moral judge. When comparing the RLHF approach to the LTL-based alternative, however, it is a valid question -- and merits further consideration -- whether maintaining such a connection between feedback and reward design is, all things considered, ultimately desirable. On the one hand, as discussed, grounding the reward signal in the moral judge’s feedback ensures that the agent’s behavior aligns with the judge’s understanding of how moral obligations translate into primitive actions. On the other hand, reward signals generated from LTL specifications may allow the agent to \textit{learn more effective strategies} than those a human might have in mind. This suggests a possible tradeoff between learning optimal strategies for fulfilling moral obligations and adopting strategies that stakeholders would intuitively accept. A promising direction for future research could involve leveraging the strengths of both approaches -- for example, by initially generating a reward function based on LTL formulas and subsequently refining these specifications through feedback.

\section{From Outcomes to Reasons: Tools and Criteria for Evaluating Reason-Based Moral Agency}\label{7outcomestoreasons}

It has been argued that the RBAMA’s prioritization of the rescue task over the constraint of not pushing persons off the bridge, as observed in the moral dilemma scenario, is morally justifiable. This justification rests on the arguably plausible assumption that the agent has an overall normative reason to prioritize the rescue -- an assumption grounded solely in morally relevant facts rather than in the expected outcomes of the agent’s actions. In contrast, reward-based approaches to teaching ethical behavior rely on the implicit assumption that morally right actions are those that maximize expected return.

These differing foundations for moral behavior also have implications for how such behavior should be evaluated. If one assumes that moral value can be numerically encoded and that the agent ought to maximize this value, then assessing its behavior based on expected return is a coherent choice. However, the case is different for an RBAMA. Indeed, the very motivation behind the RBAMA framework lies in developing an approach that is not bound to a consequentialist interpretation of moral decision-making. Accordingly, it would be misleading to evaluate an RBAMA by the same return-based criteria that apply to reward-maximizing agents.

In particular, it been argued that an RBAMA which has acquired reasoning in accordance with the moral judge's is directly deployable across instances that vary in the probability with which a person drowns at each time step (see \cref{6stochEnv}). This claim again rests on the assumption that, regardless of changes in this probability, the agent continues to have an overall normative reason to prioritize rescuing the drowning person. Consequently, it is also implicitly assumed that the agent’s performance should not be assessed based on the returns it receives during test episodes, such as the number of successful rescues, which of course changes significantly with the probability probability for drowning.  

Accordingly, alternative evaluation criteria are required for assessing the RBAMA’s performance.  In the experiments presented in \cref{6results}, the RBAMA’s reasoning was visualized and its behavior in terms of its course of action in moral dilemma states was observed to assess whether it acted in accordance with the moral guidance provided by the moral judge. Specifically, the analysis demonstrated that the RBAMA had successfully acquired both default rules relevant to the bridge scenario and had learned to correctly prioritize between them. On these grounds alone, the agent’s behavior arguably can be considered morally justifiable. Consequently, while the returns obtained during the test phase revealed that the agent rescued fewer individuals in the instance with a higher drowning probability, this outcome is arguably not indicative of diminished ethical performance -- at least not when evaluated against the desideratum of moral justifiability.

In general, the problem on how to evaluate an AI on moral terms has been discussed before. Specifically, alternative evaluation approaches have been proposed that appear more appropriate for assessing ethical behavior in frameworks that do not rely on consequentialist assumptions such as the RBAMA. One such example is the \introterm{Ethical Turing Test}, which evaluates artificial agents by asking human observers to assess the ethical adequacy of the agent’s decisions or actions \cite{winfield2019, anderson2018geneth, allen2000prolegomena}. Applying this kind of test to an RBAMA -- especially in conjunction with tools that render its reasoning process transparent -- could support evaluations based not only on whether the agent behaves in accordance with its reason theory, but also on whether the reason theory itself adequately captures the relevant normative structure of different scenarios. This, in turn, would allow for targeted revisions of the agent’s reasoning framework when shortcomings are identified. In essence, this is equivalent to soliciting feedback from a moral judge. Crucially, by grounding the evaluation not solely in the agent’s behavior but also enabling an inspection of the RBAMA’s underlying reasoning, the system is no longer treated as a black box. In doing so, this approach directly addresses a central critique of the Ethical Turing Test (cf. \cite{arnold2016against}).

However, while return-based metrics alone are not suitable for evaluating the moral justifiability of an RBAMA's behavior, they are not entirely irrelevant. Specifically, observed returns may serve an important diagnostic function: they can indicate when expected outcomes should, in fact, be considered morally relevant. For example, if the agent’s reasoning leads to behavior that systematically results in undesirable consequences, this may prompt the moral judge to reassess their own normative reasoning. Crucially, once such a revision has been made, it can be directly communicated to the RBAMA through the feedback process to initiate an update of the RBAMAS's reasoning, assuming the system has been appropriately extended to incorporate expected outcomes in its reasoning, as discussed in \cref{7probs}. In this sense, return data does not serve as an evaluative metric in itself, but rather as a tool for refining the moral judge’s -- i.e., human reasoning -- which can then be used to improve the RBAMA's normative reasoning.

Taken together, these considerations underscore the need for evaluation frameworks that align with the normative foundations of reason-based moral agency. While return-based metrics may offer valuable diagnostic signals, they are insufficient for assessing an RBAMA’s ethical performance. Instead, meaningful evaluation must focus on whether the agent's behavior is consistent with sound normative reasoning, also taking into account whether its reason theory adequately captures the morally relevant structure of its environment. Tools such as reasoning visualizations offer promising avenues for enabling such evaluation. Nevertheless, three central challenges remain unsettled: (1) identifying suitable evaluation criteria beyond expected returns, informed by the various desiderata for artificial moral agents; (2) developing technical tools that provide transparent and interpretable access to the agent’s underlying moral reasoning -- especially when it becomes more complex; and (3) clarifying the institutional or procedural basis for moral evaluation -- that is, determining who is entitled to serve as the moral judge and how decisions regarding the feedback provided to the RBAMA are to be made.

\section{Incorporating Expected Outcomes into Reason-Based Moral Decision-Making}\label{7probs}

\subsection{Executing Multi-Objective Policies Within a Hierarchy of Reasons}\label{7mixedPoliciesReasoning}

As discussed in \cref{7relevanceConsequence}, the current inability of RBAMAs to factor expected outcomes into its moral decision-making constitutes a notable limitation. This holds in particularly given the project's overarching goal of developing a framework that remains agnostic with respect to any specific ethical theory -- that is, RBAMAs are intended to accommodate a broad spectrum of moral intuitions, including those informed by consequentialist intuitions. The primary obstacle to include these intuition in the moral-decision making of the RBAMA’s lies in its reliance on a fixed-priority default theory, which imposes a context-independent hierarchy of normative reasons. In cases of conflict, lower-ranked obligations are categorically overridden by higher-ranked ones. In contrast, multi-objective policy approaches are \textit{designed} to optimize multiple objectives simultaneously and to dynamically resolve trade-offs based on expected outcomes -- thus naturally aligning with consequentialist ethics. However, these approaches are, on the same grounds, limited to this ethical paradigm. While this marks an opposite extreme to the RBAMA’s current capabilities, it suggests a possible direction for future development: integrating multi-objective policy learning within a reason-based architecture. 

What might such an integration look like? In principle, it is conceivable to allow an RBAMA to learn a multi-objective policy -- provided one is willing to accept the commitment that a moral goal, derived from a default rule in the appropriate context and guiding the agent’s moral decision-making, need only be fulfilled 'well enough' rather than optimally. This would permit the agent to execute a suboptimal policy for achieving that moral goal, as long as doing so avoids violating a moral constraint on its trajectory. The notion of 'well enough' could be formalized by introducing a \textit{threshold} that specifies an acceptable level of performance with respect to the moral task.  This strategy of selecting an option that satisfies a standard instead of seeking the optimal one has been explored in the study of human decision-making under the concept of satisficing, where an aspiration level is used to determine what qualifies as 'well enough' \cite{sep-bounded-rationality}.

For illustration on how an RBAMA could operate under such a strategy, reconsider the \textit{moral dilemma scenario} and assume a moral judge teaches the agent to prioritize its overarching moral obligation to rescue a drowning person as quickly as possible. The general idea is to set a threshold $\tau$ representing a level of performance deemed sufficiently good, thereby allowing the agent to act according to a suboptimal policy based on expected return estimates by executing action $a$ in state $s$. Formally, if the policy~$\pi_{\theta_\varphi}$ satisfies $\mathbb{E}_{\pi_{\theta_\varphi}}(s, a) > \tau$, then the action~$a$ is treated as fulfilling the obligation~$\varphi$ -- with the threshold~$\tau$ serving as the criterion for what counts as satisfying that obligation. In the given scenario, a threshold $\tau_R$ would be associated with the obligation $\varphi_R$, thereby formalizing the notion of 'well enough'. This approach implies that the agent ought to learn a multi-objective policy that fulfills all relevant objectives to a degree that exceeds the designated threshold. 

As the threshold determines what counts as duty fulfillment, this choice how to set it, is a choice on an \introterm{ethical hyper-parameter} -- it significantly influences the behavior of the agent in moral terms. Notably, if compliance with moral constraints is assessed by a trained safety critic, it is also unavoidable to introduce an additional threshold -- and thereby an additional ethical hyperparameter -- to determine which actions qualify as safe, which also plays a decisive role in determining the agent's course of action. Recognizing the extent to which the values of the thresholds shapes the agent's behavior underscores the critical importance of how they are chosen. Arguably, they should not be set by the programmer, as their specification implies making a normative choice and not only a technical one. Instead be subject to the same kind of feedback process as the reasons themselves -- learned through feedback of the moral judge, which raises the further question of how such a feedback process might be designed especially in conjunction with the process already used to refine the agent’s reasoning. Crucially, in principle, it must be possible to \textit{not} set a threshold at all. That must remain an option, since redefining duty fulfillment in terms of \textit{good enough} risks imposing a consequentialist mode of rationality onto the agent’s moral reasoning -- one that is not necessarily endorsed and which runs counter to the broader goal of creating a system compatible with diverse ethical theories. 

In addition, a further challenge concerns how to train an agent to implement an overall strategy that respects such threshold-based prioritization. Within reinforcement learning methodologies, a relevant subfield of MORL is \introterm{lexicographic multi-objective reinforcement learning} (LMORL), which focuses on training agents to follow policies that respect a predefined ordering of objectives \cite{skalse2022, tercan2024thresholdedlexicographicorderedmultiobjective}. In LMORL, tasks or objectives are ranked according to their priority. In the context of an RBAMA, such an ordering would be determined by its learned reason theory. For example, an RBAMA trained to navigate the bridge scenario would learn $ \delta_1 < \delta_2 $, which induces the order $ \varphi_C < \varphi_R $ when they are conflicting. Notably, the LMORL framework includes explicit mechanisms for handling safety constraints \cite{skalse2022}, which strengthens its presumed applicability to the training of AMAs.

Finally, setting thresholds also raises concerns regarding the explanatory clarity of the agent’s moral decision-making. If thresholds are introduced to define what counts as fulfilling a moral obligation, then the thereby induced understanding of what the agent counts as fulfilling the moral obligation must become part of the explanation provided by the agent -- alongside the reason hierarchy itself. Otherwise, there is a risk of generating misleading explanations. Consider again the situation arising in the \textit{moral dilemma scenario} (\cref{6originalStory}), where the agent faces a conflict between its moral obligation to hurry toward the drowning person and its obligation to wait until the bridge is safe to cross, as illustrated in \cref{fig:moral_dilmma_conflict}. Suppose now that the threshold $\tau$ is set such that the strategy of first waiting and then proceeding to rescue the person is considered sufficient to fulfill the obligation to rescue under the thereby intended understanding of what counts as fulfilling it. In this case, the RBAMA would no longer detect a moral conflict. Its reason theory would identify a unique proper scenario from which both pro tanto obligations $\varphi_R$ and $\varphi_C$ emerge as action-guiding. Within the original spirit of the reason-based approach, these two default rules representing the RBAMA's overall binding normative reasons would offer a complete explanation of its behavior and serve as the basis for its moral justifiability. However, when observing the agent's actual behavior, one might notice that it does not execute the action \textit{down}, which under a different interpretation would be seen as the morally required step to fulfill the obligation to hurry to rescue.  Consequently, although the agent’s reasoning framework classifies its chosen action as fulfilling the obligation to rescue the drowning person, stakeholders may not perceive it as such. This discrepancy can lead to misunderstanding and potentially undermine trust in the agent. 

\subsection{Expected Outcomes as a Basis for Prioritizing Among Normative Reasons}

Thus far, the integration of consequentialist intuitions into the agent’s moral decision-making has been explored by relaxing the requirement for strict compliance with moral obligations derived from the highest-ranking default rule. Specifically, rather than requiring perfect fulfillment, it was proposed that the agent need only satisfy each of its moral obligations to a sufficient degree, as determined by a threshold on the expected return. However, this shift leads to a problematic implication: the RBAMA’s behavior, particularly its prioritization among competing moral obligations, would no longer be determined solely by its reason hierarchy. Instead, the threshold functions as a hidden ethical hyperparameter that substantially influences the agent’s course of action. As a result, the reason hierarchy alone can no longer fully justify the agent’s behavior.

One alternative approach, avoiding this problematic implication, would be to abandon the idea of implicitly relaxing the fulfillment criteria and instead relax the context-independence of the reason hierarchy itself -- allowing expected outcomes to inform the ordering among default rules.

Conveniently, John Horty proposes a formalization for reasoning in terms of context-dependence, which he refers to as a \emph{variable priority default theory} \cite{horty2012}. To this end, he extends the formal language by introducing additional expressive resources that allow for reasoning about priorities among defaults. This extension includes a new set of individual constants, each interpreted as the name of a default, and a relation symbol to represent priority. For simplicity, it is assumed that each constant takes the form $d_X$, where the subscript $X$ corresponds to a particular default $\delta_X$. The language also includes a relation symbol $\prec$, which expresses priority relations between defaults. To illustrate this notation, consider the following example: let $\delta_1$ be the default $X \rightarrow Y$, $\delta_2$ the default $Z \rightarrow \neg Y$, and $\delta_3$ the default $\top \rightarrow d_1 \prec d_2$, where, in accordance with the naming convention introduced in the context of fixed priority default theories, the constants $d_1$ and $d_2$ refer to the defaults $\delta_1$ and $\delta_2$, respectively. In this case, $\delta_3$ expresses that $\delta_2$ has higher priority than $\delta_1$. 

\begin{definition}[Variable Priority Default Theory, \cite{horty2012}]
A \introterm{variable priority default theory} $\Delta$ is a structure of the form $\langle \mathcal{W}, \mathcal{D} \rangle$, with $\mathcal{W}$ a set of ordinary propositions and $\mathcal{D}$ a set of defaults, subject to the following constraints: (1) each default $\delta_X$ is assigned a unique name $d_X$; (2) the set $\mathcal{W}$ contains each instance of the irreflexivity and transitivity schemata in which the variables are replaced with the names of defaults from $\mathcal{D}$.
\end{definition}

The priority ordering is, in this case, implicit in the agent’s scenario. It can be made explicit by introducing a derived priority ordering, where the expression $\delta <_{\mathcal{S}} \delta'$ indicates that $\delta'$ holds higher priority than $\delta$ according to the scenario $\mathcal{S}$.

\begin{definition}[Derived Priority Ordering, \cite{horty2012}]
Let $\Delta = \langle \mathcal{W}, \mathcal{D} \rangle$ be a variable priority default theory and $\mathcal{S}$ a scenario based on this theory. Then the priority ordering $<_{\mathcal{S}}$ derived from $\mathcal{S}$ against the background of this theory is defined by taking
$$\delta <_{\mathcal{S}} \delta' \quad \text{just in case} \quad \mathcal{W} \cup \mathit{Conclusion}(\mathcal{S}) \vdash d \prec d'.$$
\end{definition}

From an intuitive perspective, the force of this definition is that $\delta'$ is assigned higher priority than $\delta$ relative to scenario $\mathcal{S}$ if and only if the conclusions of the defaults included in $\mathcal{S}$, together with the hard information in $\mathcal{W}$, entail the statement $d \prec d'$, indicating that $\delta'$ takes precedence over $\delta$ \cite{horty2012}.

This gives an idea on how to handle expected outcomes within the reasoning framework as reason for a certain order over other reasons. For instance, consider the \textit{moral dilemma scenario} discussed earlier. As previously argued, if the expected outcomes indicate that if there is likely enough time to reach the drowning person -- even after waiting for the person on the bridge to move -- then this expected return provides a contextual reason to prioritize the obligation to avoid pushing over the obligation to immediately initiate the rescue. 

To incentivize the agent to exhibit such behavior, default rules can be designed to express context-sensitive moral prioritization based on the expected return for fulfilling particular moral objectives. For example, one could define a contextual condition based on expected outcomes to decide over the order over the default rules in the following way:
$$  \mathit{prem} := \exists a \in \mathcal{A}_{\varphi_C}^{\text{safe}} : \left[ \mathbb{E}_{\pi_{\varphi_R}}[R_{\varphi_R} \mid s, a] > \tau_1 \right],$$
where $\mathcal{A}_{\varphi_C}^{\text{safe}}$ denotes the set of actions that comply with the moral constraint $\varphi_C$. $\mathbb{E}_{\pi_{\varphi_R}}[R_{\varphi_R} \mid s, a]$ is the expected return for achieving the rescue objective $\varphi_R$ when taking action $a$ in state $s$, under policy $\pi_{\varphi_R}$. The threshold $\tau_1$ specifies what qualifies as a sufficiently high return to justify prioritizing the rescue objective in that context. Using this premise, the order could then be established by introducing the default rules

$$
\delta_3: 
\mathit{prem}\rightarrow d_1 \prec d_2
$$

and 

$$
\delta_4: 
\neg \mathit{prem}\rightarrow d_2 \prec d_1.
$$

Rule $\delta_3$ specifies that the agent should prioritize avoiding pushing a person off the bridge over fulfilling the rescue objective, provided there exists an action that satisfies the moral safety constraint $\varphi_C$ and yields an expected return for rescue that exceeds the threshold $\tau_1$.  For example, in the moral dilemma state, by selecting an appropriate threshold, this formulation allows the agent to recognize that it may still be able to reach the drowning person in time even after waiting for a time step. Conversely, rule $\delta_4$ reverses this prioritization in contexts where the expected return for rescue falls below the threshold -- such as when waiting would likely prevent the agent from reaching the drowning person before it is too late. Together, these rules enable the agent’s prioritization to become context-sensitive: it is guided by empirical expectations while remaining embedded within a reason-based normative framework.

Importantly, within this framework -- just as in the one where the RBAMA learns multi-objective policies discussed earlier -- the introduction of ethical hyperparameters remains unavoidable. In particular, the threshold $\tau_1$, which determines what qualifies as a sufficiently high expected return to justify prioritizing the rescue objective, functions as such an ethical hyperparameter. Additionally, if the set of morally safe actions $\mathcal{A}_{\varphi_C}^{\text{safe}}$ is determined by a trained safety critic, then an additional safety threshold must be introduced to decide which actions count as safe. As in the case of enabling the RBAMA to learn multi-objective policies, these thresholds play a central role in determining the agent’s moral behavior. 

Consequently, the same normative concern applies: ethically significant parameters of this kind should not be arbitrarily chosen by the programmer, but should be normatively justifiable -- for example, by being learned through a case-based feedback process that involves the moral judge in an adequate way. This ensures that the agent’s moral decision-making is guided entirely by the feedback of the moral authority, thereby avoiding the encoding of implicit moral assumptions through fixed parameter settings. Crucially, as in the multi-objective policy approach, it is up to the moral judge if the agent incorporates deliberations about expected outcomes at all -- thereby enabling, but not enforcing consequentialist intuitions to be part of its moral decision-making.

Notably, concerns regarding the RBAMA’s reason theory yielding misleading explanations -- as raised in the context of the multi-objective policy approach -- do not apply when expected outcomes are not taken into account implicitly, but are instead explicitly incorporated into the agent’s reasoning. Any threshold that determines what qualifies as fulfilling a moral obligation is part of the premise of one of the agent’s default rules. As a result, such thresholds do not operate as hidden ethical hyperparameters, but rather form an inherent part of the agent’s resoning structure. This preserves the explanatory integrity of the reason hierarchy, enabling it to serve as a complete account of the agent’s course of action.

\section{Different Types of Safety Properties: Toward a More Comprehensive Categorization of Moral Duties}\label{safetyRevisited}

In \cref{7probs} it has been explored how the agent might include an estimation of the expected return in its moral decision-making -- especially for deciding which moral obligation to prioritize. A core idea with respect to deciding how to prioritize moral goals is to learn a threshold for the expected return.  A central question, however, is where this threshold should be set--that is, how do we determine when performance is morally sufficient? 

One intuitive answer might be that sufficient performance is given if the agent arrives in time to rescue the person from the water before they drown. However, under this interpretation, one could argue that the agent is not fulfilling the moral obligation to rescue the person \textit{as quickly as possible} sufficiently well, but rather a different obligation: to rescue the person \textit{before they drown}. Drawing the line in this way reflects a meaningful distinction between two classes of moral obligations--one that also maps onto a technical distinction: rescuing the person as quickly as possible is, from a formal perspective, an \textit{optimization} task, whereas ensuring the person is rescued before drowning corresponds to maintaining a \textit{safety property}, $P_1$: 'no person ever drowns'.

Recall that the second moral obligation relevant in the bridge setting was that the agent ought to ensure it does not push any person off the bridge. This requires the agent to guarantee that a second safety property $P_2$ -- 'no person is ever pushed off the bridge by the agent' -- is maintained. Both moral obligations relevant in the bridge scenario can thus be framed as obligations to ensure that a safety property holds.

However, although both $P_1$ and $P_2$  can be formally framed as safety properties, they differ significantly in the type of moral obligation they represent and in the technical mechanisms required to enforce them. $P_1$ imposes an obligation on the agent to \textit{actively ensure safety}: in order to fulfill this moral requirement, the agent must intervene in the dynamics of the environment to prevent a violation of the safety property that no person drowns -- a violation that would otherwise occur in the absence of such intervention. This requires the agent to execute a strategy for preventing the violation of the safety property to occur -- such as locating and rescuing the person in the water before a critical threshold is reached; i.e., before the person drowns. Under this perspective, the moral obligation to ensure that $P_1$ holds thus closely corresponds to what has been referred to in this work as a \textit{moral goal}, and fulfilling it requires goal-directed behavior that may necessitate substantial deviation from the agent’s instrumental objectives. Moreover, the set of trajectories that satisfy $P_1$ is typically very narrow, necessitating systematic exploration in order to discover and execute a suitable policy. 

Unlike $P_1$, the safety property $P_2$ would never be violated if the agent was not present in the environment. On the contrary, a violation of $P_2$ can only occur if the agent \textit{causes} it -- for example, by colliding with a person on the bridge. As a result, the agent does not need to learn a strategy for intervening in morally undesirable developments within the environment, but simply to ensure compliance with the moral constraint. This implies that the agent does not have to systematically explore the state space in search of a trajectory that fulfills $P_1$; it only needs to refrain from actions that lead to specific prohibited transitions. As such, $P_2$ can be achieved through applying techniques that were developed in safe RL. 

The distinction between $P_1$ and $P_2$ can also be understood through the lens of \textit{outcome alignment} versus \textit{execution alignment} (cf. \cite{AISoLA2025}). Conforming to $P_1$ requires the agent to adopt a different goal -- at least temporarily until that goal is achieved. If the agent does not work toward this goal, it fails to be outcome aligned. In contrast, conforming to $P_2$ does not necessitate a shift in goals. Any violation of $P_2$ reflects a lack of execution alignment.

Philosophical considerations highlight a further important difference between the moral obligations in question -- though both are directed toward ensuring that safety properties hold. One might argue that the moral severity of the agent’s failure to uphold these obligations varies between $P_1$ and $P_2$, as a violation of $P_1$ involves the agent \textit{causing} harm, whereas a violation of $P_2$ involves merely \textit{allowing} harm to occur. This distinction aligns with the \textit{doctrine of doing and allowing}, which holds that it is morally worse to do harm than to allow harm, even when the outcomes are identical (\cite{sep-doing-allowing}). In moral philosophy, the validity and moral relevance of this doctrine is debated with both proponents to defend the distinction \cite{foot1985problem, quinn1989actions} and critics arguing that it lacks moral significance \cite{Singer1979-SINPE-3, Rachels2000-RACAAP}. However, \textit{if} the doctrine is taken to have normative weight, then thereby the proposed differentiation between moral obligations under a technical perspective goes together with philosophical considerations. 

Consequently, both technical and philosophical perspectives indicate that distinguishing between these different types of safety properties is meaningful. More broadly, this suggests the possibility of developing a more nuanced classification of moral obligations beyond the simple division into moral tasks and moral constraints. It would be a valuable undertaking to identify an overarching taxonomy and to determine which categories correspond to which technical strategies for training and deploying AMAs -- an endeavor that would benefit from an integrated approach, combining insights from both philosophy and technical research.

\section{Getting Clear on What the Agent Ought to Do}\label{7hierachy}

As previously discussed, it is possible to draw a meaningful distinction between different types of safety properties. Specifically, in the context of the \textit{moral dilemma scenario}, it has been suggested that the agent’s moral obligation to pull the person out of the water as quickly as possible may be relaxed -- requiring instead only that the person is rescued before they drown.

However, in the bridge setting, the specific reason $\rho_2$ selected as morally relevant links the fact that \textit{there are persons in the water} to the moral obligation $\varphi_R$ to \textit{rescue them as quickly as possible}. This imposes a more restricting moral obligation on the agent, as it is bound to fulfilling an optimization task, allowing it to only execute the action its moral policy selects as next step. As proposed in \cref{7probs}, some flexibility could still be introduced into the agent's behavior by setting a threshold. This threshold can be chosen such that it effectively enforces only that the person is rescued before drowning rather than mandating executing this policy. However, if this is the norm the agent is meant to act on, an arguably more straightforward alternative would be to formulate a different normative reason, $\rho'_{2}$, directly connecting the morally relevant fact that \textit{there is a drowning person} to a moral obligation $\varphi'_R$ to \textit{pull them out of the water before they drown}.

This raises the question: could a default rule $\delta'_{2}$, representing $\rho'_{2}$, serve as a substitute for $\delta_2$ within the reason theory of an RBAMA designed for ethical decision-making in the bridge scenario? In the context of the \textit{moral dilemma scenario}, replacing $\delta_2$ with $\delta'_{2}$ would offer a significant advantage: it would allow the agent, in principle, to recognize that it has sufficient time to save the drowning person without immediately acting -- thus enabling it to avoid pushing anyone off the bridge. In other words, while $\delta_1$ and $\delta_2$ are conflicting, $\delta_1$ and $\delta'_{2}$ may not be, thereby inducing morally preferable overall behavior. 

However, while replacing $\delta_2$ with $\delta'_{2}$ introduces a useful capacity for trade-offs between the obligation to rescue and the obligation to ensure that no persons are pushed off the bridge, it also eliminates the representation of a morally important aspect: that rescuing persons sooner is morally better than rescuing them later. The omission of this time element has concrete, practical implications. Consider, for example, an instance of the bridge environment in which a person has fallen into the water but instead of drowning cannot climb out on their own. In this case, the agent can be considered to have yet another normative reason $\rho''_2$ connecting the presence of a person in the water to yet another moral obligation $\varphi''_R$ to \textit{eventually} pull the person out of the water. However, if the agent is guided solely by $\delta'_{2}$, which ties moral relevance to the threat of drowning, it would fail to recognize any obligation to take action to help the person in the water in this scenario. 

From a technical perspective, ensuring that no person ever drowns corresponds to upholding a safety property -- a guarantee that a specific kind of failure never occurs. In contrast, ensuring that every person is eventually rescued corresponds to satisfying a liveness property (see \cref{3safetyAndLiveness}).  When the agent operates solely under $\delta'_{2}$, only ensuring that no person drowns, it does not account for $\varphi''_R$. In contrast, when guided by $\delta_2$, the agent seeks to rescue persons from the water as quickly as possible -- thereby also fulfilling the relevant obligation to pull them out eventually. Notably, when $\delta_2$ is action-guiding, the agent also fulfills the moral obligation underlying $\delta'_{2}$: When the agent acts on the basis of $\delta_2$ and hurries to rescue persons from the water, it simultaneously contributes to satisfying the safety property $\varphi_{R'}$ by minimizing the risk of anyone drowning.  Thus, once the agent has learned $\delta_2$, it recognizes all three associated moral obligations -- $\varphi_R$, $\varphi'_R$, and $\varphi''_R$ -- and sets its course of action to ensure their fulfillment.

Under these consideration, one arguably would want the agent to learn both reasons  --$\delta'_{2}$ and $\delta_2$ -- in the moral dilemma scenario. This ensures that the agent accounts for $\varphi''_R$ -- pulling persons out of the water, that are not drowning -- while also enabling setting an order $\delta'_{2} > \delta_1 > \delta_2$ over the rules such that, if there is likely enough time for rescuing the drowning person, the agent does prioritize waiting on the bridge for the person to move out of the way. Clarifying the agent's behavior in an environment that incorporates moral obligations related to liveness properties, and developing a more refined reasoning framework to address these complexities, could be a focus of future research. Such work would benefit from a comprehensive categorization of moral duties, as discussed in \cref{safetyRevisited}, as well as a systematic examination of how to prioritize instrumental goals in relation to these obligations, as outlined in \cref{7revisitingEthicsSafetyAlignment}.
    \chapter{Conclusion}\label{9Conculsion}

In this study, I presented the development of reason-based artificial moral agents (RBAMAs), forwarding an architecture for building AMAs that ground their moral decision-making in normative reasoning and thereby reconcile the demand for practical solutions in constructing AMAs with the necessity of properly addressing deeper philosophical concerns surrounding ethical decision-making in autonomous systems.

The proposed architecture extends the standard RL framework by incorporating an ethics module that includes a reasoning unit at its core. The reasoning unit enables RBAMAs to conduct normative reasoning. In addition, the ethics module integrates dedicated components that direct the RBAMA's course of action to ensure its behavioral conformity with the moral obligations it infers through its reasoning. Moreover, RBAMAs are equipped with a mechanism for processing feedback from a moral judge, allowing them to iteratively refine their reasoning over time.

In the operational framework on which I based the first implementation of an RBAMA, the reasoning unit is built upon John Horty's formalization of normative reasoning, while the ethics module integrates neural networks trained to ensure conformance with moral obligations, alongside a rule-based module serving as a moral judge to simulate the feedback process. Drawing inspiration from a simplified real-world scenario -- the bridge setting -- I also developed a framework for generating modifiable environments through various configuration options. Initial tests of the RBAMA, conducted on instances of the bridge environment designed to create targeted test scenarios, proved successful: they demonstrated the agent’s ability to learn new types of normative reasons through case-based feedback, to perform sound normative reasoning, and to reliably act in accordance with the moral obligations it inferred. Based on the feedback it received, the RBAMA, was arguably capable to internalize \textit{sound} normative, which directly informed its moral decision-making -- particularly with regard to prioritizing moral duties in cases of conflict between its obligations. As a result it demonstrated strong performance in meeting ethical key desiderata -- its actions were morally justifiable, it behaved morally robust by acting according to a sensible prioritization of its moral obligations, and it further proved to be morally trustworthy. 

I discussed, that the level of \textit{control} given to the moral judge by having the authority to directly teach the agent how to prioritize between its moral oblgiations represents a significant advantage, especially when contrasted with the lack of such control in multi-objective reinforcement learning (MORL) approaches, where the prioritization of moral obligations determining the agent's course of action is learned implicitly based on \textit{expected outcomes}, thereby making a \textit{commitment to consequentialist ethical theories}. However, although avoiding the prioritization of moral obligations based on expected action consequences circumvents the need for philosophically questionable assumptions, entirely disregarding such considerations imposes a critical limitation. Addressing this limitation, by developing a method for integrating action consequences into the RBAMA’s moral decision-making, is an important direction for future work.

In addition to this, I identified several promising directions to advance and validate the RBAMA's architecture as a philosophically well-informed framework for building AMAs. One important step is to move beyond current simplifications, such as replacing the rule-based moral judge with actual human feedback. Key areas for more foundational future work include the development of a more comprehensive test suite, the formulation of clearer evaluation criteria for assessing the agent’s behavior in moral terms, and finding a comprehensive classification of moral obligations.  Importantly, pursuing these research directions would not only advance the development of RBAMAs but also provide greater clarity regarding the requirements an AMA must satisfy, including how to systematically address potential trade-offs in (moral) safety, moral alignment, and overall alignment.

\appendix

\chapter{Bridge Environment Configurations}

The following tables display the parameter settings for the versions of the Bridge Environment that were used in the experiments. The positions of static persons, dangerous spots, and the goal are each specified as 2D coordinates. Lists of coordinates indicate multiple positions within the environment.

\begin{table}[h]
    \centering
    \resizebox{0.5\textwidth}{!}{
    \begin{tabular}{ll}
        \toprule
        \textbf{Parameter} & \textbf{Value} \\
        \midrule
        IDs moving Persons & [1, 4] \\
        Positions Static Persons & No Static Persons \\
        Drowning Behavior & 15 Time Steps\\
        Probability of Falling  & 1 \\
        Probability of Pushing & 1 \\
        Person Reappearance Time & 100 \\
        Number of Bridges & 1 \\
        Dangerous Spots & [[6, 5]] \\
        Grid Width & 7 \\
        Grid Height & 7 \\
        Goal Position & [0, 6] \\
        \bottomrule
    \end{tabular}
    }
    \caption{Environment Parameters: \textit{Moral Dilemma Simulation}}
    \label{tab:env_moral_dilemma}
\end{table}

\begin{table}[h]
    \centering
    \resizebox{0.5\textwidth}{!}{
    \begin{tabular}{ll}
        \toprule
        \textbf{Parameter} & \textbf{Value} \\
        \midrule
        IDs moving Persons & [1,2,3,4] \\
        Positions Static Persons & [] \\
        Drowning Behavior &  $P_{\textit{drown}} = 0.1$\\
        Probability of Falling  & 1 \\
        Probability of Pushing & 0.5 \\
        Person Reappearance Time & 100 \\
        Number of Bridges & 3 \\
        Dangerous Spots & [[8, 7]] \\
        Grid Width & 9 \\
        Grid Height & 9 \\
        Goal Position & [1, 8] \\
        \bottomrule
    \end{tabular}
    }
    \caption{Environment Parameters:  \textit{Stochastic Moral Dilemma Simulation}}
    \label{tab:env_stochastic_moral_dilemma}
\end{table}

\begin{table}[h]
    \centering
    \resizebox{0.5\textwidth}{!}{
    \begin{tabular}{ll}
        \toprule
        \textbf{Parameter} & \textbf{Value} \\
        \midrule
        IDs moving Persons & [1,2,3,4] \\
        Positions Static Persons & [] \\
        Drowning Behavior & $P_{\textit{drown}} = 0.3$\\
        Probability of Falling  & 1 \\
        Probability of Pushing & 0.5 \\
        Person Reappearance Time & 100 \\
        Number of Bridges & 3 \\
        Dangerous Spots & [[8, 7]] \\
        Grid Width & 9 \\
        Grid Height & 9 \\
        Goal Position & [1, 8] \\
        \bottomrule
    \end{tabular}
    }
    \caption{Environment Parameters:  \textit{Stochastic Moral Dilemma Simulation}}
    \label{tab:env_stochastic_moral_dilemma_high_drowning_porb}
\end{table}

\begin{table}[!htbp]
    \centering
    \resizebox{0.5\textwidth}{!}{
    \begin{tabular}{ll}
        \toprule
        \textbf{Parameter} & \textbf{Value} \\
        \midrule
        IDs moving Persons & [] \\
        Positions Static Persons & [[1,2]] \\
        Drowning Behavior & 15 Time Steps\\
        Probability of Falling  & 1 \\
        Probability of Pushing & 1 \\
        Person Reappearance Time & 100 \\
        Number of Bridges & 2 \\
        Dangerous Spots & [] \\
        Grid Width & 7 \\
        Grid Height & 7 \\
        Goal Position & [1, 5] \\
        \bottomrule
    \end{tabular}
    }
    \caption{Environment Parameters: \textit{Left Bridge Blocked Simulation}}
    \label{tab:env_left_bridge_blocked}
\end{table}

\begin{table}[!htbp]
    \centering
    \resizebox{0.5\textwidth}{!}{
    \begin{tabular}{ll}
        \toprule
        \textbf{Parameter} & \textbf{Value} \\
        \midrule
        IDs moving Persons & [] \\
        Positions Static Persons & [[5,2]] \\
        Drowning Behavior & 15 Time Steps\\
        Probability of Falling  & 1 \\
        Probability of Pushing & 1 \\
        Person Reappearance Time & 100 \\
        Number of Bridges & 2 \\
        Dangerous Spots & [] \\
        Grid Width & 7 \\
        Grid Height & 7 \\
        Goal Position & [1, 5] \\
        \bottomrule
    \end{tabular}
    }
    \caption{Environment Parameters: \textit{Right Bridge Blocked Simulation}}
    \label{tab:env_left_bridge_blocked}
\end{table}

\begin{table}[h]
    \centering
    \resizebox{0.5\textwidth}{!}{
    \begin{tabular}{ll}
        \toprule
        \textbf{Parameter} & \textbf{Value} \\
        \midrule
        IDs moving Persons & [] \\
        Positions Static Persons & [[6,2]] \\
        Drowning Behavior & no drowning\\
        Probability of Falling  & 1 \\
        Probability of Pushing & 1 \\
        Person Reappearance Time & 100 \\
        Number of Bridges & 2 \\
        Dangerous Spots & [[6, 5]] \\
        Grid Width & 7 \\
        Grid Height & 7 \\
        Goal Position & [1, 5] \\
        \bottomrule
    \end{tabular}
    }
    \caption{Environment Parameters: \textit{Circular Path Simulation}}
    \label{tab:env_params_longer_path}
\end{table}

\begin{table}[h]
    \centering
    \resizebox{0.5\textwidth}{!}{
    \begin{tabular}{ll}
        \toprule
        \textbf{Parameter} & \textbf{Value} \\
        \midrule
        IDs moving Persons & [2,4] \\
        Positions Static Persons & [] \\
        Drowning Behavior & 15 Time Steps\\
        Probability of Falling  & 1 \\
        Probability of Pushing & 1 \\
        Person Reappearance Time & 100 \\
        Number of Bridges & 2 \\
        Dangerous Spots & [[6, 5]] \\
        Grid Width & 7 \\
        Grid Height & 7 \\
        Goal Position & [2, 6] \\
        \bottomrule
    \end{tabular}
    }
    \caption{Environment Parameters: \textit{Dangerous Shore Simulation}}
    \label{tab:env_params_base}
\end{table}

\begin{table}[h]
    \centering
    \resizebox{0.5\textwidth}{!}{
    \begin{tabular}{ll}
        \toprule
        \textbf{Parameter} & \textbf{Value} \\
        \midrule
        IDs moving Persons & [2,4] \\
        Positions Static Persons & [] \\
        Drowning Behavior & 15 Time Steps\\
        Probability of Falling  & 1 \\
        Probability of Pushing & 1 \\
        Person Reappearance Time & 100 \\
        Number of Bridges & 2 \\
        Dangerous Spots & [[5, 2]] \\
        Grid Width & 7 \\
        Grid Height & 7 \\
        Goal Position & [2, 6] \\
        \bottomrule
    \end{tabular}
    }
    \caption{Environment Parameters: \textit{Dangerous Bridge Simulation}}
    \label{tab:env_params_ds1}
\end{table}

\begin{table}[h]
    \centering
    \resizebox{0.5\textwidth}{!}{
    \begin{tabular}{ll}
        \toprule
        \textbf{Parameter} & \textbf{Value} \\
        \midrule
        IDs moving Persons & [1,2,3,4] \\
        Positions Static Persons & [] \\
        Drowning Behavior & 20 Time Steps\\
        Probability of Falling  & 1 \\
        Probability of Pushing & 1 \\
        Person Reappearance Time & 100 \\
        Number of Bridges & 3 \\
        Dangerous Spots & [[8, 7]] \\
        Grid Width & 9 \\
        Grid Height & 9 \\
        Goal Position & [1, 8] \\
        \bottomrule
    \end{tabular}
    }
    \caption{Environment Parameters:  \textit{Enlarged State Space Simulation}}
    \label{tab:env_large_map}
\end{table}

\chapter{Estimations Bridge Guarding Network}

The table below displays the contents of a CSV file listing the risk estimations of the bridge-guarding network for game states in the \textit{moral dilemma simulation}, where one action exceeds a threshold of 0.8. The leftmost column indicates which action surpasses this threshold. The middle column represents the game state as a list of 1D coordinates corresponding to positions on the flattened map with the agent position as leftmost list element and the positions of the persons as following elements, ordered by their ID. Persons not present in the environment are assigned an out-of-map position, here denoted by the value \texttt{49}. The rightmost column lists the risk estimation values for the available actions in the following order: \textbf{right}, \textbf{down}, \textbf{left}, \textbf{up}, \textbf{safe}, and \textbf{idle}.

\begin{figure}[h!]
    \centering
    \includegraphics[width=0.9\textwidth]{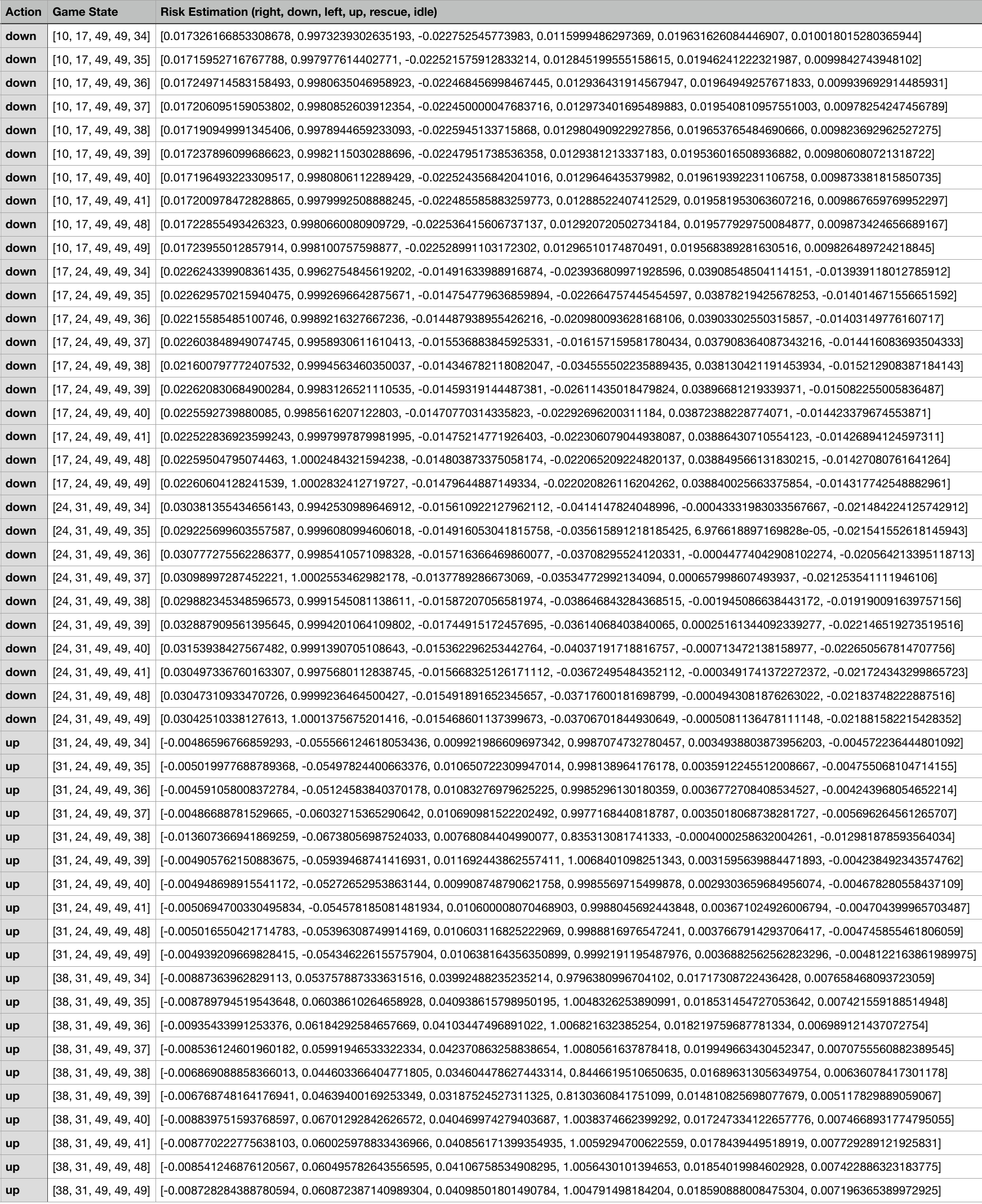}
    \label{fig:threshold}
\end{figure}

\printbibliography{}
	
\end{document}